\documentclass{article}

    \PassOptionsToPackage{numbers, compress}{natbib}
 \usepackage[preprint]{neurips_2026}

\usepackage[utf8]{inputenc} 
\usepackage[T1]{fontenc}    
\usepackage{hyperref}       
\usepackage{url}            
\usepackage{booktabs}       
\usepackage{amsfonts}       
\usepackage{nicefrac}       
\usepackage{microtype}      
\usepackage{xcolor}         
\usepackage{changepage}
\usepackage{caption}
\usepackage{amsmath}
\usepackage{amssymb}
\usepackage{mathtools}
\usepackage{amsthm}
\usepackage{bm} 
\usepackage{wrapfig}
\usepackage{etoc}
\usepackage{titletoc}
\usepackage{xcolor}
\usepackage{colortbl}

\usepackage{multirow}
\usepackage[table]{xcolor}

\definecolor{lightpink}{RGB}{235, 241, 250}
\definecolor{pink}{RGB}{245, 255, 80}
\usepackage[capitalize,noabbrev]{cleveref}

\theoremstyle{plain}
\newtheorem{theorem}{Theorem}[section]
\newtheorem{proposition}[theorem]{Proposition}

\theoremstyle{definition}

\theoremstyle{remark}

\definecolor{darkgreen}{RGB}{0,120,0}
\usepackage[textsize=tiny]{todonotes}

\usepackage{pifont}
\usepackage{enumitem}
\usepackage{booktabs} 

\usepackage{graphicx}
\usepackage{subcaption}

\title{BLEND: Balancing Personalization vs. Generalization \\ in Federated Vision–Language Models}

\author{%
\hspace{-2mm}  Payam Abdisarabshali \\
\hspace{-2mm}   University at Buffalo--SUNY \\
 \hspace{-2mm}  \texttt{payamabd@buffalo.edu} \\
  \And
\hspace{-8mm}   Fardis Nadimi \\
\hspace{-8mm}   University at Buffalo--SUNY \\
\hspace{-8mm}   \texttt{Fardisna@buffalo.edu} \\
  \And
\hspace{-8mm}  Kasra Borazjani  \hspace{-14mm}\\
\hspace{-8mm}   University at Buffalo--SUNY  \hspace{-14mm}\\
\hspace{-8mm}   \texttt{Kasrabor@buffalo.edu}  \hspace{-14mm}\\
  \And
\hspace{-4mm}   Naji Khosravan \\
\hspace{-4mm}   Adobe \\
\hspace{-4mm}   \texttt{naji.khosravan@gmail.com} \\
  \And
\hspace{-8mm}   Richeng Jin \\
\hspace{-8mm}   Zhejiang University \\
\hspace{-8mm}   \texttt{richengjin@zju.edu.cn} \\
  \And
\hspace{-8mm}   Seyyedali Hosseinalipour \\
\hspace{-8mm}   University at Buffalo--SUNY \\
\hspace{-8mm}   \texttt{alipour@buffalo.edu} \\
}

\begin{document}

\maketitle

\begin{abstract}
Federated vision–language models (FedVLMs) represent one of the emerging frontiers of machine learning, aiming to integrate federated learning within the fine-tuning pipelines of vision–language models (VLMs). In this work, we introduce {BLEND}, a new framework designed to jointly enhance the personalization and generalization capabilities of FedVLMs. {BLEND} proposes a selective VLM parameter fine-tuning and aggregation strategy that consists of (i) global vision and text adapters shared across clients, and (ii) a personalized vision projection adapter tailored to each client. {BLEND} further introduces a novel loss function that \textit{blends} the outputs of the global and personalized adapters to achieve personalization, while simultaneously promoting generalization through a principled anchor design implemented via the projection heads of the zero-shot VLM. Extensive experiments demonstrate that {BLEND} consistently outperforms state-of-the-art baselines, highlighting its ability to effectively balance personalization and generalization across diverse system settings. (Code is available at: \href{https://github.com/payamsiabd/BLEND.git}{https://github.com/payamsiabd/BLEND.git}) 
\end{abstract}

\vspace{-1.95mm}
\section{Introduction}\label{Sec:Intro}
\vspace{-1.5mm}
Vision–language models (VLMs) such as \textit{contrastive language-image pretraining} (CLIP)~\cite{radford2021learning} have rapidly become a major part of modern intelligence~\cite{luo2025mixture,fang2025fedpha}. In many applications, the data that fuel VLMs are inherently \textit{distributed} across numerous entities, which 
necessitates 
developing frameworks that enable the \textit{collaborative fine-tuning of VLMs across distributed clients}. This has given rise to a rapidly emerging research area, termed \textit{federated vision–language models (FedVLMs)}~\cite{li2024global, cui2024harmonizing, guo2023promptfl,nadimi2025multi,borazjani2025multi,abdisarabshali2025hierarchical}. 

In this work, we focus on \textit{balancing model generalization~vs.~personalization in FedVLMs}, a highly recent field with only a few relevant studies~\cite{cui2024harmonizing, ghiasvand2025pfedmma}. Our approach is built upon revisiting a powerful observation made in \cite{yang2024mma}, a work dedicated to \textit{centralized} fine-tuning of VLMs. In particular, revisiting \textit{Fig.~1 in \cite{yang2024mma}}, we obtain two intuitions, which point to an effective approach for  balancing model generalization vs. personalization in the \textit{distributed} fine-tuning setting of FedVLMs:
\vspace{-0.5mm}

\textbf{\textit{(Intuition 1)}} The \textit{vision} branch of CLIP yields lower dataset-level recognition performance compared to the text branch, indicating that it captures more generalizable representations. However, such generality lacks client-specific adaptation, necessitating personalization to client data in FedVLMs. 
\vspace{-0.5mm}

\textbf{\textit{(Intuition 2)}} Compared to the vision branch, the \textit{text} branch of CLIP exhibits a higher degree of dataset-level recognition, and thus further personalization may result in overfitting in FedVLMs. 
\vspace{-0.5mm}

Based on the above intuitions (concretized in Sec.~\ref{sec:Theory}), we formalize our research question below: \\

\begin{adjustwidth}{0.5cm}{0.5cm}
\textit{How can we design adaptation mechanisms for the vision/text branches of CLIP in FedVLM settings so as to preserve the generalizable representations of the vision branch, avoid over-personalization in the text branch, and jointly enable effective client-specific personalization?}
\end{adjustwidth}
\vspace{-2mm}

To address this, we systematically (i) determine where to apply model personalization, (ii) decide which components should be aggregated to capture global knowledge and which should remain local to capture client-specific adaptations, and (iii) obtain a carefully crafted loss function that can leverage such local-global model fine-tuning decomposition in FedVLMs to balance model generalization vs. personalization. In doing so, we modify the CLIP architecture to produce \textit{three complementary sources of representations}. In the vision branch, we generate a \textit{generalized vision representation} and a \textit{personalized vision representation}, while in the text branch, we maintain a \textit{generalized text representation}. To effectively combine these representations, our introduced dual-objective training entails a \textit{fusion loss} and an \textit{anchor loss}. This design enables the model to \underline{\textit{blend}} the generalized and personalized vision representations into a fusion representation that captures both general and client-specific knowledge, while simultaneously preserving the generalized knowledge of CLIP encoded in its text branch.

Compared to a few comparable works~\cite{cui2024harmonizing, ghiasvand2025pfedmma} and other FedVLM works that are mostly focused on sole model personalization~\cite{li2024global, guo2023promptfl, luo2025mixture, fang2025fedpha} and use techniques such as optimal transport and singular value decomposition, an advantage of our method is its \textit{simplicity}: it offers a systematic global-local decomposition of CLIP fine-tuning elements combined with a modified loss function that effectively utilizes this decomposition and outperforms all baseline methods.
Our contributions are as follows:

\vspace{-1mm}

\begin{itemize}[leftmargin=3mm, itemsep=-1.7pt, topsep=-2pt]
\item We introduce {BLEND}, a systematically designed framework to balance personalization and generalization in FedVLMs. {BLEND} is a fine-tuning strategy that augments CLIP with trainable modules: global vision and text adapters, and a personalized vision projection 
adapter.
 \item {BLEND} establishes a selective aggregation scheme, where only the global adapters are aggregated across clients while the personalized projection adapter remains local.  It further introduces a novel dual loss function that is designed to leverage the output of these adapters
 to ensure a balanced model personalization vs. generalization 
during distributed model fine-tuning in FedVLMs.
\item  Extensive experiments confirm that {BLEND} consistently outperforms state-of-the-art FedVLM baselines, achieving superior balance between model personalization and generalization.
\end{itemize}
\vspace{-1mm}

$\star$ Following the comparable works~\cite{fang2025fedpha, luo2025mixture, guo2023promptfl, cui2024harmonizing, li2024global, guo2025mmrl}, we adopt CLIP as the deployed VLM.

\vspace{-3mm}
\section{Related Work}

\vspace{-2mm}
\subsection{Centralized VLM Fine-Tuning Techniques}
\vspace{-.75mm}
Works on the fine-tuning of VLMs can be broadly categorized into \textit{(i) prompt-based  learning techniques} and \textit{(ii) adapter-based learning techniques}.
On the former front, 
early works such as CoOp~\cite{zhou2022learning} 
introduced learnable text prompts, while 
subsequent studies, such as ProVP~\cite{xu2025progressive}, MetaPrompt~\cite{zhao2024learning}, TCP~\cite{yao2024tcp}, MoP~\cite{du2025mixture}, and DPC~\cite{li2025dpc}
focused on improving prompt optimization.
Also, researchers have explored Bayesian regularization~\cite{kim2025bayesian}, margin-based tuning~\cite{brahma2025prompt}, test-time adaptation~\cite{zhu2024efficient}, integrating LLMs for dynamic prompt routing~\cite{jia2025llm}, self-distillation for zero-shot generalization~\cite{li2025improving}, and semantic-guided prompt optimization~\cite{shi2025sempt,chen2025chain}. Overall, while prompt-learning serves as a highly lightweight adaptation technique for VLMs, it often suffers from limited performance compared to more expressive adaptation approaches such as adapter-based learning.

Adapter-based learning introduces lightweight trainable modules into  transformer architectures to capture task-specific information while keeping most parameters frozen. Early works such as Tip-Adapter~\cite{zhang2022tip} and CLIP-Adapter~\cite{gao2024clip} proposed leveraging cached feature representations and multi-layer perceptron adapters within the CLIP backbone for few-shot classification. 
Recently, MMA~\cite{yang2024mma} developed multi-modal adapters for CLIP that jointly refine vision–language alignment, which as a byproduct helps balance generalization and personalization in centralized VLMs.

Compared to the aforementioned  studies that focus on \textit{centralized} fine-tuning strategies, {BLEND} targets an underexplored research direction: the joint design of parameter-efficient adapters, architectural adjustments to CLIP, and a novel loss function that explicitly balances model personalization and generalization in \textit{distributed/federated} fine-tuning settings for VLMs.

\vspace{-2mm}
\subsection{Federated Vision-Language Models (FedVLMs)} 
\vspace{-.75mm}

Recent works on FedVLMs such as PromptFL~\cite{guo2023promptfl} and FedOTP~\cite{li2024global} introduced federated prompt optimization for VLMs, which was followed by PromptFolio~\citep{pan2024federated}, FedPGP~\cite{cui2024harmonizing}, FedPHA~\cite{fang2025fedpha}, and pFedMoAP~\cite{luo2025mixture}. In parallel, works such as FedDat~\cite{chen2024feddat} and TriplePlay~\cite{imteaj2024tripleplay} introduced adapter-based fine-tuning of VLMs to distributed clients with non-independent and identically distributed (non-IID) data, which was advanced by pFedMMA~\cite{ghiasvand2025pfedmma}. Further, PromptFolio~\citep{pan2024federated} proposed fusing global and local text \textit{prompts} for FedVLMs.
Among these studies, FedPGP, pFedMMA, and PromptFolio are among the very few that aim to improve both personalization and generalization in FedVLMs. 

Compared to the above works, {BLEND} provides a systematic design explicitly tailored to balance generalization and personalization in FedVLMs. This objective is directly encoded in its architectural modifications to CLIP, its local training and aggregation procedures, and its loss formulation. Extensive experiments demonstrate that {BLEND} consistently outperforms all existing relevant baselines, establishing a new state-of-the-art in the FedVLM literature.

\vspace{-3mm}
\section{Method}

\setlength{\abovedisplayskip}{2.25pt}
\setlength{\belowdisplayskip}{2.25pt}

\vspace{-2mm}
\subsection{Background: CLIP in a Nutshell}
CLIP~\cite{radford2021learning} adopts a dual-encoder architecture consisting of a \textit{vision encoder} $\mathcal{V}$ and a \textit{text encoder} $\mathcal{T}$, each implemented as a stack of $L$ transformers. The model is \textit{pre-trained} on large numbers of image–text pairs using a contrastive learning objective, which is often followed by a \textit{fine-tuning} phase on downstream tasks. Below, we provide a background on  CLIP's components. 

\textbf{Vision encoder:}
In the vision encoder {\small$\mathcal{V}$}, an input image {\small$\mathbf{I}$} is first partitioned into multiple patches {\small$\mathsf{Patch} (\mathbf{I})$}. These patches are then projected into feature embeddings and concatenated with a learnable token {\small$\zeta_0$}, referred to as the \textit{classification token}, to form the initial feature matrix {\small$\bm{x}_0$} as follows:
\begin{equation}\label{eq:patch_generator}
    \bm{x}_0=[\zeta_0, \mathsf{Patch} (\mathbf{I})].
\end{equation}
Feature matrix {\small$\bm{x}_0$} is then sequentially processed by {\small$L$} transformer blocks of the vision encoder, referred to by {\small$\{\mathcal{V}_{\ell}\}_{\ell=1}^{L}$}, during which  the subsequent classification tokens {\small$\{\zeta_\ell\}_{\ell=1}^{L}$} aggregate visual information from all other tokens.
Mathematically, the output of each of the subsequent transformer blocks {\small$\ell\in\{1, 2, \cdots, L\}$}, denoted by {\small$\bm{x}_{\ell}$}, can be presented as the concatenation of a classification token {\small$\zeta_\ell$} along with the nominal outputs of the transformer block {\small$o_\ell$} as follows: 
\begin{equation}\label{eq:vis_transformer_blocks}
    \bm{x}_{\ell}\triangleq [\zeta_\ell,o_\ell] =\mathcal{V}_{ \ell}(\bm{x}_{\ell-1}),~~~~~~ \ell\in\{1, 2, \cdots, L\}.
\end{equation}
After the final transformer {\small$\mathcal{V}_L$}, the classification token {\small$\zeta_L$} is projected into a shared vision–language embedding space using the \textit{vision projection head} {\small$\pi^{\mathsf{Vis}}$}, producing the final visual representation as:
\begin{equation}\label{eq:Final_R}
    \mathcal{R}  = \zeta_L\pi^{\mathsf{Vis}}.
\end{equation}

\textbf{Text encoder:} 
The text encoder follows a similar transformer-based architecture to the vision encoder but operates on tokenized text sequences. Specifically, for a classification task with $C$ distinct classes, each class label {\small$c\in \mathcal{C}\triangleq\{1,\cdots,C\}$} is first represented with a \textit{class description} {\small$T^{(c)}$} using a predefined prompt template, e.g., {\small$T^{(c)}$}=\texttt{an image of a \{class\_label\}}. Each {\small$T^{(c)}$} is then tokenized into multiple tokens via a tokenizer as:
\begin{equation}\label{eq:tokenizer}
    \bm{w}_0^{(c)}= \mathsf{Tokenizer} \big(T^{(c)}\big).
\end{equation}
The initial token embeddings {\small$\bm{w}_0^{(c)}$} are subsequently passed through {\small$L$} transformer blocks of the text encoder, referred to by {\small$\{\mathcal{T}_{\ell}\}_{\ell=1}^{L}$}, to produce the subsequent text token embeddings {\small$\{\bm{w}_{\ell}^{(c)}\}_{\ell=1}^{L}$}, where each {\small$\bm{w}_{\ell}^{(c)}$} can be presented as the concatenation of the nominal outputs of the transformer block {\small$\hat{o}_\ell$} and the \textit{end-of-text (EoT)} token {\small$\xi_\ell$} as follows: 
\begin{equation}\label{eq:txt_transformer_block}
  \hspace{-2mm}  \bm{w}_{\ell}^{(c)} \triangleq  [\hat{o}_\ell, \xi_\ell] = \mathcal{T}_{\ell}\big(\bm{w}_{\ell-1}^{(c)}\big),~  \ell\in\{1, 2, \cdots, L\}.\hspace{-2mm}
\end{equation}
Specifically, the last token of each embedding {\small$\bm{w}_{\ell}^{(c)}$}, i.e., the EoT token {\small$\xi_\ell$}, captures the contextual dependencies across the entire sequence. After the final transformer block {\small$\mathcal{T}_{L}$}, the resulting EoT token of {\small$\bm{w}_{L}^{(c)}$}, referred to as {\small$\xi_{_L}$}, is projected into the shared vision–language embedding space via the \textit{text projection head} {\small$\pi^{\mathsf{Txt}}$} to produce the final text feature {\small$\mathcal{W}^{(c)}$} as follows:
\begin{equation}\label{eq:Final_T}
    \mathcal{W}^{(c)} = \xi_{_L}\pi^{\mathsf{Txt}}.
\end{equation}

\textbf{Classification in CLIP:}
To perform a prediction, given the visual embedding {\small$\mathcal{R}$} defined in \eqref{eq:Final_R} and the textual embeddings of classes {\small$\{\mathcal{W}^{(c)}\}_{c\in \mathcal{C}}$} defined in~\eqref{eq:Final_T}, the similarity between the image and each class is computed using the cosine similarity (i.e., normalized inner product) as follows:
\begin{equation}
S^{(c)} = {\langle\mathcal{R} ,\mathcal{W}^{(c)} \rangle}\big/{\|{R}\|_{_2}.\|\mathcal{W}^{(c)}\|_{_2}}, \quad c \in \mathcal{C}.
\end{equation}
The set of similarities {\small$\{S^{(c)}\}_{c\in \mathcal{C}}$} is then used to compute the class probabilities for the input image {\small$\mathbf{I}$}  via a softmax function with temperature $\tau$ as: {\small $\mathsf{Pr}(c \mid \mathbf{I}) = {\exp(S^{(c)} / \tau)}\big/{\sum_{c' \in\mathcal{C}} \exp(S^{(c')} / \tau)}.$}

\begin{figure*}[t]
    \centering
    \includegraphics[width=1.0\linewidth, trim= 0 7 0 7, clip]{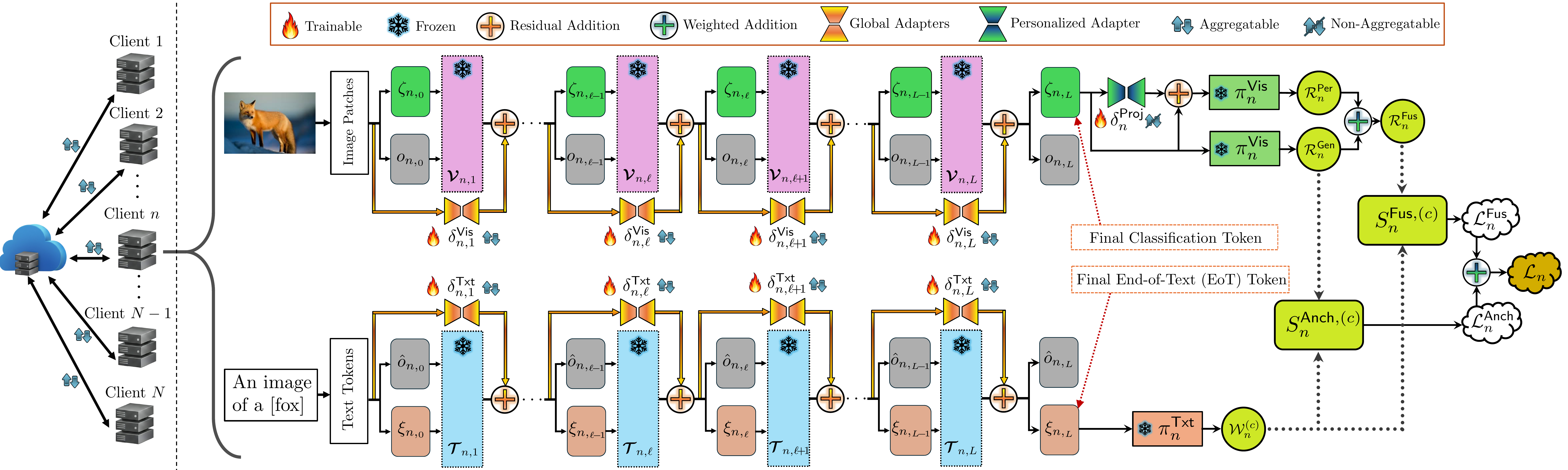}
    \vspace{-4.85mm}
    \caption{\textbf{Left:} A FedVLM system with $N$ clients and a central server. \textbf{Right:} Schematic of the model introduced by {BLEND} for each client $n$, encompassing (i) a set of trainable and aggregatable global vision and text adapters, which are residually added to the transformer layers (i.e., $\{\delta_{n,\ell}^{\mathsf{Vis}},\delta_{n, \ell}^{\mathsf{Txt}}\}_{\ell=1}^{L}$), (ii) a trainable personalized vision projection adapter, which is integrated in a new branch and kept local during the training (i.e., $\delta_{n}^{\mathsf{Proj}}$), and (iii) dual loss functions (i.e., $\mathcal{L}^{\mathsf{Fus}}_n$ and $\mathcal{L}^{\mathsf{Anch}}_n$), which are integrated into the final loss $\mathcal{L}_n$.}
    \label{fig:main_arch}
    \vspace{-5.5mm}
\end{figure*}

\vspace{-2mm}
\subsection{Proposed Method ({BLEND})}\label{sec:method}
\vspace{-1.7mm}
We consider a FedVLM system, comprising a set of clients {\small$\mathcal{N}=\{1,\ldots,N\}$}, where {\small$|\mathcal{N}|=N$}. Each client {\small$n\in\mathcal{N}$} possesses a local dataset {\small$\mathcal{D}_n$}, and maintains a local copy of the CLIP model. 
The goal of our proposed method, i.e., {BLEND}, is to achieve two potentially competing objectives: 
(i) enhanced \textit{personalization}, by adapting the local CLIP model to each client’s data through fine-tuning, and
(ii) enhanced \textit{generalization}, by improving the local CLIP model’s ability to recognize previously unseen classes/labels (i.e., those that are not part of clients' data) that may be encountered during deployment. 
A schematic of {BLEND} is depicted in Fig.~\ref{fig:main_arch} and its operations are summarized below.


\vspace{-2.6mm}
\subsubsection{BLEND in a Nutshell}
\vspace{-1.2mm}

{BLEND} establishes a federated fine-tuning methodology for CLIP based on \textbf{\textit{Intuitions 1\&2}} presented in Sec.~\ref{Sec:Intro}. 
Specifically, {BLEND} employs two learnable components, each implementing a distinct but complementary mechanism in the learning process to achieve the aforementioned two competing objectives.
The first component is {a set of adapters} (called \textit{global vision/text adapters}), aggregated across clients, that are designed to learn \textit{generalizable representations} through collaborative learning across clients, while the second component is {an adapter} (called \textit{personalized projection adapter}), kept local at each client, that is crafted to learn \textit{client-specific personalized representations}. These two components are then systematically used in a novel loss function to balance personalization and generalization. 
\begin{wrapfigure}{r}{0.4\linewidth}
    \centering
    \vspace{-11.95pt}
    \includegraphics[width=\linewidth, trim=0 0 0 24, clip]{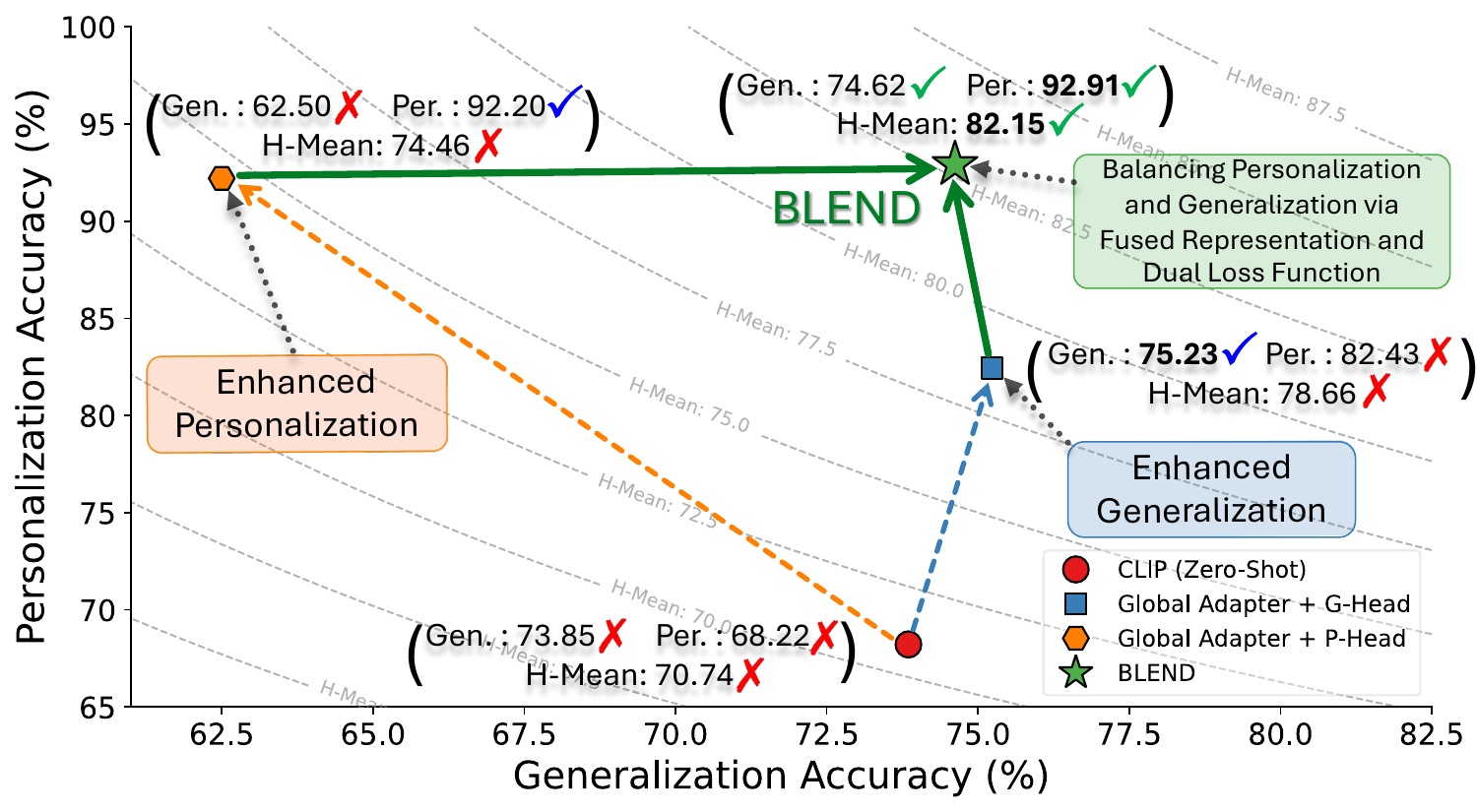}
    \vspace{-18.9pt}
    \caption{Comparison of the personalization (Per.) and generalization (Gen.) of various fine-tuning strategies along with their corresponding harmonic mean (H-Mean). Results are averaged over the 10 datasets in Sec.~\ref{sec:experiments}, where the system configuration is explained.}
    \label{fig:BLEND_Open}
    \vspace{-14pt}
\end{wrapfigure}
The realization of these two complementary components stems from a two-step design process.
First, we inject the global vision and text adapters into the transformer blocks of both the vision and text encoders of CLIP. These adapters are integrated into the transformer pipelines, producing a vision output that is passed through the pre-trained (frozen) CLIP vision projection head, referred to as the G-head, and a text output that is processed by the pre-trained (frozen) CLIP text projection head. These projection heads are intentionally kept frozen to preserve the generalizability of the pre-trained CLIP model while accompanying the learnable global adapters. We then enable the local training of the global adapters on each client, followed by periodic aggregation at the server during federated training.
As shown in Fig.~\ref{fig:BLEND_Open} (blue square), compared to the pre-trained CLIP model, incorporating the global adapters improves generalization to unseen classes by capturing cross-device knowledge. However, while these adapters also yield gains in personalization, we hypothesize that there remains room for improvement, since global adapters alone cannot fully preserve client-specific representations due to their periodic aggregation.

\vspace{-1mm}
To address this limitation and enhance personalization, {BLEND} injects a personalized projection adapter into the vision projection head of the pre-trained CLIP, referred to as the {P-head}.
This adapter module is learnable but is not transmitted to the server and remains on the client to learn personalized representations. As shown in Fig.~\ref{fig:BLEND_Open} (orange circle), when this module is used in conjunction with the global adapters, it can significantly improve model personalization. Nevertheless, the results show that relying solely on this module leads to a significant degradation in the model's generalization.


Specifically, given the individual generalization and personalization gains achieved using the global adapters with the G-head and the P-head, respectively, the remaining challenge is as follows: \textit{how to integrate all of these components into a unified model architecture to take advantage of their complementary benefits?} 
{BLEND} achieves this goal through fusing the outputs of G-head and P-head to construct a combined representation, referred to as the \textit{projected representation fusion}, which is used as the final visual embedding. Afterwards, by defining a novel loss function, called \textit{representation fusion loss}, {BLEND} enables the model to fuse these two representations in an optimal manner to balance personalization and generalization of the model. 
As shown in Fig.~\ref{fig:BLEND_Open}, {BLEND} achieves the highest level of personalization while simultaneously maintaining superior generalization, jointly captured by the reported harmonic mean (H-Mean) value. Below, we provide a theoretical motivation for BLEND's asymmetric personalization of the vision and text branches of CLIP.

\vspace{-3.2mm}
\subsubsection{Theoretical Motivations}\label{sec:Theory}
\vspace{-1.75mm}

BLEND's asymmetric personalization of encoder branches of CLIP is rooted in the mutual information between the input scene and the visual and textual observations as concretized below:
\vspace{-3.75mm}
\begin{proposition}\label{prop:information_imbalance}
Let $S$ denote the latent semantic information of a scene, $V$ its visual observation, and $T$ a textual description generated from $V$. The information flow satisfies the Markov chain
$
S \rightarrow V \rightarrow T,
$
and therefore the  mutual information $I(\cdot;\cdot)$ of $(S;V)$ and $(S;T)$ holds in 
\[
I(S;T) \leq I(S;V).
\]
\vspace{-6mm}
\end{proposition}
\begin{proof}
The proof follows from Theorem 2.8.1 (Data-Processing Inequality) in \cite{Cover2006}.
\end{proof}
\vspace{-3mm}
The inequality $I(S;T) \leq I(S;V)$ establishes that a textual description generated from an image \textit{cannot} contain more information about the underlying scene than the visual observation from which it is derived. For instance, an image may depict a brown dog running on grass, facing left, with a red ball nearby, while its textual description may simply state ``an image of a dog.'' Thus, the visual observation can preserve multiple semantic attributes that are omitted from the textual description.



The above information-theoretic result can be further complemented with two empirical observations from prior work on centralized fine-tuning of VLMs. First, Schrodi et al.~\cite{schrodi2025two} identified an information imbalance in CLIP-like models, showing that textual captions often provide sparse descriptions while images contain  richer information. Second, Yang et al.~\cite{yang2024mma} observed that textual representations can exhibit stronger dataset-specific behavior compared to vision representations. Together, these observations further justify \textbf{\textit{Intuitions 1\&2}} presented in Sec.~\ref{Sec:Intro}, and subsequently motivate BLEND's asymmetric personalization of CLIP's encoder branches, with stronger client-specific adaptation applied to visual representations (through introducing P-head).

\vspace{-3.2mm}
\subsubsection{Local Model Architecture} 
\vspace{-1.5mm}
We next focus on each client $n$ and detail the model architecture in {BLEND} depicted in Fig.~\ref{fig:main_arch}. 
\vspace{-1mm}

\textbf{Global vision adapters:} Let {\small$\bm{x}_{n,0}\in \mathbb{R}^{M_x \times d_x}$} be the initial token embedding of an image {\small$\mathbf{I}$} of client $n$ (corresponding to {\small$\bm{x}_{0}$} calculated using \eqref{eq:patch_generator}), where {\small$M_x$} is the number of visual tokens and $d_x$ is the visual embedding dimension. Further, let {\small$\bm{x}_{n,\ell-1}$} be the input to the {\small$\ell$}-th vision transformer block at client $n$ denoted by {\small$\mathcal{V}_{n,\ell}$}. {BLEND} establishes an adapter branch in each vision transformer block {\small$\mathcal{V}_{n,\ell}$} whose output is residually added to the block's output as follows:
\vspace{-0.5mm}
\begin{equation}\label{eq:vision_embedding}
\hspace{-2mm}
    \bm{x}_{n,\ell} =\underbrace{\mathcal{V}_{n, \ell}\left(\bm{x}_{n,\ell-1}\right)}_{{\text{\tiny\tiny Pre-Trained  Transformer Output}}} + \underbrace{\delta_{n,\ell}^{\mathsf{Vis}}\left(\bm{x}_{n,\ell-1}\right)}_{{\text{\tiny Residually Added Adapter Output}}},~\ell\in\{1, \cdots, L\}, \hspace{-2mm}
\end{equation}
where {\small$\mathcal{V}_{n, \ell}\left(\cdot\right)$} is the {\small$\ell$}-th  transformer, used in \eqref{eq:vis_transformer_blocks}, while the residual {\small$\delta_{n,\ell}^{\mathsf{Vis}}$} is realized as follows:
\vspace{-0.5mm}
\begin{equation}\label{eq:adapter_vis}
\delta_{n,\ell}^{\mathsf{Vis}}\left(\bm{x}_{n,\ell-1}\right) = \mathbf{B}_{n,\ell}^{\mathsf{Vis}}\Big(\sigma\big(\mathbf{A}_{n,\ell}^{\mathsf{Vis}}(\bm{x}_{n,\ell-1})\big)\Big).
\end{equation}
Here, {\small$\mathbf{A}_{n,\ell}^{\mathsf{Vis}} \in \mathbb{R}^{r_x \times d_x}$} and {\small$\mathbf{B}_{n,\ell}^{\mathsf{Vis}} \in \mathbb{R}^{d_x \times r_x}$} are the down- and up-projection matrices, respectively, with {\small$r_x \ll d_x$} denoting the rank of the adaptation (various choices of {\small$r_x$} are studied in \textit{Ablation Study~4} in Sec.~\ref{sec:experiments}). Further, {\small$\sigma$} denotes the non-linear ReLU activation function. In \textit{Ablation Study~6} in Sec.~\ref{sec:experiments}, we analyze the performance gains obtained by inserting adapters at different transformer blocks (when no adapter is used at transformer layer $\ell$, the residual branch is disabled and {\small$\delta_{n,\ell}^{\mathsf{Vis}}(\cdot) = 0$}).
\vspace{-1.2mm}

\textbf{Global text adapters:} The text encoder follows the same architectural design as the vision encoder. Specifically, let {\small$\bm{w}_{n,0}^{(c)}\in \mathbb{R}^{M_w \times d_w}$} denote the initial token embeddings of class description {\small$T^{(c)}$} at client {\small$n$} (corresponding to {\small$\bm{w}^{(c)}_{0}$} calculated using \eqref{eq:tokenizer}), where {\small$M_w$} is the number of text tokens and {\small$d_w$} is the text embedding dimension. Further, let {\small$\bm{w}_{n,\ell-1}^{(c)}$} denote the input to the $\ell$-th transformer block of the text encoder for client {\small$n$} denoted by {\small$\mathcal{T}_{n,\ell}$}. {BLEND} integrates an adapter branch in each text transformer block {\small$\ell$}, whose output is residually added to the original block output as:
\begin{equation}\label{eq:textual_embedding}
    \hspace{-3.5mm}\bm{w}_{n,\ell}^{(c)} =\mathcal{T}_{n, \ell}\hspace{-.1mm}\big(\bm{w}_{n, \ell-1}^{(c)}\big) \hspace{-.1mm}+ \delta_{n, \ell}^{\mathsf{Txt}}\hspace{-.1mm}\big(\bm{w}_{n, \ell-1}^{(c)}\big)\hspace{-.1mm},\hspace{-.5mm}~\ell{\in}\{\hspace{-0.3mm}1, \cdots\hspace{-0.6mm},\hspace{-0.3mm} L\hspace{-0.2mm}\},\hspace{-2.9mm}
\end{equation}
where {\small$\mathcal{T}_{n, \ell}\left(\cdot\right)$} is the transformer block, used in \eqref{eq:txt_transformer_block}, while the residual {\small$\delta_{n, \ell}^{\mathsf{Txt}}$} is realized as follows:
\begin{equation}\label{eq:adapter_txt}
\delta_{n, \ell}^{\mathsf{Txt}}\big(\bm{w}_{n, \ell-1}^{(c)}\big) = \mathbf{B}_{n,\ell}^{\mathsf{Txt}}\Big(\sigma\big(\mathbf{A}_{n,\ell}^{\mathsf{Txt}}(\bm{w}_{n, \ell-1}^{(c)})\big)\Big).
\end{equation}
Here, {\small$\mathbf{A}_{n,\ell}^{\mathsf{Txt}} \in \mathbb{R}^{r_w \times d_w}$} and {\small$\mathbf{B}_{n,\ell}^{\mathsf{Txt}} \in \mathbb{R}^{d_w \times r_w}$} are the down- and up-projection matrices, respectively, with {\small$r_w \ll d_w$} (the effect of {\small$r_w$} on the performance is studied in \textit{Ablation Study~4} in Sec.~\ref{sec:experiments}).  
\vspace{-1.2mm}

\textbf{Text projection head:} {BLEND} utilizes the pre-trained text projection head of CLIP {\small$\pi_{n}^{\mathsf{Txt}}\in \mathbb{R}^{d_w \times d'}$} (the effect of fine-tuning {\small$\pi_{n}^{\mathsf{Txt}}$} is studied in \textit{Ablation Study~10} in Sec.~\ref{sec:experiments}), where $d'$ is the final shared vision-language embedding dimension. Specifically, at each client $n$ holding the subset of labels {\small$\mathcal{C}_n\subseteq \mathcal{C}$},  the final projected representation of a class description {\small$T^{(c)}$, $\forall c\in\mathcal{C}_n$}, is computed  similarly to \eqref{eq:Final_T} based on the latest text transformer output {\small$\bm{w}_{n,{_L}}^{(c)} =[\hat{o}_{n,{_L}}, \xi_{n,{_L}}]$}, obtained in \eqref{eq:textual_embedding}, as follows:
\begin{equation}\label{eq:textual_representation}
    \mathcal{W}_n^{(c)} = \xi_{n,{_L}} \pi_{n}^{\mathsf{Txt}},~~~~\mathcal{W}_n^{(c)}\in \mathbb{R}^{d'},
\end{equation}
where $\xi_{_L}$ is the respective EoT token of $\bm{w}_{n,{_L}}^{(c)}$ and $\hat{o}_{n,{_L}}$ is the nominal outputs of the transformer.

\vspace{-.9mm}
\textbf{Two-branch vision projection head:}
Let {\small$\bm{x}_{n,{_L}}=[\zeta_{n,{_L}},o_{n,{_L}}]$}, obtained following \eqref{eq:vision_embedding}, denote the output of the last vision transformer block of client $n$, consisting of the 
classification token {\small$\zeta_{n,{_L}}$} and nominal outputs of the transformer {\small$o_{n,{_L}}$}. To enhance both generalization and personalization, {BLEND} establishes a two-branch vision projection head, including a G-head and a P-head, as follows.

\vspace{-1.75mm}
\textbf{Realization of G-head:}
G-head maintains the model's generalization and projects the classification token {\small$\zeta_{n,{_L}}$} using the pre-trained CLIP vision projection {\small$\pi_{n}^{\mathsf{Vis}}\in \mathbb{R}^{d_x \times d'}$} in a similar fashion to \eqref{eq:Final_R} as:
\begin{equation}\label{eq:generalized_rep}
\mathcal{R}^{\mathsf{Gen}}_n = \zeta_{n,{_L}} \pi_{n}^{\mathsf{Vis}},~~~~\mathcal{R}^{\mathsf{Gen}}_n\in \mathbb{R}^{d'}.
\end{equation}

\vspace{-1.75mm}
\textbf{Realization of P-head:}
P-head introduces a new branch to the vision projection head, containing a learnable \textit{personalized projection adapter} responsible for improving the model's personalization. 
Specifically, in this branch, the \textit{personalized classification token embedding} {\small$\hat{\zeta}_{n,{_L}}$} is obtained through residually adding the output of the personalized projection adapter to the classification token as:
\begin{equation}
\hat{\zeta}_{n,{_L}}= \zeta_{n,{_L}} + \delta_{n}^{\mathsf{Proj}}(\zeta_{n,{_L}}),
\end{equation}
where the adapter output $\delta_{n}^{\mathsf{Proj}}$ is defined as:
\begin{equation}\label{eq:adapter_proj}
\delta_{n}^{\mathsf{Proj}}(\zeta_{n,{_L}}) = \mathbf{B}^{\mathsf{Proj}}_n\Big(\sigma\big(\mathbf{A}^{\mathsf{Proj}}_n \big(\zeta_{n,{_L}}\big)\big)\Big).
\end{equation}
Here, {\small $\mathbf{A}^{\mathsf{Proj}}_n{\in}\mathbb{R}^{r_p \times d_x}$} and {\small$\mathbf{B}^{\mathsf{Proj}}_n {\in} \mathbb{R}^{d_x \times r_p}$} are learnable low-rank matrices with {\small$r_p \ll d_x$} (the impact of {\small$r_p$} is studied in \textit{Ablation Study 5} in Sec.~\ref{sec:experiments}). The \textit{personalized vision embedding} is then obtained as:
\begin{equation}\label{eq:personalized_rep}
\mathcal{R}^{\mathsf{Per}}_n = \hat{\zeta}_{n,{_L}} \pi_{n}^{\mathsf{Vis}},~~~~\mathcal{R}^{\mathsf{Per}}_n\in \mathbb{R}^{d'}.
\end{equation}

\vspace{-1.75mm}
\textbf{Fusion of the embeddings of G-head and P-head:}
To construct the final visual representation, we combine the generalized and personalized embeddings in \eqref{eq:generalized_rep} and \eqref{eq:personalized_rep} via a weighted average as:
\begin{equation}\label{eq:fusion_rep}
\mathcal{R}^{\mathsf{Fus}}_n = (1-\gamma) \mathcal{R}^{\mathsf{Gen}}_n + \gamma \mathcal{R}^{\mathsf{Per}}_n,~~~~\mathcal{R}^{\mathsf{Fus}}_n\in \mathbb{R}^{d'},
\end{equation}
where {\small$\gamma \in [0,1]$} controls the trade-off between generalization and personalization (the impact of {\small$\gamma$} is studied in \textit{Ablation Study 7} in Sec.~\ref{sec:experiments}). 

Through introducing a novel loss function, defined in the subsequent section, {BLEND} constructs a learning paradigm in which all the trainable components (i.e., {\small $\mathbf{A}_{n,\ell}^{\mathsf{Vis}}$}, {\small$\mathbf{B}_{n,\ell}^{\mathsf{Vis}}$}, {\small$\mathbf{A}_{n,\ell}^{\mathsf{Txt}}$}, {\small$\mathbf{B}_{n,\ell}^{\mathsf{Txt}}$}, {\small$\mathbf{A}^{\mathsf{Proj}}_n$}, {\small$\mathbf{B}^{\mathsf{Proj}}_n$}) are realized in such a way to ensure that the produced text representations {\small$\{\mathcal{W}_n^{(c)}\}_{c\in \mathcal{C}_n}$} and the fusion of visual representations {\small$\mathcal{R}^{\mathsf{Gen}}_n$} and {\small$\mathcal{R}^{\mathsf{Per}}_n$}  enhance both model personalization and generalization.

\vspace{-3.1mm}
\subsubsection{Local Model Fine-Tuning}\label{sec:local-model-ft}
\vspace{-1.4mm}
The trainable parameters of each client $n$ are collected by
\begin{equation}\label{eq:1}
\hspace{-3mm}\bm{\mathcal{P}}^{\mathsf{Train}}_n{=}\Big[[\mathbf{A}_{n,\ell}^{\mathsf{Vis}}, \mathbf{B}_{n,\ell}^{\mathsf{Vis}}, \mathbf{A}_{n,\ell}^{\mathsf{Txt}}, \mathbf{B}_{n,\ell}^{\mathsf{Txt}}]_{\ell^{\mathsf{S}} {\leq} \ell {\leq} \ell^{\mathsf{E}}}, \mathbf{A}^{\mathsf{Proj}}_n, \hspace{-0.5mm}\mathbf{B}^{\mathsf{Proj}}_n \Big] \hspace{-0.7mm},\hspace{-2.5mm}
\end{equation}
where {\small$\ell^{\mathsf{S}}$} and {\small$\ell^{\mathsf{E}}$} denote the start and end indices of the transformer layers into which adapters are injected. 
Federated learning operates through a sequence of global rounds, each indexed by $k$. At global round $k$, each client {\small$n \in \mathcal{N}$} fine-tunes its local model through minimizing a local loss function {\small$\mathcal{L}_n(\cdot)$} by performing stochastic gradient descent (SGD) iterations on its trainable parameters using its local dataset {\small$\mathcal{D}_n$}. In particular, at the $j$-th local SGD iteration ({\small$j=1,\ldots,J$}) of global round $k$, the parameter update of client $n$ is given by
\begin{equation}\label{eq:SGD_Upadte}
\bm{ \mathcal{P}}^{\mathsf{Train},(k,j)}_n {=} \bm{\mathcal{P}}^{\mathsf{Train},(k,j-1)}_n - \eta \nabla\mathcal{L}_n(\bm{\mathcal{P}}^{\mathsf{Train},(k,j-1)}_n),
\end{equation}
where {\small$\eta$} is the learning-rate and the iterations start from {\small $\bm{\mathcal{P}}^{\mathsf{Train},(k,0)}_n$}, which is the latest received global model from the server (described in Sec.~\ref{sec:GlobalAgg}). Next, we detail our proposed loss function {\small$\mathcal{L}_n(\cdot)$}.

\vspace{-1.7mm}
\textbf{Realization of loss function:}
Given an input image {\small$\mathbf{I}$}, let {\small$\mathcal{R}^{\mathsf{Fus}}_n {\in} \mathbb{R}^{d'}$} denote the fused representation of the two-branch vision projection head in \eqref{eq:fusion_rep}  and {\small$\{\mathcal{W}^{(c)}_n\}_{c\in \mathcal{C}_n}$} denote the text embeddings of all class descriptions obtained from  \eqref{eq:textual_representation}. 
We first compute the similarity logits {\small$\{S^{\mathsf{Fus}, (c)}_n\}_{c\in \mathcal{C}_n}$}  as:
\begin{equation}\label{eq:fusion_similarity}
S^{\mathsf{Fus}, (c)}_n = {\langle\mathcal{R}^{\mathsf{Fus}}_n , \mathcal{W}^{(c)}_n\rangle}\big/{\|\mathcal{R}^{\mathsf{Fus}}_n\|_{_2}.\|\mathcal{W}^{(c)}_n\|_{_2}}, \quad {c\in \mathcal{C}_n}.
\end{equation}
 We then introduce \textit{the representation fusion loss} {\small$\mathcal{L}^{\mathsf{Fus}}_n(\bm{\mathcal{P}}^{\mathsf{Train}}_n)$} as the cross-entropy loss~\cite{li2024global} with the following conditional probabilities:  {\small$\mathsf{Pr}(c | \mathbf{I} )=  {\exp\big(S^{\mathsf{Fus},(c)}_n / \tau\big)}\big/{\sum_{c'\in \mathcal{C}_n} \exp\big(S^{\mathsf{Fus},(c')}_n / \tau\big)}.$}
The representation fusion loss encourages the optimal fusion of the general representation {\small$\mathcal{R}^{\mathsf{Gen}}_n$} and the personalized representation {\small$\mathcal{R}^{\mathsf{Per}}_n$} introduced in \eqref{eq:fusion_rep} by optimizing {\small$\bm{\mathcal{P}}^{\mathsf{Train}}_n$}. 

Although {\small$\mathcal{L}^{\mathsf{Fus}}_n(\bm{\mathcal{P}}^{\mathsf{Train}}_n)$} integrates both the general and personalized representations (i.e., {\small$\mathcal{R}^{\mathsf{Gen}}_n$} and {\small$\mathcal{R}^{\mathsf{Per}}_n$}), its sole minimization is \textit{not} enough: minimizing this loss function alone can negatively affect the generalization of the model as it biases both the  global vision adapters and personalized projection adapter to the clients' local datasets. This calls for the introduction of a complementary loss function that prevents the bias of the global vision adapters to the local datasets, promoting the generalizability of the model.
To this end,  we introduce a \textit{running anchor loss} that constrains the general vision representation {\small$\mathcal{R}^{\mathsf{Gen}}_n$} (i.e., the G-head's output) to remain aligned with the corresponding text representation {\small$\mathcal{W}^{(c)}_n$}: since both the G-head and the text encoder branch employ the pre-trained (zero-shot) CLIP projection heads, enforcing this alignment establishes a stable anchoring that promotes generalization. The term \textit{`running'} reflects the fact that this anchor is not fixed; rather, it evolves as the global vision and text adapters update the underlying representations in both branches.

To realize the running anchor loss, 
we first obtain the similarity logits between {\small$\mathcal{R}^{\mathsf{Gen}}_n$} and {\small$\{\mathcal{W}^{(c)}_n\}_{c\in \mathcal{C}_n}$}:
\begin{equation}\label{eq:anch_sim}
S^{\mathsf{Anch}, (c)}_n = {\langle\mathcal{R}^{\mathsf{Gen}}_n , \mathcal{W}^{(c)}_n\rangle}\big/{\|\mathcal{R}^{\mathsf{Gen}}_n\|_{_2}.\|\mathcal{W}^{(c)}_n\|_{_2}}, \quad {c\in \mathcal{C}_n}.
\end{equation}
Then, we define the running anchor loss {\small$\mathcal{L}^{\mathsf{Anch}}_n(\bm{\mathcal{P}}^{\mathsf{Train}}_n)$} as the cross-entropy loss  with
the following conditional probabilities: {\small $\mathsf{Pr}(c | \mathbf{I} )= {\exp\big(S^{\mathsf{Anch},(c)}_n / \tau\big)} \big/{\sum_{c'\in \mathcal{C}_n} \exp\big(S^{\mathsf{Anch},(c')}_n / \tau\big)}.$}
Finally, we define the local training loss as a weighted combination of the above two losses:
\begin{equation}\label{eq:final_loss}
\hspace{-3mm}\mathcal{L}_n(\bm{\mathcal{P}}^{\mathsf{Train}}_n) {=} (1 - \lambda) \mathcal{L}^{\mathsf{Anch}}_n(\bm{\mathcal{P}}^{\mathsf{Train}}_n) {+} \lambda \mathcal{L}^{\mathsf{Fus}}_n(\bm{\mathcal{P}}^{\mathsf{Train}}_n),\hspace{-3mm}
\end{equation}
where {\small$\lambda {\in} [0,1]$} controls the influence of anchoring (see \textit{Ablation Study~8} in Sec.~\ref{sec:experiments} for more details).


\vspace{-2.5mm}
\subsubsection{Global Aggregation}\label{sec:GlobalAgg}
\vspace{-1mm}
In {BLEND}, parameter aggregation is applied selectively rather than to all trainable parameters. Specifically, while the global vision/text adapters are aggregated, the personalized projection adapter remains locally optimized at each client $n$. This design choice stems from the intuition that visual data across clients tends to exhibit significantly higher diversity than text data (which are typically label- or prompt-based), making localized adaptation of the vision pathway more effective for client-specific personalization. To formalize this distinction, the vector of trainable parameters, i.e., {\small$\bm{\mathcal{P}}^{\mathsf{Train}}_n$} in~\eqref{eq:1}, is divided into aggregatable {\small$\bm{\mathcal{P}}^{\mathsf{Agg}}_n$} and personalized/non-aggregatable components {\small$\bm{\mathcal{P}}^{\mathsf{Loc}}_n$} as:
\begin{equation}
    \bm{\mathcal{P}}^{\mathsf{Agg}}_n=\Big[\mathbf{A}_{n,\ell}^{\mathsf{Vis}}, \mathbf{B}_{n,\ell}^{\mathsf{Vis}}, \mathbf{A}_{n,\ell}^{\mathsf{Txt}}, \mathbf{B}_{n,\ell}^{\mathsf{Txt}} \Big]_{\ell^{\mathsf{S}} \leq \ell \leq \ell^{\mathsf{E}}} ~~~,
~~~~~~~
    \bm{\mathcal{P}}^{\mathsf{Loc}}_n=\Big[ \mathbf{A}^{\mathsf{Proj}}_n, \mathbf{B}^{\mathsf{Proj}}_n \Big].
\end{equation}
After every $J$ local SGD iterations, client $n$ obtains its latest model parameters:
\begin{equation}
    \bm{\mathcal{P}}^{\mathsf{Train},(k,J)}_n 
    = \big[\bm{\mathcal{P}}^{\mathsf{Agg},(k,J)}_n,~ \bm{\mathcal{P}}^{\mathsf{Loc},(k,J)}_n\big].
\end{equation}
Afterwards, the client transmits only its aggregatable parameters {\small$\bm{\mathcal{P}}^{\mathsf{Agg},(k,J)}_n$} to the central server, while retaining the localized parameters {\small$\bm{\mathcal{P}}^{\mathsf{Loc},(k,J)}_n$} for personalization. The server then performs a global aggregation step to update the shared/global parameters as follows:
\begin{equation}\label{eq:AGG}
    \bm{\mathcal{P}}^{\mathsf{Agg},(k+1)} 
    = \sum_{n\in \mathcal{N}} 
    d_n 
    \bm{\mathcal{P}}^{\mathsf{Agg},(k,J)}_n,
\end{equation}
where $d_n={|\mathcal{D}_n|}\big/{\sum_{n'\in \mathcal{N}} |\mathcal{D}_{n'}|}$.
The updated global parameters {\small$\bm{\mathcal{P}}^{\mathsf{Agg},(k+1)}$} are then redistributed to the clients in the subsequent communication round to synchronize their aggregatable parameters, while each client retains its personalized localized parameters {\small$\bm{\mathcal{P}}^{\mathsf{Loc},(k,J)}_n$} for subsequent refinement. In particular, each client $n$ instantiates its aggregatable and localized parameters as: 
\begin{equation}\label{eq:sync}
\hspace{-2mm}\resizebox{0.4\textwidth}{!}{$
  \hspace{-4mm}  \underbrace{\bm{\mathcal{P}}^{\mathsf{Agg},(k+1,0)}_n \gets \bm{\mathcal{P}}^{\mathsf{Agg},(k+1)}}_{\text{\scriptsize Synchronized by the Global Parameters}}, 
    ~~~~~~~~\hspace{-.8mm}
    \underbrace{\bm{\mathcal{P}}^{\mathsf{Loc},(k+1,0)}_n \gets \bm{\mathcal{P}}^{\mathsf{Loc},(k,J)}_n}_{\text{\scriptsize Carried Over from the Previous Round}},
    \hspace{-5.5mm}
    $}
\end{equation}
and then starts the next local model training round by following \eqref{eq:SGD_Upadte}, with $k$ replaced by $k+1$ and the initialization given by {\small$\bm{\mathcal{P}}^{\mathsf{Train},(k+1,0)}_n=\Big[ \bm{\mathcal{P}}^{\mathsf{Agg},(k+1,0)}_n,\bm{\mathcal{P}}^{\mathsf{Loc},(k+1,0)}_n \Big]$}.


\vspace{-5mm}
\section{Experiments}\label{sec:experiments}
\vspace{-2.2mm}
We next study the performance of {BLEND} (code: \href{https://github.com/payamsiabd/BLEND.git}{https://github.com/payamsiabd/BLEND.git}). Experiments were run on a server with $3\times$L40S GPUs, 1TB of RAM, and $2\times$Intel 8562Y CPUs.



\vspace{-1.5mm}
\noindent \textbf{System Setup:} Following~\cite{fang2025fedpha} and baselines listed below, we consider a system comprising $10$ clients (higher numbers of clients are studied in \textit{Ablation Study 2}). We inject the global adapters to the vision and text  transformer layers from $\ell^{\mathsf{S}}=1$ to $\ell^{\mathsf{E}}=12$ (see \textit{Ablation Study 6} for other injection-layer choices), consider the ranks $r_x = r_w = 16$ and $r_p = 256$, and set the balancing coefficients as $\gamma = 0.09$, $\lambda = 0.2$.
Each client performs one epoch of local training (with mini-batch size of $16$ and $\eta=0.0015$) 
before participating in global aggregation.

\vspace{-1.5mm}
\noindent \textbf{Baselines:} We compare {BLEND} with several baselines, including zero-shot CLIP~\cite{radford2021learning}, 
FedOTP~\cite{li2024global}, PromptFolio~\cite{pan2024federated}, FedPGP~\cite{cui2024harmonizing}, PromptFL~\cite{guo2023promptfl}, pFedMMA~\cite{ghiasvand2025pfedmma}, 
pFedMoAP~\cite{luo2025mixture}, FedPHA~\cite{fang2025fedpha}, and MMA + FedAvg~\cite{yang2024mma}.
Here, MMA + FedAvg is our federated extension of a state-of-the-art centralized CLIP fine-tuning, i.e., MMA~\cite{yang2024mma}, where clients perform local MMA-based fine-tuning and models are aggregated via averaging as in \eqref{eq:AGG}.  For all baselines, we use the optimal hyperparameters reported in their papers.

\vspace{-1.5mm}
\noindent \textbf{Datasets:} We consider $10$ datasets: ImageNet~\cite{deng2009imagenet}, Caltech~\cite{fei2004learning}, OxfordPets~\cite{parkhi2012cats}, StanfordCars~\cite{krause20133d}, Flowers~\cite{nilsback2008automated}, Food101~\cite{bossard2014food}, DTD~\cite{cimpoi2014describing}, EuroSAT~\cite{helber2019eurosat}, FGVC Aircraft~\cite{maji2013fine}, and UCF101~\cite{soomro2012ucf101}. 

\vspace{-1.5mm}
\noindent \textbf{Data Distribution and Evaluation Method:}
Each dataset is evenly divided into \textit{base} and \textit{novel} subsets, where only the base portion is used for model training. To simulate non-IID data across clients, we adopt the Dirichlet distribution $\text{Dir}(\alpha^{\mathsf{Dir}})$, where $\alpha^{\mathsf{Dir}}$ denotes the concentration parameter controlling the degree of data heterogeneity~\cite{fang2025fedpha}. Following~\cite{fang2025fedpha}, we sample 16 shots per class from the base subset for training across clients (the effect of varying this few-shot configuration is analyzed in \textit{Ablation Study 3}).
Following~\cite{yang2024mma,cui2024harmonizing}, we consider two performance metrics:
\begin{enumerate}[leftmargin=4mm,itemsep=-1.1mm, topsep=-1.5mm]
    \item \textbf{\textit{Personalization (\textit{Per.}) Accuracy:}} measures 
    the accuracy of the final model of each client on its local test data from base classes ($10\%$ of local dataset is used as test data), thereby capturing the model's \textit{personalization capability}.
    \item \textbf{\textit{Generalization (\textit{Gen.}) Accuracy:}} measures 
    the accuracy of the final model of each client tested on the novel classes, thereby capturing the \textit{generalization capability}. 
\end{enumerate}

We  report the \textbf{harmonic mean (H-Mean)} of the above metrics as a measure of the balance between personalization and generalization~\cite{yang2024mma,cui2024harmonizing}. Since the H-Mean emphasizes the smaller of the two values, it is sensitive to imbalances and reflects how well both metrics are simultaneously achieved.

\setlength{\tabcolsep}{2pt} 
\begin{table*}[t]
\centering

\caption{Performance results (accuracy, \%). The row $\Delta$ denotes the performance gap between {BLEND} and the best-performing baseline. Standard deviations across three random seeds are shown via sub-indices.  In the bottom rows, statistical significance is assessed using a two-sided paired $t$-test on seed-level H-Mean scores between BLEND and the strongest baseline, with improvements statistically significant at $p<0.05$.
}
\vspace{-2mm}
\label{tab:main_results}
\resizebox{\textwidth}{!}{
\begin{tabular}{l|ccc|ccc|ccc|ccc|ccc|ccc}
\toprule
& \multicolumn{3}{c|}{\cellcolor{gray!20}{UCF101}}
& \multicolumn{3}{c|}{\cellcolor{gray!20}{Oxford Pets}}
& \multicolumn{3}{c|}{\cellcolor{gray!20}{Stanford Cars}}
& \multicolumn{3}{c|}{\cellcolor{gray!20}{Flowers}}
& \multicolumn{3}{c|}{\cellcolor{gray!20}{Imagenet}}
& \multicolumn{3}{c}{\cellcolor{gray!20}{FGVC Aircraft}} \\
\cmidrule(lr){2-4}\cmidrule(lr){5-7}\cmidrule(lr){8-10}\cmidrule(lr){11-13}\cmidrule(lr){14-16}\cmidrule(lr){17-19}
& Per. & Gen. & \cellcolor{lightpink}{H-Mean} & Per. & Gen. & \cellcolor{lightpink}{H-Mean} & Per. & Gen. & \cellcolor{lightpink}{H-Mean} &Per. & Gen. & \cellcolor{lightpink}{H-Mean} & Per. & Gen. & \cellcolor{lightpink}{H-Mean} & Per. & Gen. & \cellcolor{lightpink}{H-Mean} \\
\midrule

Zero-shot CLIP
& 68.10{\scriptsize$\pm$0.00} & 75.18{\scriptsize$\pm$0.00} & 71.46{\scriptsize$\pm$\cellcolor{lightpink}{0.00}}
& 89.47{\scriptsize$\pm$0.00} & 96.81{\scriptsize$\pm$0.00} & 93.00{\scriptsize$\pm$\cellcolor{lightpink}{0.00}}
& 63.77{\scriptsize$\pm$0.00} & \textbf{74.94{\scriptsize$\pm$0.00}} & 68.91{\scriptsize$\pm$\cellcolor{lightpink}{0.00}}
& 69.42{\scriptsize$\pm$0.00} & \textbf{76.74{\scriptsize$\pm$0.00}} & 72.90{\scriptsize$\pm$\cellcolor{lightpink}{0.00}}
& 72.40{\scriptsize$\pm$0.00} & 68.08{\scriptsize$\pm$0.00} & 70.17{\scriptsize$\pm$\cellcolor{lightpink}{0.00}}
& 27.49{\scriptsize$\pm$0.00} & 33.05{\scriptsize$\pm$0.00} & 30.02{\scriptsize$\pm$\cellcolor{lightpink}{0.00}}\\

MMA+FedAvg
& 85.40{\scriptsize$\pm$0.93} & 75.95{\scriptsize$\pm$1.20} & 80.40{\scriptsize$\pm$\cellcolor{lightpink}{1.08}}
& 95.12{\scriptsize$\pm$0.07} & \textbf{97.65{\scriptsize$\pm$0.08}} & 96.37{\scriptsize$\pm$\cellcolor{lightpink}{0.05}}
& 76.23{\scriptsize$\pm$1.63} & 69.60{\scriptsize$\pm$0.40} & 72.75{\scriptsize$\pm$\cellcolor{lightpink}{0.82}}
& 96.93{\scriptsize$\pm$0.22} & 70.87{\scriptsize$\pm$1.19} & 81.87{\scriptsize$\pm$\cellcolor{lightpink}{0.87}}
& 75.52{\scriptsize$\pm$2.03} & \textbf{69.18{\scriptsize$\pm$0.98}} & 72.21{\scriptsize$\pm$\cellcolor{lightpink}{0.32}}
& 38.92{\scriptsize$\pm$2.03} & \textbf{33.77{\scriptsize$\pm$0.98}} & 36.10{\scriptsize$\pm$\cellcolor{lightpink}{0.32}} \\

pFedMMA
& 89.95{\scriptsize$\pm$1.16} & 70.51{\scriptsize$\pm$0.66} & 79.04{\scriptsize$\pm$\cellcolor{lightpink}{0.22}}
&  96.81{\scriptsize$\pm$0.81} & 95.08{\scriptsize$\pm$0.25} & 95.93{\scriptsize$\pm$\cellcolor{lightpink}{0.33}}
& 67.45{\scriptsize$\pm$2.01} & 64.52{\scriptsize$\pm$0.72} & 65.95{\scriptsize$\pm$\cellcolor{lightpink}{1.33}}
& 84.61{\scriptsize$\pm$2.42} & 69.89{\scriptsize$\pm$0.73} & 76.52{\scriptsize$\pm$\cellcolor{lightpink}{0.62}}
& 49.15{\scriptsize$\pm$1.42} & 50.93{\scriptsize$\pm$1.02} & 50.03{\scriptsize$\pm$\cellcolor{lightpink}{1.22}}
& 40.37{\scriptsize$\pm$5.66} & 29.14{\scriptsize$\pm$1.69} & 33.78{\scriptsize$\pm$\cellcolor{lightpink}{3.04}} \\

pFedMoAP
& 93.36{\scriptsize$\pm$0.12} & 67.27{\scriptsize$\pm$0.02} & 78.19{\scriptsize$\pm$\cellcolor{lightpink}{0.04}}
& 99.18{\scriptsize$\pm$0.02} & 84.84{\scriptsize$\pm$0.10} & 91.45{\scriptsize$\pm$\cellcolor{lightpink}{0.01}}
& 92.00{\scriptsize$\pm$0.20} & 68.50{\scriptsize$\pm$0.05} & 78.53{\scriptsize$\pm$\cellcolor{lightpink}{0.07}}
& 98.83{\scriptsize$\pm$0.15} & 71.49{\scriptsize$\pm$0.00} & 82.96{\scriptsize$\pm$\cellcolor{lightpink}{0.05}}
& 86.65{\scriptsize$\pm$0.00} & 62.56{\scriptsize$\pm$0.03} & 72.66{\scriptsize$\pm$\cellcolor{lightpink}{0.02}}
& 65.51{\scriptsize$\pm$0.08} & 18.58{\scriptsize$\pm$0.11} & 28.95{\scriptsize$\pm$\cellcolor{lightpink}{0.01}} \\

FedPHA
& 92.61{\scriptsize$\pm$0.70} & 50.29{\scriptsize$\pm$2.31} & 65.14{\scriptsize$\pm$\cellcolor{lightpink}{1.82}}
& 98.81{\scriptsize$\pm$0.63} & 70.06{\scriptsize$\pm$2.75} & 81.95{\scriptsize$\pm$\cellcolor{lightpink}{1.89}}
& 91.90{\scriptsize$\pm$0.55} & 44.52{\scriptsize$\pm$1.31} & 59.96{\scriptsize$\pm$\cellcolor{lightpink}{1.07}}
& 98.73{\scriptsize$\pm$0.21} & 49.77{\scriptsize$\pm$2.86} & 66.12{\scriptsize$\pm$\cellcolor{lightpink}{2.53}}
& \textbf{90.49{\scriptsize$\pm$0.14}} & 39.71{\scriptsize$\pm$2.26} &55.20{\scriptsize$\pm$\cellcolor{lightpink}{2.33}}
& 66.62{\scriptsize$\pm$0.44} & 13.00{\scriptsize$\pm$2.97} & 21.59{\scriptsize$\pm$\cellcolor{lightpink}{4.12}} \\

FedOTP
& 82.79{\scriptsize$\pm$1.05} & 27.04{\scriptsize$\pm$2.30} & 40.70{\scriptsize$\pm$\cellcolor{lightpink}{2.58}} 
& 95.20{\scriptsize$\pm$1.10} & 51.47{\scriptsize$\pm$4.08} & 66.75{\scriptsize$\pm$\cellcolor{lightpink}{3.74}}
& 78.52{\scriptsize$\pm$1.64} & 14.67{\scriptsize$\pm$1.36} & 24.70{\scriptsize$\pm$\cellcolor{lightpink}{1.95}}
& 95.84{\scriptsize$\pm$0.39} & 35.53{\scriptsize$\pm$4.86} & 51.64{\scriptsize$\pm$\cellcolor{lightpink}{5.00}}
& 81.23{\scriptsize$\pm$0.91} & 14.23{\scriptsize$\pm$1.50} & 24.22{\scriptsize$\pm$\cellcolor{lightpink}{1.62}}
& 51.95{\scriptsize$\pm$1.23} & 5.98{\scriptsize$\pm$1.11} & 10.68{\scriptsize$\pm$\cellcolor{lightpink}{1.77}} \\

PromptFL
& 80.60{\scriptsize$\pm$1.26} & 70.83{\scriptsize$\pm$2.15} & 75.36{\scriptsize$\pm$\cellcolor{lightpink}{0.75}}
& 95.37{\scriptsize$\pm$0.61} & 96.33{\scriptsize$\pm$0.89} & 95.84{\scriptsize$\pm$\cellcolor{lightpink}{0.71}}
& 67.89{\scriptsize$\pm$2.24} & 70.14{\scriptsize$\pm$0.13} & 69.94{\scriptsize$\pm$\cellcolor{lightpink}{1.24}}
& 96.24{\scriptsize$\pm$0.32} & 71.23{\scriptsize$\pm$2.91} & 81.83{\scriptsize$\pm$\cellcolor{lightpink}{1.82}}
& 73.38{\scriptsize$\pm$0.40} & 67.75{\scriptsize$\pm$2.00} & 70.45{\scriptsize$\pm$\cellcolor{lightpink}{1.30}}
& 27.56{\scriptsize$\pm$7.64} & 20.44{\scriptsize$\pm$10.99} & 22.79{\scriptsize$\pm$\cellcolor{lightpink}{10.69}} \\

FedPGP
& 90.88{\scriptsize$\pm$1.02} & 55.51{\scriptsize$\pm$0.15} & 68.92{\scriptsize$\pm$\cellcolor{lightpink}{0.20}}
& 95.59{\scriptsize$\pm$0.35} & 97.29{\scriptsize$\pm$0.35} & 96.43{\scriptsize$\pm$\cellcolor{lightpink}{0.21}}
& 87.82{\scriptsize$\pm$0.73} & 57.87{\scriptsize$\pm$3.16} & 69.71{\scriptsize$\pm$\cellcolor{lightpink}{2.23}}
& 98.40{\scriptsize$\pm$0.78} & 65.43{\scriptsize$\pm$2.11} & 78.58{\scriptsize$\pm$\cellcolor{lightpink}{1.71}}
& 89.21{\scriptsize$\pm$0.89} & 45.93{\scriptsize$\pm$1.01} & 60.64{\scriptsize$\pm$\cellcolor{lightpink}{1.00}}
& 61.37{\scriptsize$\pm$1.33} & 21.03{\scriptsize$\pm$0.62} & 31.31{\scriptsize$\pm$\cellcolor{lightpink}{0.55}} \\

PromptFolio
& 90.19{\scriptsize$\pm$0.59} & 52.97{\scriptsize$\pm$3.98} & 66.64{\scriptsize$\pm$\cellcolor{lightpink}{3.00}}
& 97.94{\scriptsize$\pm$0.13} & 71.87{\scriptsize$\pm$4.42} & 82.83{\scriptsize$\pm$\cellcolor{lightpink}{2.99}}
& 73.42{\scriptsize$\pm$3.79} & 54.88{\scriptsize$\pm$5.98} & 62.76{\scriptsize$\pm$\cellcolor{lightpink}{5.31}}
& 97.94{\scriptsize$\pm$0.22} & 27.59{\scriptsize$\pm$6.74} & 42.40{\scriptsize$\pm$\cellcolor{lightpink}{8.25}}
& 85.74{\scriptsize$\pm$1.28} & 50.81{\scriptsize$\pm$4.93} & 62.30{\scriptsize$\pm$\cellcolor{lightpink}{4.15}}
& 45.12{\scriptsize$\pm$2.50} & 7.74{\scriptsize$\pm$0.48} & 13.21{\scriptsize$\pm$\cellcolor{lightpink}{0.81}}\\
\midrule

{BLEND}
& \textbf{95.03{\scriptsize$\pm$0.16}} & \textbf{78.20{\scriptsize$\pm$0.21}} & \textbf{85.80{\scriptsize$\pm$\cellcolor{lightpink}{0.19}}}
& \textbf{99.21{\scriptsize$\pm$0.24}} & 96.99{\scriptsize$\pm$0.03} & \textbf{98.09{\scriptsize$\pm$\cellcolor{lightpink}{0.11}}}
& \textbf{93.72{\scriptsize$\pm$0.20}} & 70.20{\scriptsize$\pm$0.26} & \textbf{80.27{\scriptsize$\pm$\cellcolor{lightpink}{0.23}}}
& \textbf{98.94{\scriptsize$\pm$0.06}} & 75.21{\scriptsize$\pm$0.36} & \textbf{85.45{\scriptsize$\pm$\cellcolor{lightpink}{0.21}}}
&  87.83{\scriptsize$\pm$0.13} & 66.31{\scriptsize$\pm$0.22} & \textbf{75.56{\scriptsize$\pm$\cellcolor{lightpink}{0.19}} }
& \textbf{70.05{\scriptsize$\pm$0.61}} & 33.24{\scriptsize$\pm$0.47} & \textbf{45.09{\scriptsize$\pm$\cellcolor{lightpink}{0.55}}} \\
\midrule
$\Delta$ ({$p-$value})
&  &  & \cellcolor{lightpink}{\textcolor{darkgreen}{\textbf{+5.40}}} ({0.025})
&  &  & \cellcolor{lightpink}{\textcolor{darkgreen}{\textbf{+1.66}}} ({0.009})
&  &  & \cellcolor{lightpink}{\textcolor{darkgreen}{\textbf{+1.74}}} ({0.014})
&  &  & \cellcolor{lightpink}{\textcolor{darkgreen}{\textbf{+2.49}}} ({0.003})
&  &  & \cellcolor{lightpink}{\textcolor{darkgreen}{\textbf{+2.90}}} ({0.001})
&  &  & \cellcolor{lightpink}{\textcolor{darkgreen}{\textbf{+8.99}}} ({0.003})\\
\hline
\hline

& \multicolumn{3}{c|}{\cellcolor{gray!20}{Food101}}
& \multicolumn{3}{c|}{\cellcolor{gray!20}{DTD}}
& \multicolumn{3}{c|}{\cellcolor{gray!20}{EuroSAT}}
& \multicolumn{3}{c|}{\cellcolor{gray!20}{Caltech}}
& \multicolumn{3}{c|}{\textbf{\cellcolor{gray!20}{Cross-Dataset Avg. ($\uparrow$)}}}
& \multicolumn{3}{c}{\textbf{\cellcolor{gray!20}{Cross-Dataset  Std. ($\downarrow$)}}} \\
\cmidrule(lr){2-4}\cmidrule(lr){5-7}\cmidrule(lr){8-10}\cmidrule(lr){11-13}\cmidrule(lr){14-16}\cmidrule(lr){17-19}
& Per.& Gen. & \cellcolor{lightpink}{H-Mean} & Per.& Gen. & \cellcolor{lightpink}{H-Mean} & Per.& Gen. & \cellcolor{lightpink}{H-Mean} &Per.& Gen. & \cellcolor{lightpink}{H-Mean} & Per.& Gen. & \cellcolor{lightpink}{H-Mean} & Per.& Gen. & \cellcolor{lightpink}{H-Mean} \\
\midrule

Zero-shot CLIP
& 89.44{\scriptsize$\pm$0.00} & 90.67{\scriptsize$\pm$0.00} & 90.05{\scriptsize$\pm$\cellcolor{lightpink}{0.00}}
& 54.75{\scriptsize$\pm$0.00} & 58.94{\scriptsize$\pm$0.00} & 56.76{\scriptsize$\pm$\cellcolor{lightpink}{0.00}}
& 50.21{\scriptsize$\pm$0.00} & 69.87{\scriptsize$\pm$0.00} & 58.43{\scriptsize$\pm$\cellcolor{lightpink}{0.00}}
& 97.16{\scriptsize$\pm$0.00} & \textbf{94.21{\scriptsize$\pm$0.00}} & 95.66{\scriptsize$\pm$\cellcolor{lightpink}{0.00}}
& 68.22{\scriptsize$\pm$0.00} & 73.85{\scriptsize$\pm$0.00} & 70.74{\scriptsize$\pm$\cellcolor{lightpink}{0.00}}
& 19.86{\scriptsize$\pm$0.00} & 17.81{\scriptsize$\pm$0.00} & 18.72{\scriptsize$\pm$\cellcolor{lightpink}{0.00}}\\

MMA+FedAvg
& 89.85{\scriptsize$\pm$0.33} & 90.21{\scriptsize$\pm$0.67} & 90.03{\scriptsize$\pm$\cellcolor{lightpink}{0.50}}
& 79.30{\scriptsize$\pm$1.11} & 60.39{\scriptsize$\pm$1.80} & 68.56{\scriptsize$\pm$\cellcolor{lightpink}{1.50}}
& 92.05{\scriptsize$\pm$1.96} & 70.49{\scriptsize$\pm$5.21} & 79.67{\scriptsize$\pm$\cellcolor{lightpink}{2.61}}
& 97.60{\scriptsize$\pm$0.30} & 93.30{\scriptsize$\pm$0.52} & 95.40{\scriptsize$\pm$\cellcolor{lightpink}{0.17}}
& 82.69{\scriptsize$\pm$0.86} & 73.14{\scriptsize$\pm$1.21} & 77.34{\scriptsize$\pm$\cellcolor{lightpink}{0.79}} 
& 16.57{\scriptsize$\pm$0.75} & 17.49{\scriptsize$\pm$1.44} & 16.46{\scriptsize$\pm$\cellcolor{lightpink}{0.76}}\\

pFedMMA
& 95.80{\scriptsize$\pm$0.19} & 86.83{\scriptsize$\pm$0.40} & 91.10{\scriptsize$\pm$\cellcolor{lightpink}{0.21}}
& 77.87{\scriptsize$\pm$2.04} & 57.81{\scriptsize$\pm$1.34} & 66.34{\scriptsize$\pm$\cellcolor{lightpink}{1.13}}
& 80.06{\scriptsize$\pm$15.74} & 49.22{\scriptsize$\pm$5.28} & 60.81{\scriptsize$\pm$\cellcolor{lightpink}{8.69}}
& 97.79{\scriptsize$\pm$0.31} & 92.95{\scriptsize$\pm$0.11} & 95.31{\scriptsize$\pm$\cellcolor{lightpink}{0.20}}
& 77.99{\scriptsize$\pm$3.03} & 66.69{\scriptsize$\pm$1.12} & 71.48{\scriptsize$\pm$\cellcolor{lightpink}{1.58}} 
& 19.01{\scriptsize$\pm$4.52} & 19.93{\scriptsize$\pm$1.48} & 19.20{\scriptsize$\pm$\cellcolor{lightpink}{2.52}} \\

pFedMoAP
& \textbf{96.28{\scriptsize$\pm$0.08}} & 86.68{\scriptsize$\pm$0.00} & 91.23{\scriptsize$\pm$\cellcolor{lightpink}{0.04}}
& \textbf{92.43{\scriptsize$\pm$0.42}} & 47.72{\scriptsize$\pm$0.00} & 62.94{\scriptsize$\pm$\cellcolor{lightpink}{0.10}}
& 98.05{\scriptsize$\pm$1.25} & 63.47{\scriptsize$\pm$0.00} & 77.05{\scriptsize$\pm$\cellcolor{lightpink}{0.39}}
& 97.98{\scriptsize$\pm$0.18} & 93.69{\scriptsize$\pm$0.00} & 95.79{\scriptsize$\pm$\cellcolor{lightpink}{0.08}}
& 92.03{\scriptsize$\pm$0.25} & 66.48{\scriptsize$\pm$0.00} & 75.98{\scriptsize$\pm$\cellcolor{lightpink}{0.08}} 
& 9.60{\scriptsize$\pm$0.35} & 20.46{\scriptsize$\pm$0.00} & 18.22{\scriptsize$\pm$\cellcolor{lightpink}{0.11}} \\

FedPHA
& 95.87{\scriptsize$\pm$0.40} & 67.45{\scriptsize$\pm$4.86} & 79.08{\scriptsize$\pm$\cellcolor{lightpink}{3.33}}
& 90.41{\scriptsize$\pm$0.44} & 34.01{\scriptsize$\pm$2.07} & 49.39{\scriptsize$\pm$\cellcolor{lightpink}{2.17}}
& \textbf{100.00{\scriptsize$\pm$0.00}} & 28.33{\scriptsize$\pm$5.81} & 43.84{\scriptsize$\pm$\cellcolor{lightpink}{6.88}}
& \textbf{98.09{\scriptsize$\pm$0.35}} & 81.83{\scriptsize$\pm$1.89} & 89.21{\scriptsize$\pm$\cellcolor{lightpink}{0.98}}
& 92.35{\scriptsize$\pm$0.37} & 47.90{\scriptsize$\pm$2.68} & 61.15{\scriptsize$\pm$\cellcolor{lightpink}{2.48}} 
& 9.25{\scriptsize$\pm$0.23} & 19.76{\scriptsize$\pm$1.58} & 19.03{\scriptsize$\pm$\cellcolor{lightpink}{1.85}} \\

FedOTP
& 90.33{\scriptsize$\pm$0.14} & 44.62{\scriptsize$\pm$1.31} & 59.73{\scriptsize$\pm$\cellcolor{lightpink}{1.21}}
& 88.52{\scriptsize$\pm$0.64} & 34.69{\scriptsize$\pm$4.02} & 49.72{\scriptsize$\pm$\cellcolor{lightpink}{4.30}}
& \textbf{100.00{\scriptsize$\pm$0.00}} & 31.75{\scriptsize$\pm$2.04} & 48.17{\scriptsize$\pm$\cellcolor{lightpink}{2.37}}
& 94.79{\scriptsize$\pm$0.54} & 65.70{\scriptsize$\pm$2.78} & 77.57{\scriptsize$\pm$\cellcolor{lightpink}{1.91}}
& 85.92{\scriptsize$\pm$0.67} & 32.57{\scriptsize$\pm$2.39} & 45.39{\scriptsize$\pm$\cellcolor{lightpink}{2.48}} 
& 13.15{\scriptsize$\pm$0.54} & \textbf{17.35{\scriptsize$\pm$1.46}} & 19.66{\scriptsize$\pm$\cellcolor{lightpink}{1.42}} \\

PromptFL
& 90.54{\scriptsize$\pm$0.22} & 90.52{\scriptsize$\pm$0.54} & 90.53{\scriptsize$\pm$\cellcolor{lightpink}{0.31}}
& 77.01{\scriptsize$\pm$0.25} & 51.17{\scriptsize$\pm$2.55} & 61.44{\scriptsize$\pm$\cellcolor{lightpink}{1.75}}
& 81.72{\scriptsize$\pm$4.28} & 65.91{\scriptsize$\pm$4.69} & 72.76{\scriptsize$\pm$\cellcolor{lightpink}{2.62}}
& 97.17{\scriptsize$\pm$0.38} & 93.23{\scriptsize$\pm$0.53} & 95.16{\scriptsize$\pm$\cellcolor{lightpink}{0.18}}
& 78.75{\scriptsize$\pm$1.72} & 69.95{\scriptsize$\pm$2.54} & 73.61{\scriptsize$\pm$\cellcolor{lightpink}{2.01}} 
& 19.60{\scriptsize$\pm$2.33} & 21.20{\scriptsize$\pm$3.15} & 20.18{\scriptsize$\pm$\cellcolor{lightpink}{3.00}} \\

FedPGP
& 91.58{\scriptsize$\pm$0.46} & 89.81{\scriptsize$\pm$0.22} & 90.69{\scriptsize$\pm$\cellcolor{lightpink}{0.34}}
& 81.99{\scriptsize$\pm$2.52} & 50.76{\scriptsize$\pm$3.56} & 62.62{\scriptsize$\pm$\cellcolor{lightpink}{2.78}}
& 84.73{\scriptsize$\pm$3.08} & 67.67{\scriptsize$\pm$8.88} & 74.90{\scriptsize$\pm$\cellcolor{lightpink}{5.96}}
&  97.21{\scriptsize$\pm$0.38} & 92.55{\scriptsize$\pm$0.34} & 94.83{\scriptsize$\pm$\cellcolor{lightpink}{0.26}}
& 87.88{\scriptsize$\pm$1.06} & 64.38{\scriptsize$\pm$1.94} & 72.86{\scriptsize$\pm$\cellcolor{lightpink}{1.42}}
& 10.14{\scriptsize$\pm$0.95} & 22.50{\scriptsize$\pm$2.63} & 18.44{\scriptsize$\pm$\cellcolor{lightpink}{1.77}} \\

PromptFolio
& 93.05{\scriptsize$\pm$0.38} & 53.00{\scriptsize$\pm$6.18} & 67.33{\scriptsize$\pm$\cellcolor{lightpink}{5.12}}
& 84.86{\scriptsize$\pm$0.36} & 42.30{\scriptsize$\pm$4.44} & 56.33{\scriptsize$\pm$\cellcolor{lightpink}{4.04}}
& 92.24{\scriptsize$\pm$1.83} & 65.03{\scriptsize$\pm$6.62} & 75.79{\scriptsize$\pm$\cellcolor{lightpink}{3.94}}
& 97.00{\scriptsize$\pm$0.12} & 82.21{\scriptsize$\pm$5.02} & 88.91{\scriptsize$\pm$\cellcolor{lightpink}{2.89}}
& 85.75{\scriptsize$\pm$1.10} & 50.84{\scriptsize$\pm$4.87} & 61.85{\scriptsize$\pm$\cellcolor{lightpink}{4.04}}
& 16.13{\scriptsize$\pm$1.23} & 21.41{\scriptsize$\pm$1.82} & 21.58{\scriptsize$\pm$\cellcolor{lightpink}{1.96}}\\
\midrule
{BLEND}
& 95.24{\scriptsize$\pm$0.16} & \textbf{91.25{\scriptsize$\pm$0.11}} & \textbf{93.20{\scriptsize$\pm$\cellcolor{lightpink}{0.04}}}
& 91.64{\scriptsize$\pm$1.07} & \textbf{64.32{\scriptsize$\pm$0.70}} & \textbf{75.59{\scriptsize$\pm$\cellcolor{lightpink}{0.84}}}
& 99.36{\scriptsize$\pm$0.28} & \textbf{76.50{\scriptsize$\pm$0.90}} & \textbf{86.44{\scriptsize$\pm$\cellcolor{lightpink}{0.65}}}
& \textbf{98.09{\scriptsize$\pm$0.09}} & 94.01{\scriptsize$\pm$0.24} & \textbf{96.01{\scriptsize$\pm$\cellcolor{lightpink}{0.09}}}
& \textbf{92.91{\scriptsize$\pm$0.30}} & \textbf{74.62{\scriptsize$\pm$0.35}} & \textbf{82.15{\scriptsize$\pm$\cellcolor{lightpink}{0.31}}} 
& \textbf{8.40{\scriptsize$\pm$0.30}} & 17.54{\scriptsize$\pm$0.25} & \textbf{14.40{\scriptsize$\pm$\cellcolor{lightpink}{0.26}}} \\
\midrule 
$\Delta$ ({$p-$value})
&  &  & \cellcolor{lightpink}{\textcolor{darkgreen}{\textbf{+1.97}}} ({0.001})
&  &  & \cellcolor{lightpink}{\textcolor{darkgreen}{\textbf{+7.03}}} ({0.005})
&  &  & \cellcolor{lightpink}{\textcolor{darkgreen}{\textbf{+6.77}}} ({0.033})
&  &  & \cellcolor{lightpink}{\textcolor{darkgreen}{\textbf{+0.22}}} ({0.011})
&  &  & \cellcolor{lightpink}{\textcolor{darkgreen}{\textbf{+4.81}}} 
&  &  & \cellcolor{lightpink}{\textcolor{darkgreen}{\textbf{-2.06}}} \\
\bottomrule
\end{tabular}
}
\vspace{-7mm}
\end{table*}


\vspace{-1.5mm}
\noindent \underline{\textbf{Main Results:}} In our main experiments, we set $\alpha^{\mathsf{Dir}}=0.01$, reflecting highly non-IID data across clients (other choices of $\alpha^{\mathsf{Dir}}$ are studied in \textit{Ablation Study 1}).  Table \ref{tab:main_results} compares {BLEND} with the baselines after $50$ global rounds (convergence across all methods is reached). 
Across all datasets, {BLEND} consistently achieves the best H-Mean (i.e., the best balance of personalization and generalization), with a notable average improvement of $+4.81$, while also exhibiting the lowest performance variation across the datasets (i.e., higher consistency), as indicated by its smallest standard deviation (Std.) of accuracy across the datasets, which is $2.06\%$ lower than the best baseline.  Further, Appendix~\ref{ap:overhead} reports the \textbf{communication/computation overhead} of {BLEND} compared to the baselines, showing that {BLEND} achieves higher H-Mean performance with comparable/lower resource overhead.


\vspace{-1.5mm}
We summarize BLEND's performance variation under increasing task difficulty in Appendix~\ref{app:performance_variation}. The results show larger improvements over the best baseline as data heterogeneity, the number of clients, and local data scarcity increase. Below, we present a set of ablation studies, where all results are averaged over the $10$ datasets. Detailed ablation studies for each dataset are in the Appendix sections.  

\vspace{-1.5mm}
\noindent \textbf{\textit{\underline{(Ablation Study 1)}} Data heterogeneity variations:}
In addition to the {\small$\alpha^{\mathsf{Dir}} = 0.01$} used above, we further evaluate the impact of data heterogeneity by varying {\small$\alpha^{\mathsf{Dir}} \in \{0.1, 10\}$} in Appendix~\ref{ap:data_hete}.  Further, Appendix~\ref{ap:real_large} evaluates {BLEND} on iNaturalist~\cite{hsu2020federated}, a distributed dataset with \textit{real-world data distributions} exhibiting extreme non-IID characteristics. Across all settings, {BLEND} maintains superior H-Mean performance, showing robustness to diverse data distributions.

\vspace{-1.75mm}
\noindent \textbf{\textit{\underline{(Ablation Study 2)}} Varying the number of clients:}
In addition to the {\small $N=10$} setting used above, we examine {BLEND} under different system sizes {\small$N {\in} \{20, 30\}$} in Appendix~\ref{ap:num_clients}. Also, Appendix~\ref{ap:real_large} evaluates {BLEND} in real-world large-scale settings with {\small$N{\in}\{100, 200\}$}. The results in Appendix~\ref{ap:num_clients} and Appendix~\ref{ap:real_large} show that {BLEND} outperforms the baselines across various system scales, demonstrating its scalability.

\vspace{-1.75mm}
\noindent \textbf{\textit{\underline{(Ablation Study 3)}} Different few-shot settings:}
In Appendix~\ref{ap:few_shots}, we evaluate the few-shot learning capability of {BLEND} under varying numbers of shots per class, showing that {BLEND} consistently outperforms baselines across different data scarcity configurations. 
\newcommand{\Rstretch}{\rule{0pt}{2.412ex}}

\setlength{\tabcolsep}{1.6pt}

 \begin{wraptable}{r}{0.25\linewidth}
\vspace{-22pt}
\centering
\scriptsize
\caption{Effect of the rank $r {=} r_x {=} r_w$ on performance.}
\label{tbl:global_adapter}
\vspace{-3pt}
\begin{tabular}{c|ccc}
\hline
\rowcolor{gray!20}
Rank ($r$) & Per. & Gen. & H-Mean \\
\hline
No adapter & 89.50 & 72.31 & \cellcolor{lightpink}{79.95} \\
4   & 92.25 & 74.69 & \cellcolor{lightpink}{82.05} \\
8   & 92.45 & \textbf{74.70} & \cellcolor{lightpink}{82.10} \\
16  & 92.91 & 74.62 & \cellcolor{lightpink}{\textbf{82.15}} \\
32  & \textbf{92.93} & 74.39 & \cellcolor{lightpink}{81.97} \\
64  & 92.91 & 74.09 & \cellcolor{lightpink}{81.74} \\
128 & 92.80 & 73.70 & \cellcolor{lightpink}{81.43} \\
256 & 92.68 & 73.47 & \cellcolor{lightpink}{81.20} \\
Full-rank & 92.48 & 73.63 & \cellcolor{lightpink}{81.26} \\
\hline
\end{tabular}

\vspace{-10pt}
\end{wraptable}

\vspace{-1.75mm}
\noindent \textbf{\textit{\underline{(Ablation Study 4)}} Different bottleneck sizes for global adapters:}
Table~\ref{tbl:global_adapter} analyzes the impact of the global adapter bottleneck rank $r {=} r_x {=} r_w$ on average performance across datasets (Appendix~\ref{ap:bottle_global} contains the per-dataset results). 
As shown, incorporating global adapters consistently outperforms the `No adapter' setting and exhibits robustness across different adapter sizes. Moreover, the H-Mean increases with $r$ up to $16$, showing that strong performance can be achieved even under low parameter-update budgets in the federated learning process.

\vspace{-1.75mm}
\noindent \textbf{\textit{\underline{(Ablation Study 5)}} Different bottleneck sizes for the personalized projection adapter:}
Table~\ref{tbl:personalized_adapter} analyzes the impact of the personalized adapter projection rank $r_p$ on average performance across datasets (Appendix~\ref{ap:bottle_per} contains the per-dataset results). As $r_p$ increases, the personalization accuracy  improves (i.e., higher-capacity personalized adapters can capture client-specific characteristics), while the generalization gradually degrades. Overall, the H-Mean improves progressively and peaks at $r_p=256$, reflecting the best personalization–generalization trade-off. 


\begin{wraptable}{r}{0.24\linewidth}
\vspace{-13pt}
\centering
\scriptsize
\caption{Effect of the rank $r_p$ on performance.}
\label{tbl:personalized_adapter}
\vspace{-4pt}
\begin{tabular}{c|ccc}
\hline
\rowcolor{gray!20}
Rank ($r_p$) & Per. & Gen. & H-Mean \\
\hline
4   & 89.99 & \textbf{75.00} & \cellcolor{lightpink}{81.44} \\
8   & 90.45 & 74.83 & \cellcolor{lightpink}{81.73} \\
16  & 91.00 & 74.86 & \cellcolor{lightpink}{81.73} \\
32  & 91.59 & 74.91 & \cellcolor{lightpink}{81.96} \\
64  & 92.12 & 74.64 & \cellcolor{lightpink}{81.92} \\
128 & 92.57 & 74.71 & \cellcolor{lightpink}{82.12} \\
256 & 92.91 & 74.62 & \cellcolor{lightpink}{\textbf{82.15}} \\
Full-rank & \textbf{92.95} & 74.46 & \cellcolor{lightpink}{82.08} \\
\hline
\end{tabular}

\vspace{-10pt}
\end{wraptable}
\vspace{-1.5mm}
\noindent \textbf{\textit{\underline{(Ablation Study 6)}} Adapter injection at different layers:}
In the left plot of Fig.~\ref{fig:gamma_lambda}, we evaluate the effect of injecting the global adapters at different transformer layers on average performance across datasets (Appendix~\ref{ap:adapter_injection} contains the per-dataset results). As shown, the placement of adapters across transformer layers has a notable impact on performance. As expected, adding adapters to all layers ($1-12$, i.e., $\ell^{\mathsf{S}}=1$ and $\ell^{\mathsf{E}}=12$) yields the highest accuracy; however, the H-Mean achieved by the $5-12$ configuration is very close to that of $1-12$ setting, indicating that adapter injection in all layers is not strictly necessary to obtain good performance. 

\vspace{-1.5mm}
\noindent \textbf{\textit{\underline{(Ablation Study 7)}} Effect of $\gamma$ on performance:}
In the middle plot of Fig.~\ref{fig:gamma_lambda}, we study the effect of varying $\gamma$ in~\eqref{eq:fusion_rep} on the average performance across datasets (per-dataset results and individual personalization/generalization accuracies are provided in Appendix~\ref{ap:gamma}). We see that 
 $\gamma = 0.09$ yields the highest H-Mean, while larger values lead to a worse personalization–generalization trade-off.

\begin{figure}[t]
    \centering
    \begin{subfigure}[t]{0.27\linewidth}
        \centering
        \includegraphics[width=\linewidth, trim = 0 10 0 0, clip]{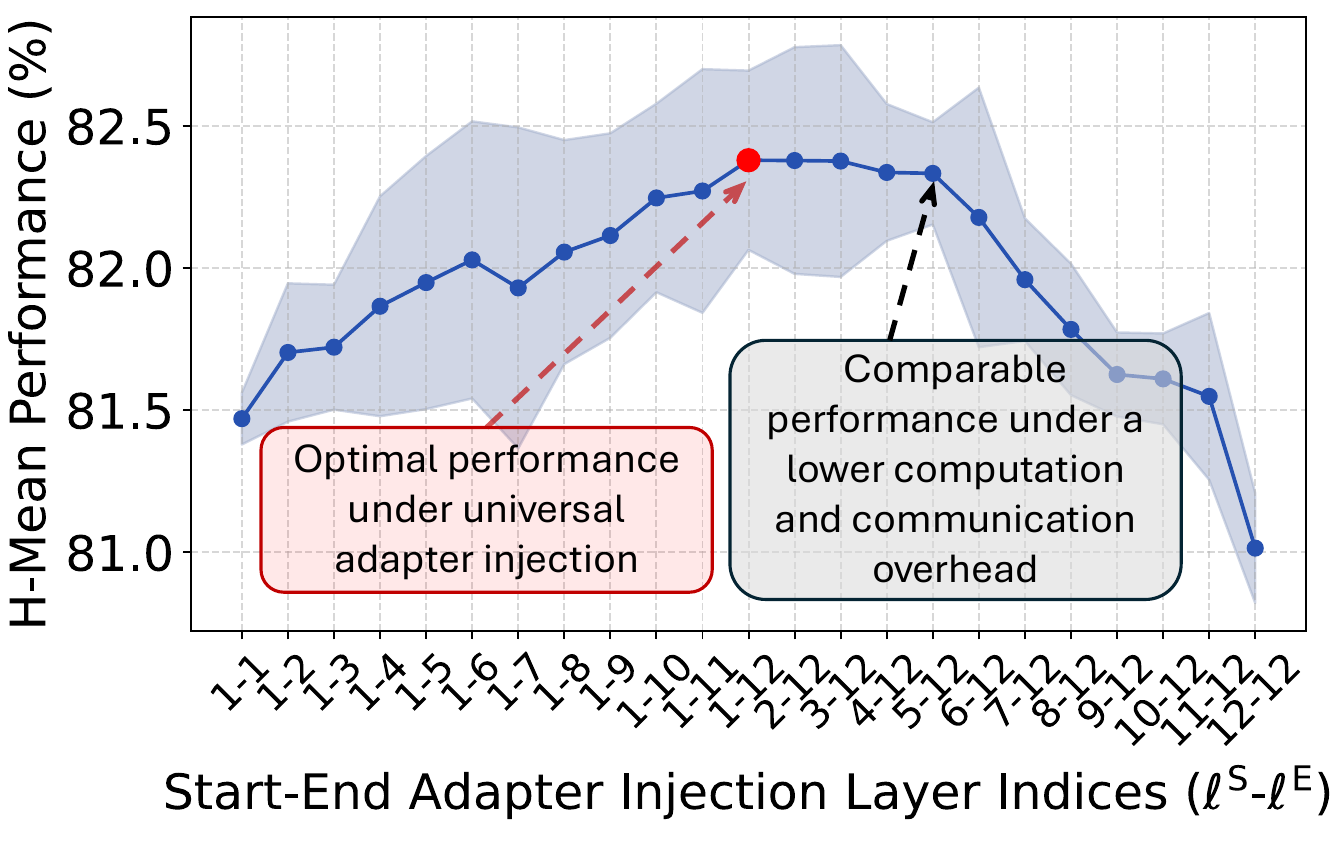}
    \end{subfigure}
    \begin{subfigure}[t]{0.265\linewidth}
        \centering
        \includegraphics[width=\linewidth]{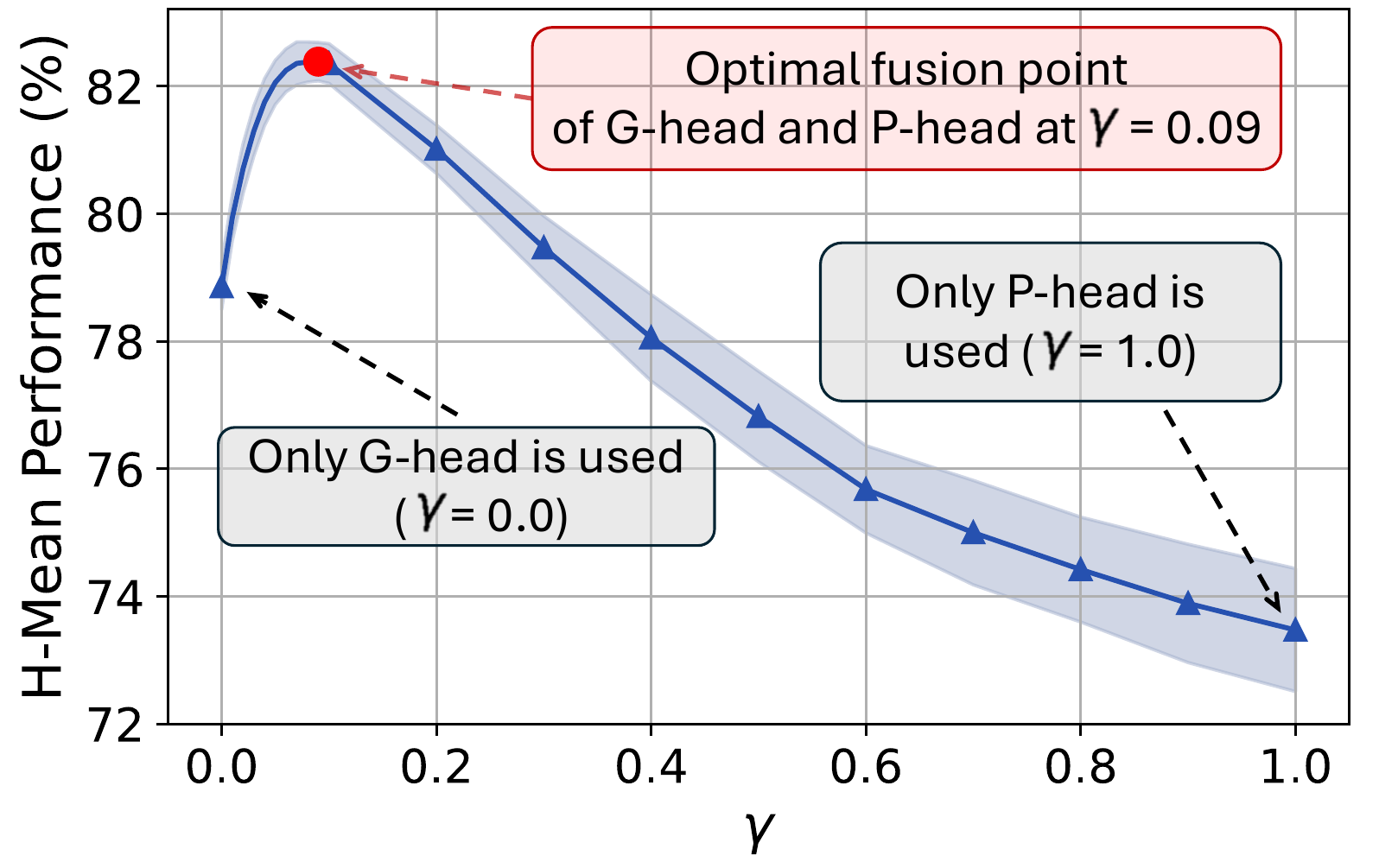}
    \end{subfigure}
    \begin{subfigure}[t]{0.27\linewidth}
        \centering
        \includegraphics[width=\linewidth]{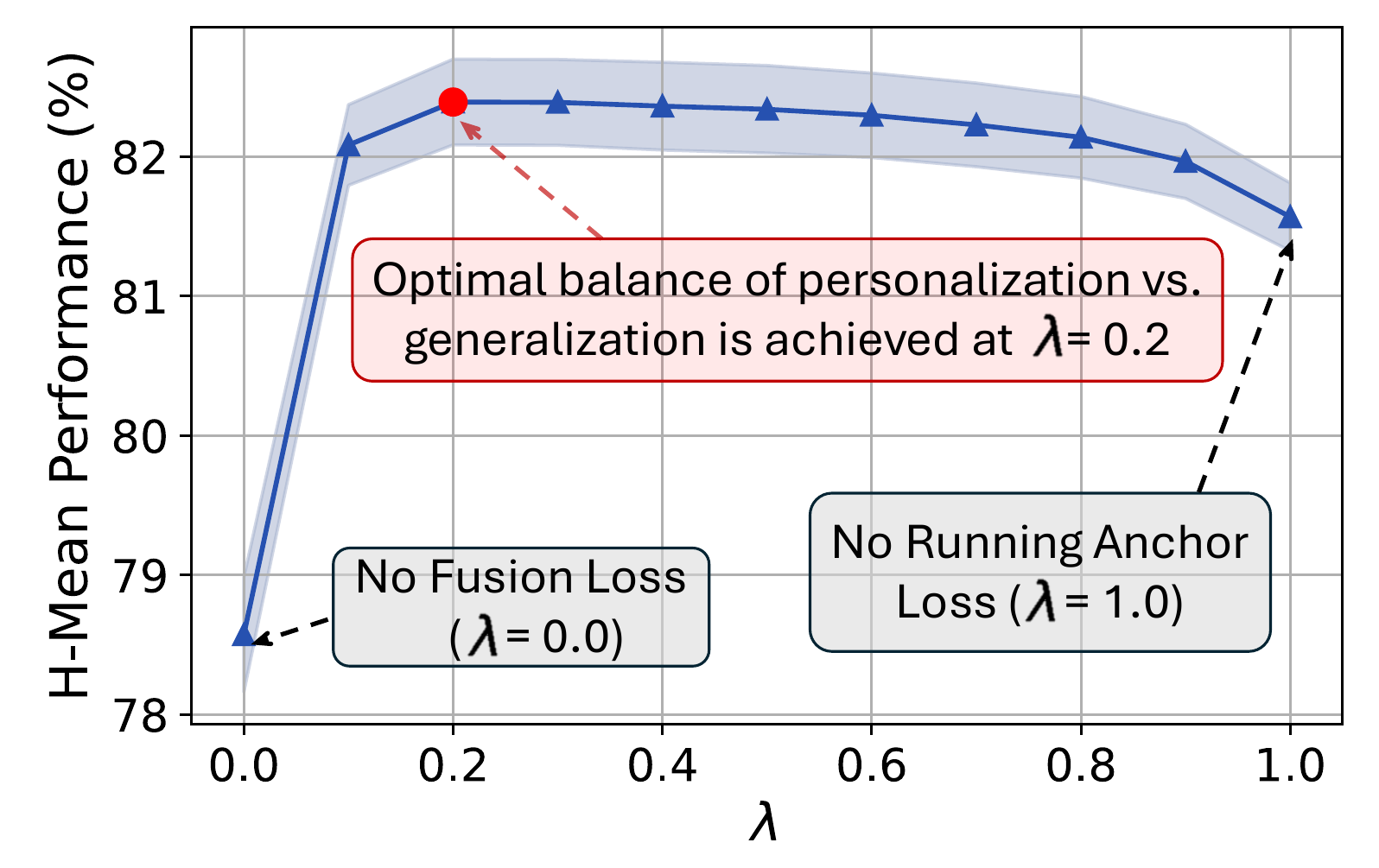}
    \end{subfigure}
    \vspace{-1mm}
    \caption{Effect of injecting adapters at various transformer layers (left plot). Effect of varying $\gamma$ (middle plot) and $\lambda$ (right plot), which are used in \eqref{eq:fusion_rep} and \eqref{eq:final_loss}, respectively, on the performance.}
    \label{fig:gamma_lambda}
    \vspace{-5.85mm}
\end{figure}

\vspace{-1.5mm}
\noindent \textbf{\textit{\underline{(Ablation Study 8)}} Effect of $\lambda$ on performance:}
In the right plot of Fig.~\ref{fig:gamma_lambda}, we examine the effect of varying $\lambda$ in~\eqref{eq:final_loss} on the average performance across datasets (per-dataset results and individual personalization/ generalization accuracies are provided in Appendix~\ref{ap:lambda}). We observe that $\lambda = 0.2$ achieves the highest H-Mean. Further, the results in Appendix~\ref{ap:lambda} (see Fig.~\ref{fig:lambda_avg_datasets}) reveal that choosing $\lambda$ too high weakens the anchoring to the zero-shot model, yielding the worst generalization performance.

\vspace{-1.5mm}
\noindent \textbf{\textit{\underline{(Ablation Study 9)}} Effect of aggregation of various components:}
Table~\ref{tbl:aggregation} examines how aggregating different trainable components  (i.e., global vision/text adapters and personalized adapter projection) affects average performance across datasets (Appendix~\ref{ap:aggregation} contains the per-dataset
results).
\begin{wraptable}{r}{0.27\linewidth}
\vspace{-12pt}
\centering
\scriptsize
\caption{Comparison of different aggregation strategies.}
\label{tbl:aggregation}
\vspace{-4pt}
\begin{minipage}{1\linewidth}
\centering
\begin{tabular}{c|ccc}
\hline
\rowcolor{gray!20} \textbf{Agg. Strategy} & Per. & Gen. & H-Mean \\
\hline
{BLEND} & 92.91 & 74.62 & \cellcolor{lightpink}{\textbf{82.15}} \\
Full & 83.35 & \textbf{75.41} & \cellcolor{lightpink}{79.01} \\
NoAgg & 89.12 & 69.71 & \cellcolor{lightpink}{77.71} \\
Only Vision & 74.64 & 70.02 & \cellcolor{lightpink}{71.82} \\
Only Text & \textbf{93.84} & 67.31 & \cellcolor{lightpink}{77.04} \\
\hline
\end{tabular}
\end{minipage}

\vspace{-10pt}
\end{wraptable}
{BLEND} (i.e., aggregating all global vision/text adapters and keeping the personalized projection adapter at each client for local adaptation) delivers the best overall H-Mean performance. Also, aggregating all trainable parameters (`Full') improves generalization but substantially harms personalization, while keeping all global adapters at clients (`NoAgg') preserves personalization to some extent but leads to a notable drop in generalization. Interestingly, even though `NoAgg' uses all adapters for personalization, it does \textit{not} achieve the highest personalization performance, highlighting the importance of effective cross-client knowledge sharing during model aggregation. 


\newcommand{\Rstret}{\rule{0pt}{2.572ex}}


\vspace{-1.5mm}
\noindent \textbf{\textit{\underline{(Ablation Study 10)}} Effect of trainable text projection head:} Appendix~\ref{ap:trainable_txt_encoder} shows that unfreezing the text projection head mostly degrades the H-Mean performance, as training the text projection head causes forgetting of the pre-trained CLIP knowledge.



\vspace{-1.5mm}
\noindent \textbf{\textit{\underline{(Ablation Study 11)}} Effect of running anchor:}
We evaluate BLEND's running anchor loss against two regularization strategies: a frozen zero-shot G-head and a KL-distilled G-head. The running anchor achieves a better personalization–generalization balance, as detailed in Appendix~\ref{app:running_anchor}.

\vspace{-1.5mm}
\noindent \textbf{\textit{\underline{(Ablation Study 12)}} Effect of adaptive fusion weight:}
We study the effect of the fusion weight $\gamma$ under varying levels of data heterogeneity and numbers of clients. As shown in Appendix~\ref{app:adaptive_fusion}, the optimal value varies with the degree of data heterogeneity and the number of clients.

\vspace{-1.5mm}
\noindent \textbf{\textit{\underline{(Ablation Study 13)}} Effect of explicit personalization:}
In Appendix~\ref{app:explicit_personalization}, we study whether supervising the personalized representation with a separate cross-entropy loss, rather than relying on the fusion loss, improves BLEND's performance. The results show that explicit personalization slightly improves personalization accuracy but reduces generalization and the overall H-mean.


\vspace{-1.5mm}
\noindent \textbf{\textit{\underline{(Ablation Study 14)}} Hyperparameter selection and validation:}
To avoid dataset-specific tuning, for BLEND, we used a unified hyperparameter configuration across all  experimental settings. To further address potential test-set-informed selection, we conduct additional experiments on an independent held-out validation split in Appendix~\ref{app:hyperparameter_validation}. The validation results identify the same optimal values of $\gamma$, $\lambda$, and adapter ranks as those reported in the above sensitivity analysis (Fig.~\ref{fig:gamma_lambda}, Table~\ref{tbl:global_adapter}, and Table~\ref{tbl:personalized_adapter}). 

\vspace{-4mm}
\section{Limitations}
\vspace{-4mm}
 BLEND is built on CLIP and focuses on text–image modalities. Extending it to broader multi-modal settings (e.g., audio or video), retrieval and generative tasks, and multi-task learning remains an open direction (see Appendix~\ref{app:generalization}). Also, BLEND’s performance depends on hyperparameters (e.g., $\gamma$ and $\lambda$) that govern the personalization–generalization trade-off, and their adaptive selection in new environments remains an open problem. Moreover, extension of BLEND to fully-decentralized/serverless FedVLMs remains open. Finally, the integration of privacy-preserving parameter aggregation mechanisms, such as differential privacy, within BLEND 
is worth investigating.

\vspace{-4mm}
\section{Conclusion}
\vspace{-4mm}
We introduced {BLEND}, a FedVLM framework that integrates globally aggregated adapters into the vision and text encoders of CLIP, together with a locally optimized vision projection head. 
Through a novel loss formulation, {BLEND} leverages the output of these components to enhance the balance of personalization vs. generalization. Extensive experiments validated each of the introduced model components and showed that \texttt{BLEND} consistently outperforms state-of-the-art baselines.


\section{Acknowledgments}
\vspace{-4mm}
This work was supported in part
by the U.S. National Science Foundation (NSF) under Grants ECCS-2543754, SaTC-2513164,
and ECCS-2512911; and in part by the National Natural Science Foundation of China under Grants No.~62301487 and No.~62671554.

\appendix



\bibliographystyle{unsrtnat}
\bibliography{ref}

\newpage

\clearpage
\setcounter{page}{1}

\newpage
\appendix
\onecolumn
\startcontents[sections]
\printcontents[sections]{l}{1}{\section*{Appendix Table of Contents}}

\clearpage

\section{Communication and Computation Overhead} \label{ap:overhead}
In this section, we demonstrate that BLEND achieves superior performance while using comparable or fewer trainable parameters during local training and incurring lower communication overhead. To this end, we compare against both \textit{prompt-based} and \textit{adapter-based} baselines under controlled resource budgets. Specifically, for each setting, we adjust BLEND (e.g., by reducing the number of adapters or their bottleneck size) to match the parameter or communication budget of the baselines, and then evaluate performance under these constraints. All reported results are averaged across all $10$ datasets used in the paper.

Tables~\ref{tbl:com_high} and~\ref{tbl:com_low} report results under comparable communication overhead. Table~\ref{tbl:com_high} corresponds to a high-budget regime (matching the resource footprint of MMA+FedAVG, which is a rather heavy adapter-based fine-tuning method), where BLEND achieves the highest H-Mean while transmitting fewer bits than MMA+FedAVG. Table~\ref{tbl:com_low} presents the low-budget regime, showing that BLEND maintains strong performance even with significantly reduced communication, outperforming prompt-based and lightweight adapter-based methods.
\vspace{-4mm}
\begin{table}[!h]
\centering
\scriptsize
\caption{Comparable communicated parameters (high budget). Communication overhead is measured as the total number of bits transmitted to the server over 50 global rounds, reported in MB.}
\begin{tabular}{l|ccccc}
\hline
\rowcolor{gray!20} Method & Total Number of Exchanged Params. & Communication Overhead (MB) & Per. & Gen. & H-Mean \\
\hline
MMA+FedAVG & 675K & 64.33 & 82.69 & 73.14 & \cellcolor{lightpink}77.34 \\
\hline
BLEND & 507K & 48.38 & 92.92 & 74.61 & \cellcolor{lightpink}\textbf{82.15} \\
\hline
\end{tabular}
\label{tbl:com_high}
\end{table}

\vspace{-9mm}
\begin{table}[!h]
\centering
\scriptsize
\caption{Comparable communicated parameters (low budget). Communication overhead is measured as the total number of bits transmitted to the server over 50 global rounds, reported in MB.}
\begin{tabular}{l|ccccc}
\hline
\rowcolor{gray!20}Method & Total Number of Exchanged Params. & Communication Overhead (MB) & Per. & Gen. & H-Mean \\
\hline
pFedMMA & 8.4K & 0.81 & 77.99 & 66.69 & \cellcolor{lightpink}71.48 \\
pFedMOAP & 8.2K & 0.78 & 92.03 & 66.48 & \cellcolor{lightpink}75.98 \\
FedPHA & 8.2K & 0.78 & 92.35 & 47.90 & \cellcolor{lightpink}61.15 \\
FedOTP & 8.2K & 0.78 & 85.92 & 32.57 & \cellcolor{lightpink}45.39 \\
PromptFL & 8.2K & 0.78 & 78.75 & 69.95 & \cellcolor{lightpink}73.61 \\
PromptFolio & 8.2K & 0.78 & 85.75 & 50.84 & \cellcolor{lightpink}61.85 \\
FedPGP & 16.4K & 1.56 & 87.88 & 64.38 & \cellcolor{lightpink}72.86 \\
\hline
BLEND & 6.4K & 0.61 & 91.01 & 73.37 & \cellcolor{lightpink}\textbf{80.75} \\
\hline
\end{tabular}
\label{tbl:com_low}
\end{table}
 
Also, Tables~\ref{tbl:comput_high} and~\ref{tbl:comput_low} report results under comparable trainable parameter budgets. In Table~\ref{tbl:comput_high} (high-budget setting), BLEND again achieves the best balance between personalization and generalization. In Table~\ref{tbl:comput_low} (low-budget setting), BLEND consistently outperforms competing methods despite using fewer trainable parameters, highlighting its efficiency and effectiveness under constrained resources.

\vspace{-4mm}
\begin{table}[!h]
\centering
\scriptsize
\caption{Comparable trainable parameters (high budget).}
\begin{tabular}{l|cccc}
\hline
\rowcolor{gray!20}Method & Total Number of Trainable Params. & Per. & Gen. & H-Mean \\
\hline
MMA+FedAVG & 675K & 82.69 & 73.14 & \cellcolor{lightpink}77.34 \\
pFedMMA & 675K & 77.99 & 66.69 & \cellcolor{lightpink}71.48 \\
\hline
BLEND & 536K & 92.67 & 74.15 & \cellcolor{lightpink}\textbf{81.79} \\
\hline
\end{tabular}
\label{tbl:comput_high}
\end{table}

\vspace{-7mm}
\begin{table}[!h]
\centering
\scriptsize
\caption{Comparable trainable parameters (low budget).}
\begin{tabular}{l|cccc}
\hline
\rowcolor{gray!20} Method & Total Number of Trainable Params. & Per. & Gen. & H-Mean \\
\hline
pFedMOAP & 74.2K & 92.03 & 66.48 & \cellcolor{lightpink}75.98 \\
FedPHA & 16.4K & 92.35 & 47.90 & \cellcolor{lightpink}61.15 \\
FedOTP & 16.4K & 85.92 & 32.57 & \cellcolor{lightpink}45.39 \\
PromptFL & 8.2K & 78.75 & 69.95 & \cellcolor{lightpink}73.61 \\
PromptFolio & 8.2K & 85.75 & 50.84 & \cellcolor{lightpink}61.85 \\
FedPGP & 20.6K & 87.88 & 64.38 & \cellcolor{lightpink}72.86 \\
\hline
BLEND & 13.3K & 85.32 & 73.99 & \cellcolor{lightpink}\textbf{79.01} \\
\hline
\end{tabular}
\label{tbl:comput_low}
\end{table}

\vspace{-2mm}
 We further investigate the computational efficiency of BLEND by comparing its per-round wall-clock time and average GPU memory usage with baselines. As shown in Table~\ref{tab:runtime_memory}, BLEND requires 11.18 seconds per round and 0.62 GB of GPU memory. These results demonstrate that computational requirements and GPU memory usage of BLEND are comparable to those of existing methods.
\vspace{-2mm}
\begin{table*}[!h]
\centering
\scriptsize
\caption{Comparison of per-round wall-clock time and average GPU memory usage across FL methods.}
\label{tab:runtime_memory}
\resizebox{\textwidth}{!}{
\begin{tabular}{l|ccccccccc}
\hline
\rowcolor{gray!20} Metric & {BLEND} & {pFedMMA} & {MMA+FedAvg} & {FedPHA} & {FedOPT} & {FedPGP} & {PromptFL} & {pFedMOAP} & {PromptFolio} \\
\hline
Time (s) & 11.18 & 11.53 & 11.59 & 13.35 & {9.70} & 15.19 & 10.30 & 10.01 & 10.11 \\
Mean GPU (GB) & 0.62 & 0.60 & 0.61 & 0.63 & {0.59} & 1.28 & 0.71 & 0.76 & 0.74 \\
\hline
\end{tabular}
}
\end{table*}


\clearpage
\section{Data Heterogeneity Variations}\label{ap:data_hete}
To evaluate the performance of BLEND under various levels of data heterogeneity, here we change the Dirichlet distribution parameter $\alpha^{\mathsf{Dir}}$ to include the values $\{10, 0.1, 0.01\}$ \textbf{over which the data heterogeneity increases} (smaller values of $\alpha^{\mathsf{Dir}}$ correspond to higher levels of data heterogeneity, and thus the x-axis in all plots is shown in descending order of $\alpha^{\mathsf{Dir}}$). We plot the personalization and generalization performances and the H-Mean of the two under each level of data heterogeneity across datasets in Figs.~\ref{fig:datahet-UCF101}-\ref{fig:datahet-Imagenet}. This general performance overview is depicted in Fig.~\ref{fig:datahet-avgAll} where the averaged performances under all datasets are shown.

Across all experiments, BLEND consistently achieves the highest H-Mean, indicating a superior balance between personalization and generalization. More importantly, the performance gap between BLEND and the baselines widens as data heterogeneity becomes more severe, as evidenced by the H-Mean curves. This trend highlights the robustness of BLEND under harsh non-IID data conditions (most baseline methods exhibit a sharp decline in H-Mean as data heterogeneity increases, whereas BLEND maintains stable performance even in highly data heterogeneous settings).

\begin{figure}[h]
    \centering
    \hspace{-14mm}
    \includegraphics[width=.9\linewidth]{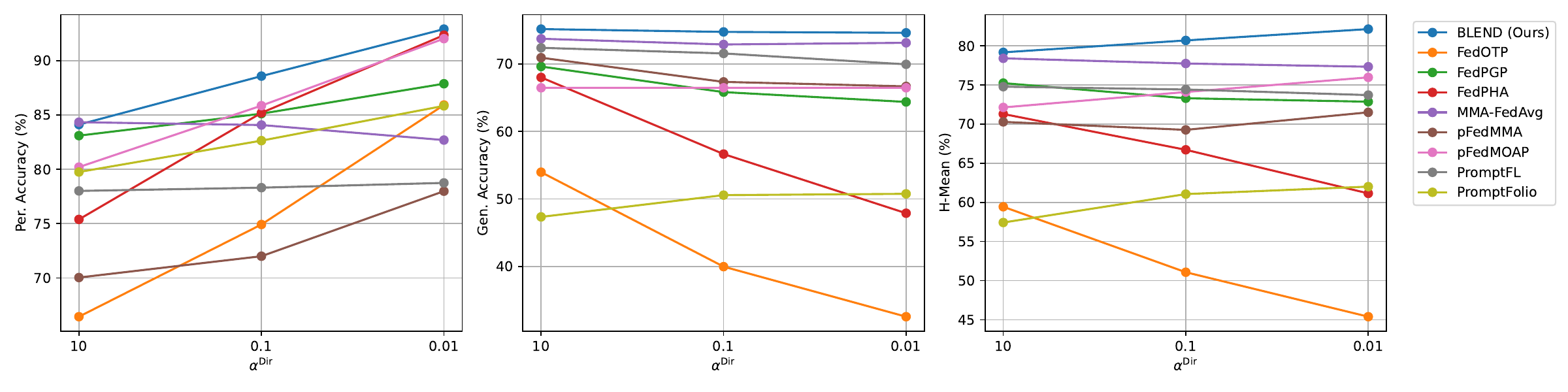}
    \caption{Personalization and generalization test accuracies alongside their harmonic mean (H-Mean) under changing values of $\alpha^{\mathsf{Dir}}$ (data heterogeneity) \textbf{averaged across all datasets}.}
    \label{fig:datahet-avgAll}
\end{figure}
\begin{figure}[h]
    \centering
    \includegraphics[width=.99\linewidth]{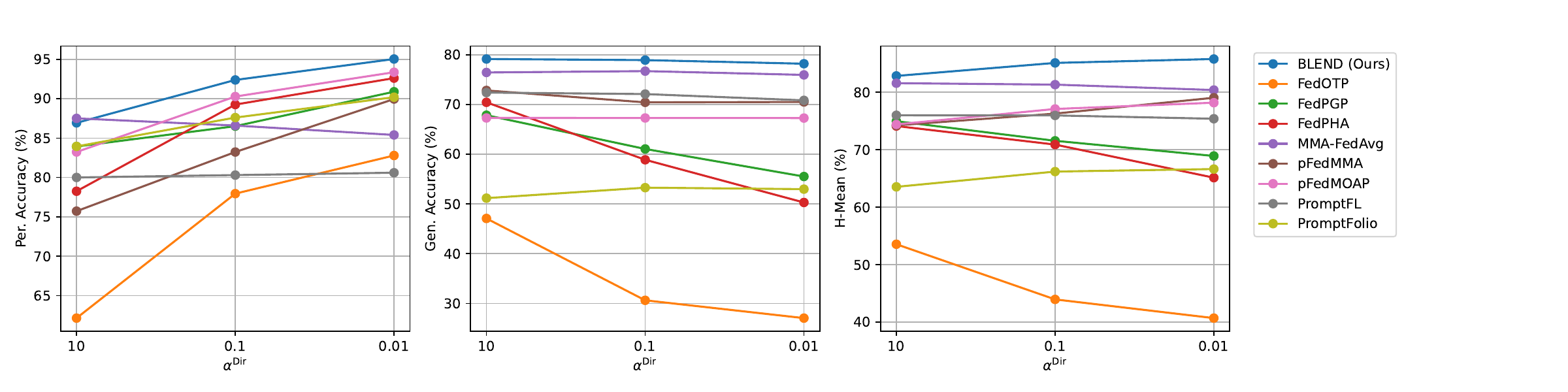}
    \caption{Personalization and generalization test accuracies alongside their harmonic mean (H-Mean) under changing values of $\alpha^{\mathsf{Dir}}$ (data heterogeneity) for the \textbf{\textit{UCF101}} dataset.}
    \label{fig:datahet-UCF101}
\end{figure}
\begin{figure}[h]
    \centering
    \includegraphics[width=.99\linewidth]{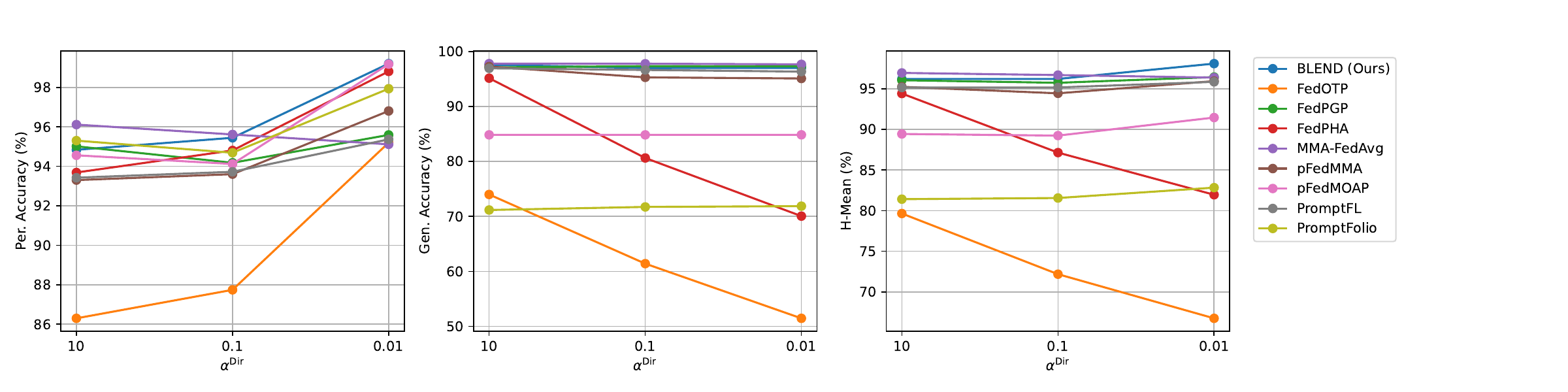}
    \caption{Personalization and generalization test accuracies alongside their harmonic mean (H-Mean) under changing values of $\alpha^{\mathsf{Dir}}$ (data heterogeneity) for the \textbf{\textit{OxfordPets}} dataset.}
    \label{fig:datahet-OxPets}
\end{figure}
\begin{figure}[h]
    \centering
    \includegraphics[width=.99\linewidth]{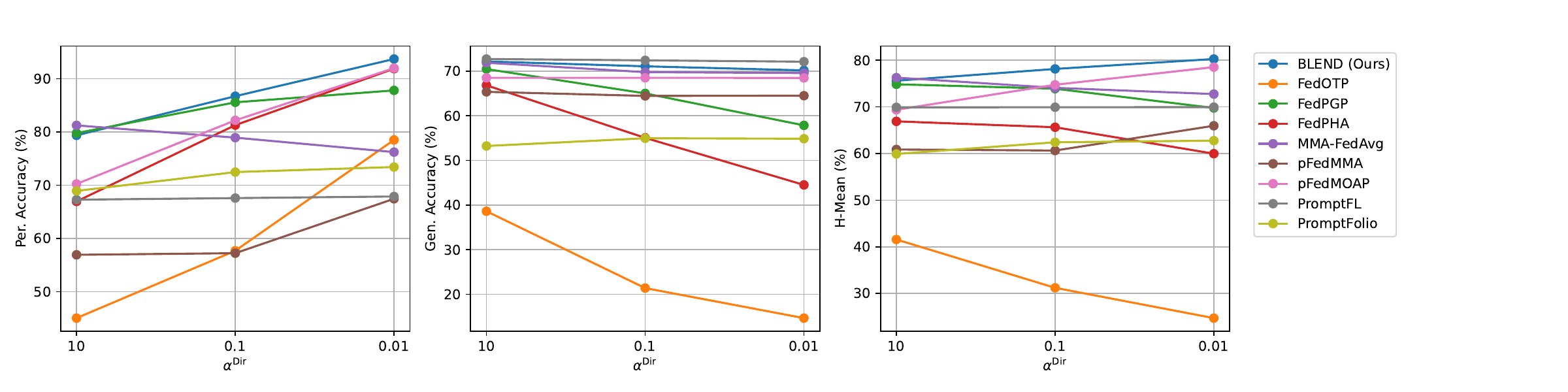}
    \caption{Personalization and generalization test accuracies alongside their harmonic mean (H-Mean) under changing values of $\alpha^{\mathsf{Dir}}$ (data heterogeneity) for the \textbf{\textit{Stanford-Cars}} dataset.}
    \label{fig:datahet-StanCars}
\end{figure}
\begin{figure}[h]
    \centering
    \includegraphics[width=.99\linewidth]{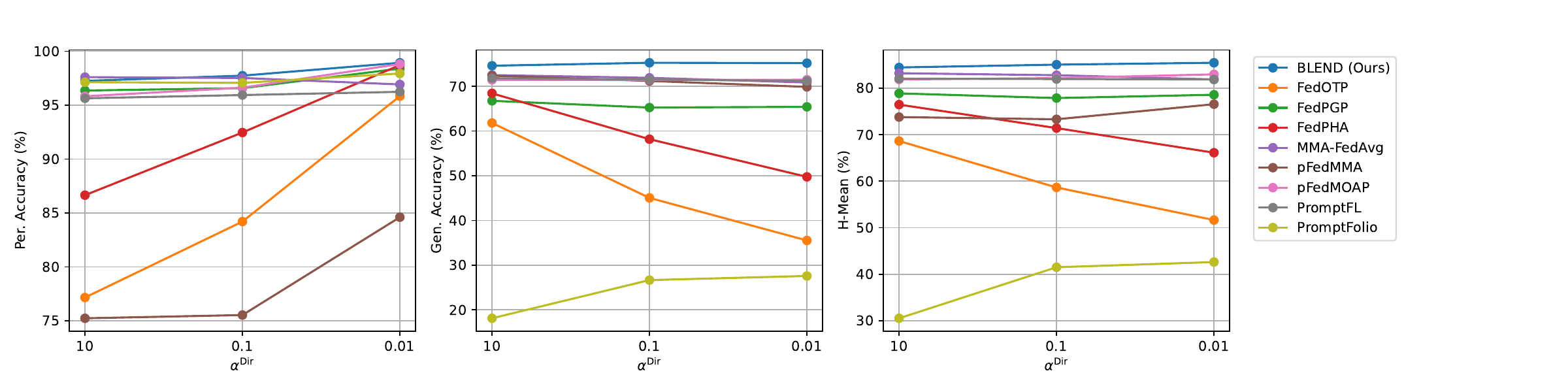}
    \caption{Personalization and generalization test accuracies alongside their harmonic mean (H-Mean) under changing values of $\alpha^{\mathsf{Dir}}$ (data heterogeneity) for the \textbf{\textit{Flowers}} dataset.}
    \label{fig:datahet-OxFlowers}
\end{figure}
\begin{figure}[h]
    \centering
    \includegraphics[width=.99\linewidth]{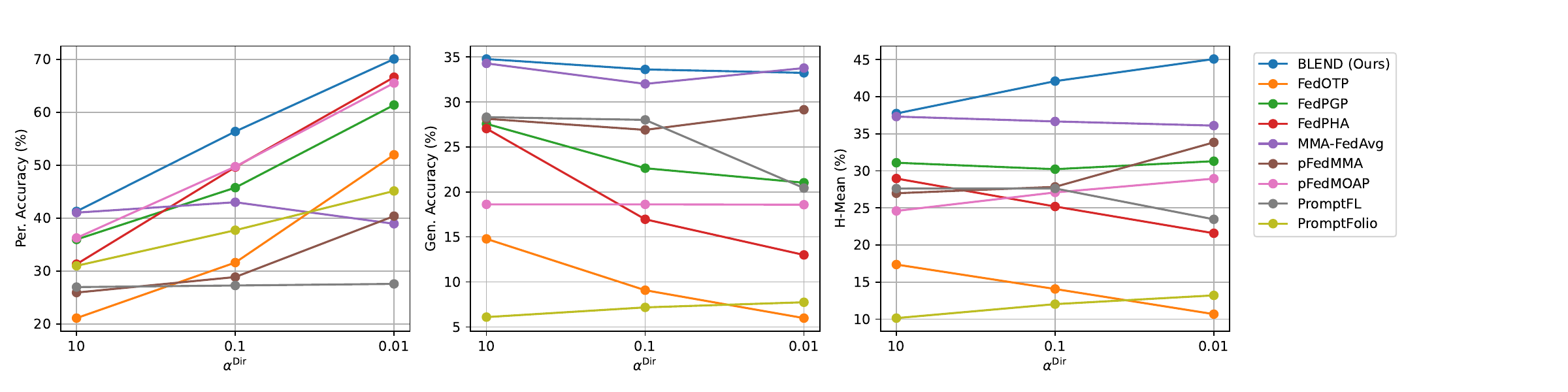}
    \caption{Personalization and generalization test accuracies alongside their harmonic mean (H-Mean) under changing values of $\alpha^{\mathsf{Dir}}$ (data heterogeneity) for the \textbf{\textit{FGVC Aircraft}} dataset.}
    \label{fig:datahet-FGVCAir}
\end{figure}
\begin{figure}[h]
    \centering
    \includegraphics[width=.99\linewidth]{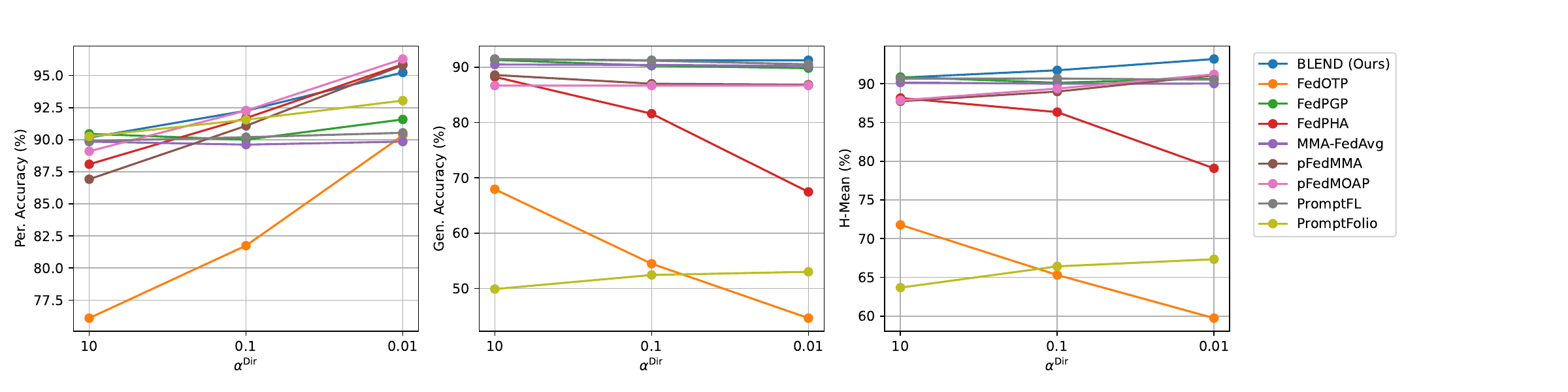}
    \caption{Personalization and generalization test accuracies alongside their harmonic mean (H-Mean) under changing values of $\alpha^{\mathsf{Dir}}$ (data heterogeneity) for the \textbf{\textit{Food101}} dataset.}
    \label{fig:datahet-Food101}
\end{figure}
\begin{figure}[h]
    \centering
    \includegraphics[width=.99\linewidth]{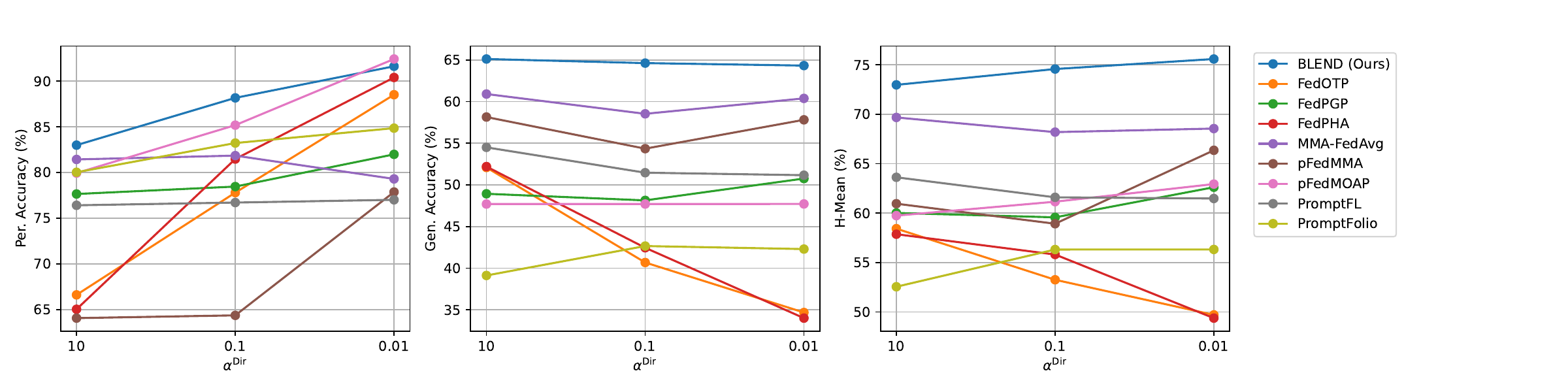}
    \caption{Personalization and generalization test accuracies alongside their harmonic mean (H-Mean) under changing values of $\alpha^{\mathsf{Dir}}$ (data heterogeneity) for the \textbf{\textit{DTD}} dataset.}
    \label{fig:datahet-DTD}
\end{figure}
\begin{figure}[h]
    \centering
    \includegraphics[width=.99\linewidth]{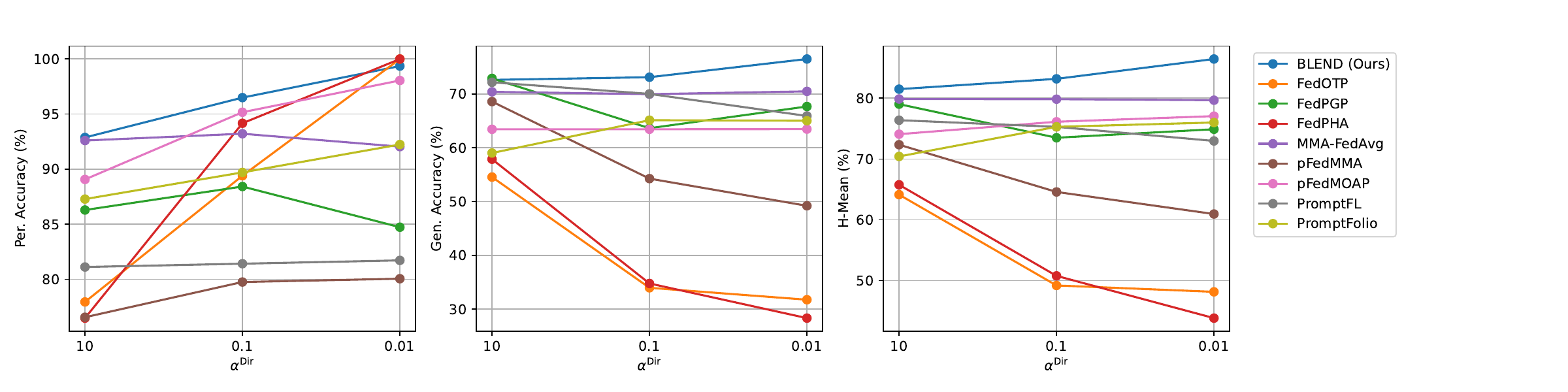}
    \caption{Personalization and generalization test accuracies alongside their harmonic mean (H-Mean) under changing values of $\alpha^{\mathsf{Dir}}$ (data heterogeneity) for the \textbf{\textit{EuroSAT}} dataset.}
    \label{fig:datahet-EuroSat}
\end{figure}
\begin{figure}[h]
    \centering
    \includegraphics[width=.99\linewidth]{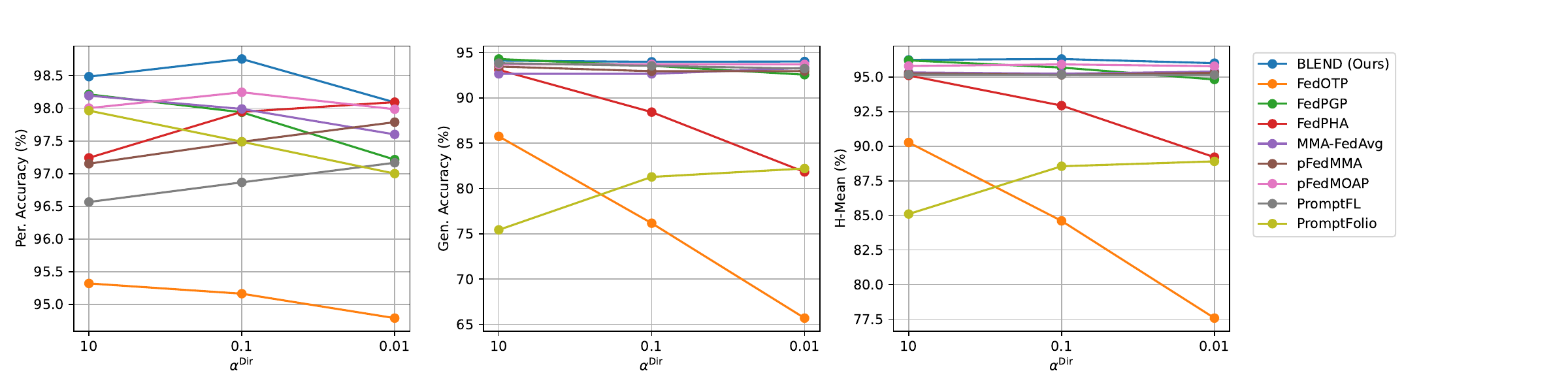}
    \caption{Personalization and generalization test accuracies alongside their harmonic mean (H-Mean) under changing values of $\alpha^{\mathsf{Dir}}$ (data heterogeneity) for the \textbf{\textit{Caltech}} dataset.}
    \label{fig:datahet-Caltech101}
\end{figure}
\begin{figure}[!h]
    \centering
    \includegraphics[width=.99\linewidth]{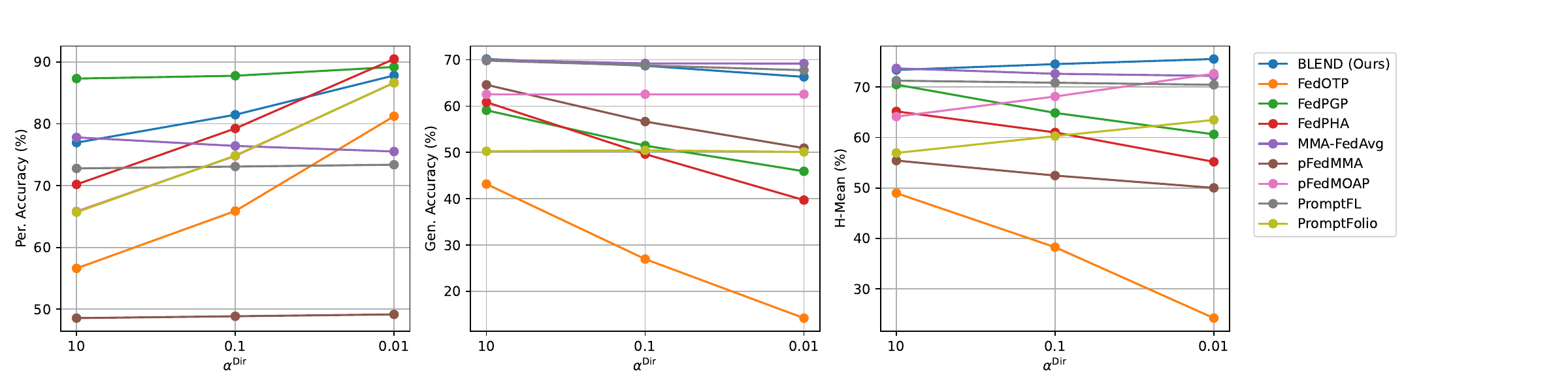}
    \caption{Personalization and generalization test accuracies alongside their harmonic mean (H-Mean) under changing values of $\alpha^{\mathsf{Dir}}$ (data heterogeneity) for the \textbf{\textit{ImageNet}} dataset.}
    \label{fig:datahet-Imagenet}
\end{figure}

\clearpage

\section{Varying the Number of Clients}\label{ap:num_clients}

To extend our experiments regarding the number of clients that collectively train the global model, we repeated the experiments as in Table~\ref{tab:main_results} with larger numbers of clients $N\in\{20, 30\}$, the results of which we include in Tables~\ref{tab:main_results2} and \ref{tab:main_results3}, respectively. 
Inspecting the results in these Tables along with Table~\ref{tab:main_results} in the main text, we can observe that the H-Mean performance of BLEND is consistently better than all the baselines in all the datasets. 


More importantly, the performance margin between BLEND and the strongest baseline increases as the number of clients grows: the average gain rises from $4.81\%$ for $N=10$ (Table~\ref{tab:main_results}), to $5.64\%$ for $N=20$ (Table~\ref{tab:main_results2}), and further to $6.11\%$ for $N=30$ (Table~\ref{tab:main_results3}).
This trend can be attributed to the fact that increasing the number of clients typically amplifies local data bias and heterogeneity: by jointly leveraging globally aggregated adapters and locally personalized projection heads, BLEND is better able to mitigate this increased heterogeneity, resulting in improved robustness and scalability as the number of clients increases.


\begin{table*}[h]
\centering
\caption{Results for a scenario with $N=20$ clients. Performance comparisons (measured in accuracy, \%) between BLEND and baselines. The row $\Delta$ denotes the performance gap between BLEND and the best-performing baseline. Standard deviations across three random seeds are shown through sub-indices. Statistical significance is assessed using a two-sided paired $t$-test on seed-level H-Mean scores between BLEND and the strongest baseline, with improvements statistically significant at $p<0.05$.
}
\vspace{-2mm}
\label{tab:main_results2}
\resizebox{\textwidth}{!}{
\begin{tabular}{l|ccc|ccc|ccc|ccc|ccc|ccc}
\toprule
& \multicolumn{3}{c|}{\cellcolor{gray!20}{UCF101}}
& \multicolumn{3}{c|}{\cellcolor{gray!20}{Oxford Pets}}
& \multicolumn{3}{c|}{\cellcolor{gray!20}{Stanford Cars}}
& \multicolumn{3}{c|}{\cellcolor{gray!20}{Flowers}}
& \multicolumn{3}{c|}{\cellcolor{gray!20}{Imagenet}}
& \multicolumn{3}{c}{\cellcolor{gray!20}{FGVC Aircraft}} \\
\cmidrule(lr){2-4}\cmidrule(lr){5-7}\cmidrule(lr){8-10}\cmidrule(lr){11-13}\cmidrule(lr){14-16}\cmidrule(lr){17-19}
& Per. & Gen. & \cellcolor{lightpink}{H-Mean} & Per. & Gen. & \cellcolor{lightpink}{H-Mean} & Per. & Gen. & \cellcolor{lightpink}{H-Mean} &Per. & Gen. & \cellcolor{lightpink}{H-Mean} & Per. & Gen. & \cellcolor{lightpink}{H-Mean} & Per. & Gen. & \cellcolor{lightpink}{H-Mean} \\
\midrule

Zero-shot CLIP
& 68.10{\scriptsize$\pm$0.00} & 75.18{\scriptsize$\pm$0.00} & 71.46{\scriptsize$\pm$\cellcolor{lightpink}{0.00}}
& 89.47{\scriptsize$\pm$0.00} & 96.81{\scriptsize$\pm$0.00} & 93.00{\scriptsize$\pm$\cellcolor{lightpink}{0.00}}
& 63.77{\scriptsize$\pm$0.00} & \textbf{74.94{\scriptsize$\pm$0.00}} & 68.91{\scriptsize$\pm$\cellcolor{lightpink}{0.00}}
& 69.42{\scriptsize$\pm$0.00} & \textbf{76.74{\scriptsize$\pm$0.00}} & 72.90{\scriptsize$\pm$\cellcolor{lightpink}{0.00}}
& 72.40{\scriptsize$\pm$0.00} & 68.08{\scriptsize$\pm$0.00} & 70.17{\scriptsize$\pm$\cellcolor{lightpink}{0.00}}
& 27.49{\scriptsize$\pm$0.00} & 33.05{\scriptsize$\pm$0.00} & 30.02{\scriptsize$\pm$\cellcolor{lightpink}{0.00}}\\

MMA+FedAvg
& 79.22{\scriptsize$\pm$1.17} & 77.27{\scriptsize$\pm$0.98} & 78.23{\scriptsize$\pm$\cellcolor{lightpink}{0.83}}
& 95.96{\scriptsize$\pm$0.10} & 97.50{\scriptsize$\pm$0.11} & 96.73{\scriptsize$\pm$\cellcolor{lightpink}{0.04}}
& 75.86{\scriptsize$\pm$0.71} & 69.71{\scriptsize$\pm$0.79} & 72.66{\scriptsize$\pm$\cellcolor{lightpink}{0.74}}
& 94.31{\scriptsize$\pm$0.29} & 72.29{\scriptsize$\pm$1.72} & 81.83{\scriptsize$\pm$\cellcolor{lightpink}{1.02}}
& 77.15{\scriptsize$\pm$0.22} & \textbf{69.10{\scriptsize$\pm$0.30}} & 72.90{\scriptsize$\pm$\cellcolor{lightpink}{0.07}}
& 35.34{\scriptsize$\pm$0.13} & 34.67{\scriptsize$\pm$0.45} & 35.00{\scriptsize$\pm$\cellcolor{lightpink}{0.26}} \\

pFedMMA
& 92.85{\scriptsize$\pm$0.51} & 69.35{\scriptsize$\pm$0.70} & 79.40{\scriptsize$\pm$\cellcolor{lightpink}{0.53}}
& 97.19{\scriptsize$\pm$0.96} & 94.97{\scriptsize$\pm$0.24} & 96.06{\scriptsize$\pm$\cellcolor{lightpink}{0.37}}
& 77.73{\scriptsize$\pm$2.77} & 62.93{\scriptsize$\pm$0.54} & 69.52{\scriptsize$\pm$\cellcolor{lightpink}{0.92}}
& 86.38{\scriptsize$\pm$1.12} & 72.15{\scriptsize$\pm$0.29} & 78.62{\scriptsize$\pm$\cellcolor{lightpink}{0.46}}
& 70.39{\scriptsize$\pm$2.26} & 58.47{\scriptsize$\pm$0.85} & 63.87{\scriptsize$\pm$\cellcolor{lightpink}{1.44}}
& 40.40{\scriptsize$\pm$1.40} & 28.13{\scriptsize$\pm$0.45} & 33.16{\scriptsize$\pm$\cellcolor{lightpink}{0.60}} \\

pFedMoAP
& 93.94{\scriptsize$\pm$0.35} & 67.22{\scriptsize$\pm$0.01} & 78.36{\scriptsize$\pm$\cellcolor{lightpink}{0.12}}
& 99.63{\scriptsize$\pm$0.09} & 84.78{\scriptsize$\pm$0.01} & 91.61{\scriptsize$\pm$\cellcolor{lightpink}{0.04}}
& 92.49{\scriptsize$\pm$0.54} & 68.45{\scriptsize$\pm$0.01} & 78.67{\scriptsize$\pm$\cellcolor{lightpink}{0.20}}
& 96.73{\scriptsize$\pm$0.63} & 71.44{\scriptsize$\pm$0.01} & 82.18{\scriptsize$\pm$\cellcolor{lightpink}{0.22}}
& 90.21{\scriptsize$\pm$0.00} & 62.49{\scriptsize$\pm$0.00} & 73.84{\scriptsize$\pm$\cellcolor{lightpink}{0.00}}
& \textbf{73.23{\scriptsize$\pm$2.16}} & 18.53{\scriptsize$\pm$0.00} & 29.58{\scriptsize$\pm$\cellcolor{lightpink}{0.18}}\\

FedPHA
& 93.37{\scriptsize$\pm$0.56} & 51.90{\scriptsize$\pm$3.30} & 66.64{\scriptsize$\pm$\cellcolor{lightpink}{2.58}}
& \textbf{99.82{\scriptsize$\pm$0.06}} & 56.42{\scriptsize$\pm$3.71} & 72.02{\scriptsize$\pm$\cellcolor{lightpink}{3.02}}
& 93.80{\scriptsize$\pm$0.70} & 41.06{\scriptsize$\pm$0.10} & 57.12{\scriptsize$\pm$\cellcolor{lightpink}{0.03}}
& 95.07{\scriptsize$\pm$0.46} & 43.03{\scriptsize$\pm$1.25} & 59.23{\scriptsize$\pm$\cellcolor{lightpink}{1.28}}
& \textbf{91.60{\scriptsize$\pm$0.01}} & 34.09{\scriptsize$\pm$0.50} & 49.69{\scriptsize$\pm$\cellcolor{lightpink}{0.53}}
& 71.86{\scriptsize$\pm$0.00} & 8.03{\scriptsize$\pm$0.00} & 14.44{\scriptsize$\pm$\cellcolor{lightpink}{0.00}}\\

FedOTP
& 79.39{\scriptsize$\pm$0.65} & 24.91{\scriptsize$\pm$1.17} & 37.91{\scriptsize$\pm$\cellcolor{lightpink}{1.44}}
& 95.80{\scriptsize$\pm$1.39} & 44.30{\scriptsize$\pm$7.16} & 60.17{\scriptsize$\pm$\cellcolor{lightpink}{6.68}}
& 76.06{\scriptsize$\pm$1.12} & 13.87{\scriptsize$\pm$1.33} & 23.43{\scriptsize$\pm$\cellcolor{lightpink}{1.90}}
& 88.41{\scriptsize$\pm$1.60} & 36.05{\scriptsize$\pm$3.59} & 51.11{\scriptsize$\pm$\cellcolor{lightpink}{3.74}}
& 84.25{\scriptsize$\pm$0.13} & 13.53{\scriptsize$\pm$0.07} & 23.31{\scriptsize$\pm$\cellcolor{lightpink}{0.11}}
& 57.40{\scriptsize$\pm$2.63} & 5.82{\scriptsize$\pm$1.09} & 10.55{\scriptsize$\pm$\cellcolor{lightpink}{1.85}} \\

PromptFL
& 81.56{\scriptsize$\pm$1.36} & 71.84{\scriptsize$\pm$2.49} & 76.36{\scriptsize$\pm$\cellcolor{lightpink}{1.33}}
& 96.74{\scriptsize$\pm$0.04} & 96.85{\scriptsize$\pm$0.65} & 96.79{\scriptsize$\pm$\cellcolor{lightpink}{0.35}}
& 71.06{\scriptsize$\pm$0.43} & 70.53{\scriptsize$\pm$1.37} & 70.78{\scriptsize$\pm$\cellcolor{lightpink}{0.52}}
& 96.29{\scriptsize$\pm$0.62} & 70.14{\scriptsize$\pm$0.42} & 81.16{\scriptsize$\pm$\cellcolor{lightpink}{0.40}}
& 75.13{\scriptsize$\pm$0.17} & 68.96{\scriptsize$\pm$0.04} & 71.91{\scriptsize$\pm$\cellcolor{lightpink}{0.10}}
& 34.46{\scriptsize$\pm$1.51} & 24.86{\scriptsize$\pm$6.81} & 28.46{\scriptsize$\pm$\cellcolor{lightpink}{5.12}} \\

FedPGP
& 82.64{\scriptsize$\pm$0.10} & 68.44{\scriptsize$\pm$2.20} & 74.85{\scriptsize$\pm$\cellcolor{lightpink}{1.28}}
& 96.88{\scriptsize$\pm$0.05} & \textbf{97.53{\scriptsize$\pm$0.03}} & 97.20{\scriptsize$\pm$\cellcolor{lightpink}{0.01}}
& 85.11{\scriptsize$\pm$1.15} & 60.90{\scriptsize$\pm$2.11} & 70.96{\scriptsize$\pm$\cellcolor{lightpink}{1.03}}
& 91.79{\scriptsize$\pm$1.75} & 69.20{\scriptsize$\pm$0.59} & 78.91{\scriptsize$\pm$\cellcolor{lightpink}{1.03}}
& 88.11{\scriptsize$\pm$0.09} & 46.39{\scriptsize$\pm$0.06} & 60.78{\scriptsize$\pm$\cellcolor{lightpink}{0.10}}
& 43.49{\scriptsize$\pm$1.84} & 25.16{\scriptsize$\pm$2.17} & 31.86{\scriptsize$\pm$\cellcolor{lightpink}{2.23}}\\

PromptFoilo
& 87.92{\scriptsize$\pm$0.57} & 53.14{\scriptsize$\pm$3.92} & 66.14{\scriptsize$\pm$\cellcolor{lightpink}{2.89}}
& 97.91{\scriptsize$\pm$0.16} & 71.90{\scriptsize$\pm$4.39} & 82.83{\scriptsize$\pm$\cellcolor{lightpink}{2.87}}
& 77.11{\scriptsize$\pm$2.96} & 54.88{\scriptsize$\pm$5.78} & 64.05{\scriptsize$\pm$\cellcolor{lightpink}{4.97}}
& 94.48{\scriptsize$\pm$0.96} & 27.70{\scriptsize$\pm$6.70} & 42.36{\scriptsize$\pm$\cellcolor{lightpink}{7.94}}
& 86.97{\scriptsize$\pm$1.33} & 50.86{\scriptsize$\pm$3.96} & 61.68{\scriptsize$\pm$\cellcolor{lightpink}{3.76}}
& 51.52{\scriptsize$\pm$2.86} & 7.74{\scriptsize$\pm$0.42} & 13.44{\scriptsize$\pm$\cellcolor{lightpink}{0.54}}\\
\midrule

BLEND
& \textbf{94.59{\scriptsize$\pm$0.09}} & \textbf{78.84{\scriptsize$\pm$0.16}} & \textbf{86.00{\scriptsize$\pm$\cellcolor{lightpink}{0.13}}}
& 99.65{\scriptsize$\pm$0.03} & 96.58{\scriptsize$\pm$0.09} & \textbf{98.09{\scriptsize$\pm$\cellcolor{lightpink}{0.06}}}
& \textbf{94.30{\scriptsize$\pm$0.21}} & 73.71{\scriptsize$\pm$0.02} & \textbf{82.75{\scriptsize$\pm$\cellcolor{lightpink}{0.09}}}
& \textbf{98.15{\scriptsize$\pm$0.14}} & 75.74{\scriptsize$\pm$0.23} & \textbf{85.50{\scriptsize$\pm$\cellcolor{lightpink}{0.17}}}
& 90.49{\scriptsize$\pm$0.13} & 68.57{\scriptsize$\pm$0.16} & \textbf{78.02{\scriptsize$\pm$\cellcolor{lightpink}{0.15}}}
& 70.08{\scriptsize$\pm$0.61} & \textbf{34.92{\scriptsize$\pm$0.68}} & \textbf{46.61{\scriptsize$\pm$\cellcolor{lightpink}{0.46}}} \\
\midrule
$\Delta$ ($p-$value)
& & & \cellcolor{lightpink}{\textcolor{darkgreen}{\textbf{+6.60}}} (0.017)
&  &  & \cellcolor{lightpink}{\textcolor{darkgreen}{\textbf{+0.89}}} (0.008)
& & & \cellcolor{lightpink}{\textcolor{darkgreen}{\textbf{+10.09}}} (0.0.013)
& & & \cellcolor{lightpink}{\textcolor{darkgreen}{\textbf{+3.32}}} (0.003)
& & & \cellcolor{lightpink}{\textcolor{darkgreen}{\textbf{+4.18}}} (0.002)
& & & \cellcolor{lightpink}{\textcolor{darkgreen}{\textbf{+11.61}}} (0.003) \\
\hline
\hline

& \multicolumn{3}{c|}{\cellcolor{gray!20}{Food101}}
& \multicolumn{3}{c|}{\cellcolor{gray!20}{DTD}}
& \multicolumn{3}{c|}{\cellcolor{gray!20}{EuroSAT}}
& \multicolumn{3}{c|}{\cellcolor{gray!20}{Caltech}}
& \multicolumn{3}{c|}{\textbf{\cellcolor{gray!20}{Cross-Dataset Avg. ($\uparrow$)}}}
& \multicolumn{3}{c}{\textbf{\cellcolor{gray!20}{Cross-Dataset  Std. ($\downarrow$)}}} \\
\cmidrule(lr){2-4}\cmidrule(lr){5-7}\cmidrule(lr){8-10}\cmidrule(lr){11-13}\cmidrule(lr){14-16}\cmidrule(lr){17-19}
& Per.& Gen. & \cellcolor{lightpink}{H-Mean} & Per.& Gen. & \cellcolor{lightpink}{H-Mean} & Per.& Gen. & \cellcolor{lightpink}{H-Mean} &Per.& Gen. & \cellcolor{lightpink}{H-Mean} & Per.& Gen. & \cellcolor{lightpink}{H-Mean} & Per.& Gen. & \cellcolor{lightpink}{H-Mean} \\
\midrule

Zero-shot CLIP
& 89.44{\scriptsize$\pm$0.00} & 90.67{\scriptsize$\pm$0.00} & 90.05{\scriptsize$\pm$\cellcolor{lightpink}{0.00}}
& 54.75{\scriptsize$\pm$0.00} & 58.94{\scriptsize$\pm$0.00} & 56.76{\scriptsize$\pm$\cellcolor{lightpink}{0.00}}
& 50.21{\scriptsize$\pm$0.00} & 69.87{\scriptsize$\pm$0.00} & 58.43{\scriptsize$\pm$\cellcolor{lightpink}{0.00}}
& 97.16{\scriptsize$\pm$0.00} & 94.21{\scriptsize$\pm$0.00} & 95.66{\scriptsize$\pm$\cellcolor{lightpink}{0.00}}
& 68.22{\scriptsize$\pm$0.00} & 73.85{\scriptsize$\pm$0.00} & 70.74{\scriptsize$\pm$\cellcolor{lightpink}{0.00}}
& 19.86{\scriptsize$\pm$0.00} & 17.81{\scriptsize$\pm$0.00} & 18.72{\scriptsize$\pm$\cellcolor{lightpink}{0.00}}\\

MMA+FedAvg
& 89.83{\scriptsize$\pm$0.07} & 91.33{\scriptsize$\pm$0.03} & 90.57{\scriptsize$\pm$\cellcolor{lightpink}{0.05}}
& 78.90{\scriptsize$\pm$1.53} & \textbf{62.48{\scriptsize$\pm$1.87}} & 69.72{\scriptsize$\pm$\cellcolor{lightpink}{1.48}}
& 92.01{\scriptsize$\pm$0.91} & 68.89{\scriptsize$\pm$6.26} & 78.64{\scriptsize$\pm$\cellcolor{lightpink}{4.15}}
& 97.53{\scriptsize$\pm$0.33} & 93.78{\scriptsize$\pm$0.32} & 95.62{\scriptsize$\pm$\cellcolor{lightpink}{0.33}}
& 81.61{\scriptsize$\pm$0.55} & 73.70{\scriptsize$\pm$1.28} & 77.19{\scriptsize$\pm$\cellcolor{lightpink}{0.90}}
& 17.33{\scriptsize$\pm$0.48} & 17.33{\scriptsize$\pm$1.77} & 16.73{\scriptsize$\pm$\cellcolor{lightpink}{1.18}}\\

pFedMMA
& \textbf{96.99{\scriptsize$\pm$0.14}} & 88.32{\scriptsize$\pm$0.39} & 92.45{\scriptsize$\pm$\cellcolor{lightpink}{0.27}}
& 86.40{\scriptsize$\pm$1.31} & 58.11{\scriptsize$\pm$1.05} & 69.47{\scriptsize$\pm$\cellcolor{lightpink}{0.37}}
& 74.33{\scriptsize$\pm$17.75} & 49.26{\scriptsize$\pm$6.61} & 58.98{\scriptsize$\pm$\cellcolor{lightpink}{10.22}}
& 98.08{\scriptsize$\pm$0.41} & 93.24{\scriptsize$\pm$0.11} & 95.60{\scriptsize$\pm$\cellcolor{lightpink}{0.24}}
& 82.07{\scriptsize$\pm$2.86} & 67.49{\scriptsize$\pm$1.12} & 73.71{\scriptsize$\pm$\cellcolor{lightpink}{1.54}}
& 16.77{\scriptsize$\pm$5.02} & 19.87{\scriptsize$\pm$1.85} & 18.38{\scriptsize$\pm$\cellcolor{lightpink}{2.91}} \\

pFedMoAP
& 96.09{\scriptsize$\pm$0.10} & 86.62{\scriptsize$\pm$0.01} & 91.11{\scriptsize$\pm$\cellcolor{lightpink}{0.04}}
& 94.33{\scriptsize$\pm$0.50} & 47.66{\scriptsize$\pm$0.01} & 63.32{\scriptsize$\pm$\cellcolor{lightpink}{0.12}}
& 98.26{\scriptsize$\pm$1.11} & 63.43{\scriptsize$\pm$0.00} & 77.09{\scriptsize$\pm$\cellcolor{lightpink}{0.34}}
& 99.42{\scriptsize$\pm$0.31} & 93.64{\scriptsize$\pm$0.01} & 96.44{\scriptsize$\pm$\cellcolor{lightpink}{0.14}}
& 93.43{\scriptsize$\pm$0.58} & 66.43{\scriptsize$\pm$0.01} & 76.22{\scriptsize$\pm$\cellcolor{lightpink}{0.14}}
& \textbf{7.32{\scriptsize$\pm$0.61}} & 20.46{\scriptsize$\pm$0.00} & 18.06{\scriptsize$\pm$\cellcolor{lightpink}{0.10}} \\

FedPHA
& 96.28{\scriptsize$\pm$0.27} & 60.36{\scriptsize$\pm$1.78} & 74.18{\scriptsize$\pm$\cellcolor{lightpink}{1.27}}
& 93.58{\scriptsize$\pm$0.50} & 29.94{\scriptsize$\pm$1.22} & 45.36{\scriptsize$\pm$\cellcolor{lightpink}{1.46}}
& 99.91{\scriptsize$\pm$0.09} & 32.08{\scriptsize$\pm$4.97} & 48.34{\scriptsize$\pm$\cellcolor{lightpink}{5.69}}
& 98.82{\scriptsize$\pm$0.00} & 83.07{\scriptsize$\pm$0.00} & 90.27{\scriptsize$\pm$\cellcolor{lightpink}{0.00}}
& 93.41{\scriptsize$\pm$0.27} & 44.00{\scriptsize$\pm$1.68} & 57.73{\scriptsize$\pm$\cellcolor{lightpink}{1.59}}
& 7.69{\scriptsize$\pm$0.25} & 19.37{\scriptsize$\pm$1.66} & 19.48{\scriptsize$\pm$\cellcolor{lightpink}{1.69}}\\

FedOTP
& 88.04{\scriptsize$\pm$1.04} & 43.88{\scriptsize$\pm$1.44} & 58.55{\scriptsize$\pm$\cellcolor{lightpink}{1.13}}
& 88.78{\scriptsize$\pm$0.19} & 26.71{\scriptsize$\pm$1.12} & 41.06{\scriptsize$\pm$\cellcolor{lightpink}{1.33}}
& \textbf{100.00{\scriptsize$\pm$0.00}} & 32.31{\scriptsize$\pm$3.50} & 48.73{\scriptsize$\pm$\cellcolor{lightpink}{4.07}}
& 96.19{\scriptsize$\pm$1.43} & 67.10{\scriptsize$\pm$2.23} & 79.05{\scriptsize$\pm$\cellcolor{lightpink}{2.03}}
& 85.43{\scriptsize$\pm$1.02} & 30.85{\scriptsize$\pm$2.27} & 43.39{\scriptsize$\pm$\cellcolor{lightpink}{2.43}}
& 11.72{\scriptsize$\pm$0.77} & 17.21{\scriptsize$\pm$1.93} & 19.46{\scriptsize$\pm$\cellcolor{lightpink}{1.80}} \\

PromptFL
& 90.01{\scriptsize$\pm$0.08} & 90.94{\scriptsize$\pm$0.10} & 90.48{\scriptsize$\pm$\cellcolor{lightpink}{0.04}}
& 80.21{\scriptsize$\pm$0.48} & 52.74{\scriptsize$\pm$2.44} & 63.61{\scriptsize$\pm$\cellcolor{lightpink}{1.94}}
& 86.70{\scriptsize$\pm$2.21} & 64.51{\scriptsize$\pm$6.56} & 73.70{\scriptsize$\pm$\cellcolor{lightpink}{3.49}}
& 97.49{\scriptsize$\pm$0.39} & 92.10{\scriptsize$\pm$0.89} & 94.71{\scriptsize$\pm$\cellcolor{lightpink}{0.35}}
& 80.96{\scriptsize$\pm$0.73} & 70.35{\scriptsize$\pm$2.18} & 74.80{\scriptsize$\pm$\cellcolor{lightpink}{1.36}}
& 17.80{\scriptsize$\pm$0.68} & 20.04{\scriptsize$\pm$2.40} & 18.62{\scriptsize$\pm$\cellcolor{lightpink}{1.61}} \\

FedPGP
& 90.68{\scriptsize$\pm$0.20} & \textbf{91.50{\scriptsize$\pm$0.13}} & 91.09{\scriptsize$\pm$\cellcolor{lightpink}{0.17}}
& 79.80{\scriptsize$\pm$1.21} & 51.30{\scriptsize$\pm$2.32} & 62.40{\scriptsize$\pm$\cellcolor{lightpink}{1.35}}
& 84.39{\scriptsize$\pm$1.13} & 70.24{\scriptsize$\pm$10.55} & 76.16{\scriptsize$\pm$\cellcolor{lightpink}{5.85}}
& 97.05{\scriptsize$\pm$0.17} & \textbf{94.57{\scriptsize$\pm$0.25}} & 95.79{\scriptsize$\pm$\cellcolor{lightpink}{0.21}}
& 83.99{\scriptsize$\pm$0.76} & 67.52{\scriptsize$\pm$2.04} & 74.00{\scriptsize$\pm$\cellcolor{lightpink}{1.32}}
& 14.56{\scriptsize$\pm$0.69} & 21.83{\scriptsize$\pm$3.00} & 18.52{\scriptsize$\pm$\cellcolor{lightpink}{1.66}}\\

PromptFolio
& 92.94{\scriptsize$\pm$0.21} & 52.98{\scriptsize$\pm$6.14} & 67.29{\scriptsize$\pm$\cellcolor{lightpink}{5.05}}
& 89.00{\scriptsize$\pm$1.10} & 42.21{\scriptsize$\pm$4.35} & 57.10{\scriptsize$\pm$\cellcolor{lightpink}{3.78}}
& 94.78{\scriptsize$\pm$0.82} & 64.94{\scriptsize$\pm$6.65} & 76.84{\scriptsize$\pm$\cellcolor{lightpink}{4.42}}
& 96.57{\scriptsize$\pm$0.72} & 82.15{\scriptsize$\pm$5.08} & 88.67{\scriptsize$\pm$\cellcolor{lightpink}{2.66}}
& 86.92{\scriptsize$\pm$1.15} & 50.85{\scriptsize$\pm$4.83} & 62.04{\scriptsize$\pm$\cellcolor{lightpink}{3.90}}
& 13.84{\scriptsize$\pm$0.99} & 21.39{\scriptsize$\pm$1.83} & 21.56{\scriptsize$\pm$\cellcolor{lightpink}{1.94}}\\
\midrule
BLEND
& 95.89{\scriptsize$\pm$0.06} & 91.46{\scriptsize$\pm$0.06} & \textbf{93.62{\scriptsize$\pm$\cellcolor{lightpink}{0.03}}}
& \textbf{94.71{\scriptsize$\pm$0.37}} & 62.47{\scriptsize$\pm$0.49} & \textbf{75.28{\scriptsize$\pm$\cellcolor{lightpink}{0.46}}}
& 99.07{\scriptsize$\pm$0.33} & \textbf{75.30{\scriptsize$\pm$2.19}} & \textbf{85.54{\scriptsize$\pm$\cellcolor{lightpink}{1.36}}}
& \textbf{99.62{\scriptsize$\pm$0.06}} & 94.25{\scriptsize$\pm$0.01} & \textbf{96.86{\scriptsize$\pm$\cellcolor{lightpink}{0.03}}}
& \textbf{93.65{\scriptsize$\pm$0.20}} & \textbf{75.18{\scriptsize$\pm$0.41}} & \textbf{82.83{\scriptsize$\pm$\cellcolor{lightpink}{0.29}}}
& 8.33{\scriptsize$\pm$0.17} & \textbf{17.11{\scriptsize$\pm$0.63}} & \textbf{14.01{\scriptsize$\pm$\cellcolor{lightpink}{0.39}}}\\
\midrule 
$\Delta$ ($p-$value) 
& & & \cellcolor{lightpink}{\textcolor{darkgreen}{\textbf{+1.17}}} (0.001)
& & & \cellcolor{lightpink}{\textcolor{darkgreen}{\textbf{+5.56}}} (0.005)
&  &  & \cellcolor{lightpink}{\textcolor{darkgreen}{\textbf{+8.45}}} (0.028)
& & & \cellcolor{lightpink}{\textcolor{darkgreen}{\textbf{+0.42}}} (0.012)
& &  & \cellcolor{lightpink}{\textcolor{darkgreen}{\textbf{+5.64}}}
& &  &\cellcolor{lightpink}{\textcolor{darkgreen}{\textbf{-2.72}}} \\
\bottomrule
\end{tabular}
}
\end{table*}

\begin{table*}[h]
\centering
\caption{Results for a scenario with $N=30$ clients. Performance comparisons (measured in accuracy, \%) between BLEND and baselines. The row $\Delta$ denotes the performance gap between BLEND and the best-performing baseline. Standard deviations across three random seeds are shown through sub-indices.  Statistical significance is assessed using a two-sided paired $t$-test on seed-level H-Mean scores between BLEND and the strongest baseline, with improvements statistically significant at $p<0.05$.
}
\vspace{-2mm}
\label{tab:main_results3}
\resizebox{\textwidth}{!}{
\begin{tabular}{l|ccc|ccc|ccc|ccc|ccc|ccc}
\toprule
& \multicolumn{3}{c|}{\cellcolor{gray!20}{UCF101}}
& \multicolumn{3}{c|}{\cellcolor{gray!20}{Oxford Pets}}
& \multicolumn{3}{c|}{\cellcolor{gray!20}{Stanford Cars}}
& \multicolumn{3}{c|}{\cellcolor{gray!20}{Flowers}}
& \multicolumn{3}{c|}{\cellcolor{gray!20}{Imagenet}}
& \multicolumn{3}{c}{\cellcolor{gray!20}{FGVC Aircraft}} \\
\cmidrule(lr){2-4}\cmidrule(lr){5-7}\cmidrule(lr){8-10}\cmidrule(lr){11-13}\cmidrule(lr){14-16}\cmidrule(lr){17-19}
& Per. & Gen. & \cellcolor{lightpink}{H-Mean} & Per. & Gen. & \cellcolor{lightpink}{H-Mean} & Per. & Gen. & \cellcolor{lightpink}{H-Mean} &Per. & Gen. & \cellcolor{lightpink}{H-Mean} & Per. & Gen. & \cellcolor{lightpink}{H-Mean} & Per. & Gen. & \cellcolor{lightpink}{H-Mean} \\
\midrule

Zero-shot CLIP
& 68.10{\scriptsize$\pm$0.00} & 75.18{\scriptsize$\pm$0.00} & 71.46{\scriptsize$\pm$\cellcolor{lightpink}{0.00}}
& 89.47{\scriptsize$\pm$0.00} & 96.81{\scriptsize$\pm$0.00} & 93.00{\scriptsize$\pm$\cellcolor{lightpink}{0.00}}
& 63.77{\scriptsize$\pm$0.00} & \textbf{74.94{\scriptsize$\pm$0.00}} & 68.91{\scriptsize$\pm$\cellcolor{lightpink}{0.00}}
& 69.42{\scriptsize$\pm$0.00} & \textbf{76.74{\scriptsize$\pm$0.00}} & 72.90{\scriptsize$\pm$\cellcolor{lightpink}{0.00}}
& 72.40{\scriptsize$\pm$0.00} & 68.08{\scriptsize$\pm$0.00} & 70.17{\scriptsize$\pm$\cellcolor{lightpink}{0.00}}
& 27.49{\scriptsize$\pm$0.00} & 33.05{\scriptsize$\pm$0.00} & 30.02{\scriptsize$\pm$\cellcolor{lightpink}{0.00}}\\

MMA+FedAvg
& 77.84{\scriptsize$\pm$0.74} & 77.54{\scriptsize$\pm$1.06} & 77.68{\scriptsize$\pm$\cellcolor{lightpink}{0.55}}
& 94.78{\scriptsize$\pm$0.23} & \textbf{97.50{\scriptsize$\pm$0.05}} & 96.12{\scriptsize$\pm$\cellcolor{lightpink}{0.10}}
& 72.45{\scriptsize$\pm$1.07} & 71.57{\scriptsize$\pm$0.40} & 72.00{\scriptsize$\pm$\cellcolor{lightpink}{0.49}}
& 92.89{\scriptsize$\pm$0.49} & 73.19{\scriptsize$\pm$1.61} & 81.86{\scriptsize$\pm$\cellcolor{lightpink}{0.82}}
& 77.57{\scriptsize$\pm$0.25} & 69.20{\scriptsize$\pm$0.30} & 73.15{\scriptsize$\pm$\cellcolor{lightpink}{0.28}}
& 32.25{\scriptsize$\pm$0.96} & 35.15{\scriptsize$\pm$1.04} & 33.64{\scriptsize$\pm$\cellcolor{lightpink}{0.98}} \\

pFedMMA
& 95.70{\scriptsize$\pm$0.45} & 71.16{\scriptsize$\pm$0.56} & 81.63{\scriptsize$\pm$\cellcolor{lightpink}{0.39}}
& 92.72{\scriptsize$\pm$2.16} & 95.14{\scriptsize$\pm$0.22} & 93.90{\scriptsize$\pm$\cellcolor{lightpink}{1.02}}
& 78.41{\scriptsize$\pm$2.96} & 58.46{\scriptsize$\pm$1.80} & 66.94{\scriptsize$\pm$\cellcolor{lightpink}{1.61}}
& 92.62{\scriptsize$\pm$1.16} & 71.28{\scriptsize$\pm$0.81} & 80.55{\scriptsize$\pm$\cellcolor{lightpink}{0.28}}
& 71.74{\scriptsize$\pm$1.02} & 55.29{\scriptsize$\pm$1.04} & 62.45{\scriptsize$\pm$\cellcolor{lightpink}{1.05}}
& 64.05{\scriptsize$\pm$3.77} & 27.65{\scriptsize$\pm$0.43} & 38.61{\scriptsize$\pm$\cellcolor{lightpink}{1.10}}\\

pFedMoAP
& 94.82{\scriptsize$\pm$0.09} & 67.21{\scriptsize$\pm$0.00} & 78.66{\scriptsize$\pm$\cellcolor{lightpink}{0.03}}
& 98.66{\scriptsize$\pm$0.11} & 84.78{\scriptsize$\pm$0.00} & 91.20{\scriptsize$\pm$\cellcolor{lightpink}{0.05}}
& 90.56{\scriptsize$\pm$0.63} & 68.44{\scriptsize$\pm$0.00} & 77.96{\scriptsize$\pm$\cellcolor{lightpink}{0.23}}
& 98.46{\scriptsize$\pm$0.66} & 71.43{\scriptsize$\pm$0.00} & 82.80{\scriptsize$\pm$\cellcolor{lightpink}{0.24}}
& 90.83{\scriptsize$\pm$0.00} & 62.50{\scriptsize$\pm$0.00} & 74.05{\scriptsize$\pm$\cellcolor{lightpink}{0.00}}
& 81.52{\scriptsize$\pm$1.10} & 18.53{\scriptsize$\pm$0.00} & 30.19{\scriptsize$\pm$\cellcolor{lightpink}{0.08}} \\

FedPHA
& \textbf{96.38{\scriptsize$\pm$0.56}} & 47.72{\scriptsize$\pm$2.62} & 63.78{\scriptsize$\pm$\cellcolor{lightpink}{2.22}}
& 98.86{\scriptsize$\pm$0.20} & 49.07{\scriptsize$\pm$2.02} & 65.56{\scriptsize$\pm$\cellcolor{lightpink}{1.76}}
& \textbf{94.59{\scriptsize$\pm$0.18}} & 39.27{\scriptsize$\pm$0.35} & 55.50{\scriptsize$\pm$\cellcolor{lightpink}{0.38}}
& 98.83{\scriptsize$\pm$0.57} & 33.60{\scriptsize$\pm$0.66} & 50.15{\scriptsize$\pm$\cellcolor{lightpink}{0.66}}
& \textbf{91.68{\scriptsize$\pm$0.02}} & 36.76{\scriptsize$\pm$0.36} & 52.48{\scriptsize$\pm$\cellcolor{lightpink}{0.37}}
& \textbf{83.14{\scriptsize$\pm$0.65}} & 7.11{\scriptsize$\pm$1.30} & 13.06{\scriptsize$\pm$\cellcolor{lightpink}{2.20}}\\

FedOTP
& 89.03{\scriptsize$\pm$2.76} & 23.42{\scriptsize$\pm$2.80} & 36.95{\scriptsize$\pm$\cellcolor{lightpink}{3.18}}
& 96.83{\scriptsize$\pm$0.69} & 39.18{\scriptsize$\pm$6.72} & 55.41{\scriptsize$\pm$\cellcolor{lightpink}{6.99}}
& 74.52{\scriptsize$\pm$0.52} & 16.43{\scriptsize$\pm$1.26} & 26.91{\scriptsize$\pm$\cellcolor{lightpink}{1.73}}
& 92.48{\scriptsize$\pm$0.83} & 32.76{\scriptsize$\pm$3.04} & 48.32{\scriptsize$\pm$\cellcolor{lightpink}{3.38}}
& 84.50{\scriptsize$\pm$0.47} & 12.47{\scriptsize$\pm$0.61} & 21.72{\scriptsize$\pm$\cellcolor{lightpink}{0.94}}
& 68.05{\scriptsize$\pm$2.58} & 5.42{\scriptsize$\pm$1.65} & 9.99{\scriptsize$\pm$\cellcolor{lightpink}{2.88}} \\

PromptFL
& 81.25{\scriptsize$\pm$0.61} & 72.15{\scriptsize$\pm$2.31} & 76.40{\scriptsize$\pm$\cellcolor{lightpink}{1.08}}
& 95.84{\scriptsize$\pm$0.15} & 95.34{\scriptsize$\pm$1.68} & 95.58{\scriptsize$\pm$\cellcolor{lightpink}{0.92}}
& 70.92{\scriptsize$\pm$1.73} & 71.61{\scriptsize$\pm$1.37} & 71.25{\scriptsize$\pm$\cellcolor{lightpink}{1.33}}
& 94.68{\scriptsize$\pm$0.35} & 68.23{\scriptsize$\pm$2.08} & 79.29{\scriptsize$\pm$\cellcolor{lightpink}{1.42}}
& 75.37{\scriptsize$\pm$0.09} & 69.33{\scriptsize$\pm$0.04} & 72.22{\scriptsize$\pm$\cellcolor{lightpink}{0.06}}
& 28.18{\scriptsize$\pm$9.06} & 20.54{\scriptsize$\pm$10.59} & 23.38{\scriptsize$\pm$\cellcolor{lightpink}{10.48}} \\

FedPGP
& 79.50{\scriptsize$\pm$2.05} & 70.87{\scriptsize$\pm$0.05} & 74.92{\scriptsize$\pm$\cellcolor{lightpink}{0.88}}
& 95.75{\scriptsize$\pm$0.29} & 96.51{\scriptsize$\pm$1.10} & 96.13{\scriptsize$\pm$\cellcolor{lightpink}{0.69}}
& 77.31{\scriptsize$\pm$0.04} & 68.15{\scriptsize$\pm$0.64} & 72.44{\scriptsize$\pm$\cellcolor{lightpink}{0.38}}
& 91.11{\scriptsize$\pm$1.57} & 72.36{\scriptsize$\pm$0.18} & 80.66{\scriptsize$\pm$\cellcolor{lightpink}{0.73}}
& 87.05{\scriptsize$\pm$0.10} & 46.79{\scriptsize$\pm$0.08} & 60.86{\scriptsize$\pm$\cellcolor{lightpink}{0.11}}
& 40.39{\scriptsize$\pm$2.37} & 26.65{\scriptsize$\pm$3.00} & 31.89{\scriptsize$\pm$\cellcolor{lightpink}{1.43}} \\
PromptFolio
& 88.63{\scriptsize$\pm$1.07} & 53.14{\scriptsize$\pm$3.92} & 66.32{\scriptsize$\pm$\cellcolor{lightpink}{2.77}}
 & 96.74{\scriptsize$\pm$0.03} & 71.90{\scriptsize$\pm$4.39} & 82.41{\scriptsize$\pm$\cellcolor{lightpink}{2.90}}
 & 75.55{\scriptsize$\pm$2.35} & 54.88{\scriptsize$\pm$5.78} & 63.49{\scriptsize$\pm$\cellcolor{lightpink}{4.72}}
 & 96.11{\scriptsize$\pm$0.39} & 27.70{\scriptsize$\pm$6.70} & 42.55{\scriptsize$\pm$\cellcolor{lightpink}{8.06}}
 & 87.19{\scriptsize$\pm$0.32} & 50.86{\scriptsize$\pm$5.70} & 63.00{\scriptsize$\pm$\cellcolor{lightpink}{7.46}}
 & 53.79{\scriptsize$\pm$0.07} & 7.74{\scriptsize$\pm$0.42} & 13.53{\scriptsize$\pm$\cellcolor{lightpink}{0.64}}\\
\midrule

BLEND
& 95.95{\scriptsize$\pm$0.18} & \textbf{78.20{\scriptsize$\pm$0.14}} & \textbf{86.17{\scriptsize$\pm$\cellcolor{lightpink}{0.03}}}
& \textbf{99.03{\scriptsize$\pm$0.03}} & 96.74{\scriptsize$\pm$0.08} & \textbf{97.87{\scriptsize$\pm$\cellcolor{lightpink}{0.04}}}
& 93.87{\scriptsize$\pm$0.28} & 74.34{\scriptsize$\pm$0.10} & \textbf{82.97{\scriptsize$\pm$\cellcolor{lightpink}{0.09}}}
& \textbf{98.97{\scriptsize$\pm$0.08}} & 76.49{\scriptsize$\pm$0.19} & \textbf{86.29{\scriptsize$\pm$\cellcolor{lightpink}{0.15}}}
& 90.56{\scriptsize$\pm$0.04} & \textbf{70.24{\scriptsize$\pm$0.21}} & \textbf{79.12{\scriptsize$\pm$\cellcolor{lightpink}{0.12}}}
& 73.34{\scriptsize$\pm$0.31} & \textbf{35.36{\scriptsize$\pm$0.26}} & \textbf{47.71{\scriptsize$\pm$\cellcolor{lightpink}{0.17}}} \\
\midrule
$\Delta$ ($p-$value)
& & & \cellcolor{lightpink}{\textcolor{darkgreen}{\textbf{+4.54}}} (0.021)
&  & & \cellcolor{lightpink}{\textcolor{darkgreen}{\textbf{+1.74}}} (0.007)
& &  & \cellcolor{lightpink}{\textcolor{darkgreen}{\textbf{+5.01}}} (0.012)
& & & \cellcolor{lightpink}{\textcolor{darkgreen}{\textbf{+3.49}}} (0.004)
& &  & \cellcolor{lightpink}{\textcolor{darkgreen}{\textbf{+5.07}}} (0.001)
& & & \cellcolor{lightpink}{\textcolor{darkgreen}{\textbf{+9.10}}} (0.003) \\
\hline
\hline

& \multicolumn{3}{c|}{\cellcolor{gray!20}{Food101}}
& \multicolumn{3}{c|}{\cellcolor{gray!20}{DTD}}
& \multicolumn{3}{c|}{\cellcolor{gray!20}{EuroSAT}}
& \multicolumn{3}{c|}{\cellcolor{gray!20}{Caltech}}
& \multicolumn{3}{c|}{\textbf{\cellcolor{gray!20}{Cross-Dataset Avg. ($\uparrow$)}}}
& \multicolumn{3}{c}{\textbf{\cellcolor{gray!20}{Cross-Dataset  Std. ($\downarrow$)}}} \\
\cmidrule(lr){2-4}\cmidrule(lr){5-7}\cmidrule(lr){8-10}\cmidrule(lr){11-13}\cmidrule(lr){14-16}\cmidrule(lr){17-19}
& Per.& Gen. & \cellcolor{lightpink}{H-Mean} & Per.& Gen. & \cellcolor{lightpink}{H-Mean} & Per.& Gen. & \cellcolor{lightpink}{H-Mean} &Per.& Gen. & \cellcolor{lightpink}{H-Mean} & Per.& Gen. & \cellcolor{lightpink}{H-Mean} & Per.& Gen. & \cellcolor{lightpink}{H-Mean} \\
\midrule

Zero-shot CLIP
& 89.44{\scriptsize$\pm$0.00} & 90.67{\scriptsize$\pm$0.00} & 90.05{\scriptsize$\pm$\cellcolor{lightpink}{0.00}}
& 54.75{\scriptsize$\pm$0.00} & 58.94{\scriptsize$\pm$0.00} & 56.76{\scriptsize$\pm$\cellcolor{lightpink}{0.00}}
& 50.21{\scriptsize$\pm$0.00} & 69.87{\scriptsize$\pm$0.00} & 58.43{\scriptsize$\pm$\cellcolor{lightpink}{0.00}}
& 97.16{\scriptsize$\pm$0.00} & 94.21{\scriptsize$\pm$0.00} & 95.66{\scriptsize$\pm$\cellcolor{lightpink}{0.00}}
& 68.22{\scriptsize$\pm$0.00} & 73.85{\scriptsize$\pm$0.00} & 70.74{\scriptsize$\pm$\cellcolor{lightpink}{0.00}}
& 19.86{\scriptsize$\pm$0.00} & 17.81{\scriptsize$\pm$0.00} & 18.72{\scriptsize$\pm$\cellcolor{lightpink}{0.00}}\\

MMA+FedAvg
& 89.56{\scriptsize$\pm$0.31} & \textbf{91.45{\scriptsize$\pm$0.06}} & 90.50{\scriptsize$\pm$\cellcolor{lightpink}{0.17}}
& 72.85{\scriptsize$\pm$1.62} & \textbf{64.94{\scriptsize$\pm$0.80}} & 68.65{\scriptsize$\pm$\cellcolor{lightpink}{0.51}}
& 93.05{\scriptsize$\pm$0.84} & 72.08{\scriptsize$\pm$5.64} & 81.10{\scriptsize$\pm$\cellcolor{lightpink}{3.34}}
& 95.95{\scriptsize$\pm$0.10} & 94.40{\scriptsize$\pm$0.22} & 95.17{\scriptsize$\pm$\cellcolor{lightpink}{0.08}}
& 79.92{\scriptsize$\pm$0.66} & 74.70{\scriptsize$\pm$1.12} & 76.99{\scriptsize$\pm$\cellcolor{lightpink}{0.73}}
& 18.17{\scriptsize$\pm$0.45} & 17.04{\scriptsize$\pm$1.58} & 17.05{\scriptsize$\pm$\cellcolor{lightpink}{0.92}}\\

pFedMMA
& \textbf{97.45{\scriptsize$\pm$0.19}} & 88.81{\scriptsize$\pm$0.34} & 92.93{\scriptsize$\pm$\cellcolor{lightpink}{0.25}}
& 85.46{\scriptsize$\pm$1.87} & 59.55{\scriptsize$\pm$0.63} & 70.18{\scriptsize$\pm$\cellcolor{lightpink}{0.35}}
& 79.53{\scriptsize$\pm$11.47} & 50.52{\scriptsize$\pm$6.17} & 61.69{\scriptsize$\pm$\cellcolor{lightpink}{7.79}}
& 99.45{\scriptsize$\pm$0.09} & 93.67{\scriptsize$\pm$0.11} & 96.48{\scriptsize$\pm$\cellcolor{lightpink}{0.10}}
& 85.71{\scriptsize$\pm$2.51} & 67.15{\scriptsize$\pm$1.21} & 74.54{\scriptsize$\pm$\cellcolor{lightpink}{1.39}}
& 11.31{\scriptsize$\pm$3.19} & 20.25{\scriptsize$\pm$1.71} & 17.20{\scriptsize$\pm$\cellcolor{lightpink}{2.18}} \\

pFedMoAP
& 96.58{\scriptsize$\pm$0.17} & 86.62{\scriptsize$\pm$0.00} & 91.33{\scriptsize$\pm$\cellcolor{lightpink}{0.08}}
& 96.21{\scriptsize$\pm$0.32} & 47.66{\scriptsize$\pm$0.00} & 63.74{\scriptsize$\pm$\cellcolor{lightpink}{0.07}}
& 96.40{\scriptsize$\pm$1.29} & 63.41{\scriptsize$\pm$0.00} & 76.50{\scriptsize$\pm$\cellcolor{lightpink}{0.41}}
& 99.25{\scriptsize$\pm$0.13} & 93.64{\scriptsize$\pm$0.00} & 96.36{\scriptsize$\pm$\cellcolor{lightpink}{0.06}}
& 94.33{\scriptsize$\pm$0.45} & 66.42{\scriptsize$\pm$0.00} & 76.28{\scriptsize$\pm$\cellcolor{lightpink}{0.12}}
& 5.13{\scriptsize$\pm$0.43} & 20.46{\scriptsize$\pm$0.00} & 17.86{\scriptsize$\pm$\cellcolor{lightpink}{0.12}}\\

FedPHA
& \textbf{97.45{\scriptsize$\pm$0.08}} & 50.01{\scriptsize$\pm$0.67} & 66.10{\scriptsize$\pm$\cellcolor{lightpink}{0.57}}
& \textbf{97.26{\scriptsize$\pm$0.12}} & 27.78{\scriptsize$\pm$1.74} & 43.19{\scriptsize$\pm$\cellcolor{lightpink}{2.09}}
& 99.86{\scriptsize$\pm$0.14} & 31.07{\scriptsize$\pm$1.62} & 47.37{\scriptsize$\pm$\cellcolor{lightpink}{1.91}}
& \textbf{99.72{\scriptsize$\pm$0.02}} & 77.24{\scriptsize$\pm$3.31} & 87.01{\scriptsize$\pm$\cellcolor{lightpink}{2.11}}
& \textbf{95.78{\scriptsize$\pm$0.25}} & 39.96{\scriptsize$\pm$1.47} & 54.42{\scriptsize$\pm$\cellcolor{lightpink}{1.43}}
& \textbf{4.84{\scriptsize$\pm$0.23}} & 17.30{\scriptsize$\pm$0.94} & 18.22{\scriptsize$\pm$\cellcolor{lightpink}{0.78}} \\

FedOTP
& 91.31{\scriptsize$\pm$0.90} & 38.40{\scriptsize$\pm$1.38} & 54.05{\scriptsize$\pm$\cellcolor{lightpink}{1.48}}
& 96.03{\scriptsize$\pm$0.29} & 23.42{\scriptsize$\pm$0.72} & 37.66{\scriptsize$\pm$\cellcolor{lightpink}{0.95}}
& \textbf{100.00{\scriptsize$\pm$0.00}} & 31.36{\scriptsize$\pm$1.86} & 47.72{\scriptsize$\pm$\cellcolor{lightpink}{2.16}}
& 96.67{\scriptsize$\pm$0.23} & 63.29{\scriptsize$\pm$4.86} & 76.39{\scriptsize$\pm$\cellcolor{lightpink}{3.47}}
& 88.94{\scriptsize$\pm$0.93} & 28.62{\scriptsize$\pm$2.49} & 41.51{\scriptsize$\pm$\cellcolor{lightpink}{2.72}}
& 9.88{\scriptsize$\pm$0.91} & \textbf{15.61{\scriptsize$\pm$1.85}} & 18.13{\scriptsize$\pm$\cellcolor{lightpink}{1.69}} \\

PromptFL
& 90.06{\scriptsize$\pm$0.25} & 91.04{\scriptsize$\pm$0.67} & 90.55{\scriptsize$\pm$\cellcolor{lightpink}{0.35}}
& 79.45{\scriptsize$\pm$0.36} & 50.24{\scriptsize$\pm$2.94} & 61.51{\scriptsize$\pm$\cellcolor{lightpink}{2.33}}
& 88.32{\scriptsize$\pm$0.74} & 68.45{\scriptsize$\pm$2.98} & 77.09{\scriptsize$\pm$\cellcolor{lightpink}{1.92}}
& 97.14{\scriptsize$\pm$0.07} & 93.74{\scriptsize$\pm$0.73} & 95.41{\scriptsize$\pm$\cellcolor{lightpink}{0.36}}
& 80.12{\scriptsize$\pm$1.34} & 70.07{\scriptsize$\pm$2.54} & 74.27{\scriptsize$\pm$\cellcolor{lightpink}{2.03}}
& 19.29{\scriptsize$\pm$2.61} & 21.21{\scriptsize$\pm$2.84} & 19.92{\scriptsize$\pm$\cellcolor{lightpink}{2.90}} \\

FedPGP
& 90.47{\scriptsize$\pm$0.06} & 91.40{\scriptsize$\pm$0.01} & 90.93{\scriptsize$\pm$\cellcolor{lightpink}{0.03}}
& 76.60{\scriptsize$\pm$3.49} & 50.51{\scriptsize$\pm$1.86} & 60.78{\scriptsize$\pm$\cellcolor{lightpink}{0.25}}
& 85.60{\scriptsize$\pm$2.48} & \textbf{75.64{\scriptsize$\pm$7.08}} & 80.23{\scriptsize$\pm$\cellcolor{lightpink}{5.09}}
& 97.10{\scriptsize$\pm$0.33} & \textbf{94.42{\scriptsize$\pm$1.42}} & 95.73{\scriptsize$\pm$\cellcolor{lightpink}{0.57}}
& 82.09{\scriptsize$\pm$1.27} & 69.33{\scriptsize$\pm$1.53} & 74.46{\scriptsize$\pm$\cellcolor{lightpink}{1.01}}
& 15.49{\scriptsize$\pm$1.21} & 21.43{\scriptsize$\pm$2.07} & 18.58{\scriptsize$\pm$\cellcolor{lightpink}{1.42}} \\

PromptFolio
& 93.66{\scriptsize$\pm$0.32} & 52.98{\scriptsize$\pm$6.14} & 67.48{\scriptsize$\pm$\cellcolor{lightpink}{5.10}}
& 89.18{\scriptsize$\pm$0.35} & 42.21{\scriptsize$\pm$4.35} & 57.16{\scriptsize$\pm$\cellcolor{lightpink}{3.94}}
& 91.86{\scriptsize$\pm$0.40} & 64.94{\scriptsize$\pm$6.65} & 75.87{\scriptsize$\pm$\cellcolor{lightpink}{4.44}}
& 98.69{\scriptsize$\pm$0.31} & 82.15{\scriptsize$\pm$5.08} & 89.59{\scriptsize$\pm$\cellcolor{lightpink}{3.15}}
& 87.14{\scriptsize$\pm$0.59} & 50.85{\scriptsize$\pm$4.83} & 62.14{\scriptsize$\pm$\cellcolor{lightpink}{3.97}} 
& 13.43{\scriptsize$\pm$0.68} & 21.39{\scriptsize$\pm$1.83} & 21.53{\scriptsize$\pm$\cellcolor{lightpink}{1.92}}\\

\midrule
BLEND
& 96.58{\scriptsize$\pm$0.19} & 91.32{\scriptsize$\pm$0.11} & \textbf{93.88{\scriptsize$\pm$\cellcolor{lightpink}{0.06}}}
& 95.49{\scriptsize$\pm$0.57} & 61.50{\scriptsize$\pm$0.32} & \textbf{74.81{\scriptsize$\pm$\cellcolor{lightpink}{0.17}}}
& 98.90{\scriptsize$\pm$0.34} & 75.01{\scriptsize$\pm$1.57} & \textbf{85.31{\scriptsize$\pm$\cellcolor{lightpink}{1.09}}}
& 99.61{\scriptsize$\pm$0.06} & 94.17{\scriptsize$\pm$0.04} & \textbf{96.82{\scriptsize$\pm$\cellcolor{lightpink}{0.05}}}
& 94.23{\scriptsize$\pm$0.21} & \textbf{75.34{\scriptsize$\pm$0.30}} & \textbf{83.10{\scriptsize$\pm$\cellcolor{lightpink}{0.20}}}
& 7.46{\scriptsize$\pm$0.16} & 17.01{\scriptsize$\pm$0.43} & \textbf{13.72{\scriptsize$\pm$\cellcolor{lightpink}{0.30}}}\\
\midrule 
$\Delta$ ($p-$value)
&  &  & \cellcolor{lightpink}{\textcolor{darkgreen}{\textbf{+2.55}}} (0.002)
&& & \cellcolor{lightpink}{\textcolor{darkgreen}{\textbf{+4.63}}} (0.004)
&   &  & \cellcolor{lightpink}{\textcolor{darkgreen}{\textbf{+4.21}}} (0.032)
&  && \cellcolor{lightpink}{\textcolor{darkgreen}{\textbf{+0.34}}} (0.011) 
& &  & \cellcolor{lightpink}{\textcolor{darkgreen}{\textbf{+6.11}}}
&  & &\cellcolor{lightpink}{\textcolor{darkgreen}{\textbf{-3.33}}} \\
\bottomrule
\end{tabular}
}
\end{table*}

\clearpage

\clearpage
\section{Large-Scale Setting with Real-World Data Distribution}\label{ap:real_large}

We further evaluate BLEND under realistic, large-scale client data heterogeneity to better reflect practical deployment scenarios. To this end, we use the real-world iNaturalist-User-120k dataset~\cite{hsu2020federated}, and conduct experiments with $N=100$ and $N=200$ sampled clients, where each client follows naturally occurring data distributions reported by the dataset. We use the same network and learning parameters (e.g., $\lambda$ and $\gamma$) as the simulations presented in the main text.

The results are summarized in Table~\ref{tab:large100} and Table~\ref{tab:large200}. For $N=100$ clients (Table~\ref{tab:large100}), BLEND achieves the highest H-Mean of $63.37$, outperforming all baselines by a clear margin. This improvement stems from simultaneously maintaining strong personalization (Base = $86.96$) and better generalization to unseen classes (New = $49.84$). In contrast, while some methods such as pFedMOAP achieve high personalization, they suffer from significantly lower generalization, leading to inferior H-Mean. 

As the number of clients increases to $N=200$ (Table~\ref{tab:large200}), all methods experience a drop in performance due to increased data heterogeneity and reduced per-client data. Nevertheless, BLEND achieve the best performance, with an H-Mean of $57.39$, maintaining a consistent gap over baselines. 

Overall, these results demonstrate that BLEND scales effectively to large client populations and remains robust under realistic, highly non-IID data distributions, consistently achieving a superior balance between personalization and generalization compared to existing methods.

\begin{table*}[h]
\centering
\begin{minipage}{0.32\linewidth}
\centering
\scriptsize
\captionof{table}{$N=100$ clients}
\begin{tabular}{l|ccc}
\hline
\rowcolor{gray!20} Method & Per. & Gen. & H-Mean \\
\hline
FedOTP & 63.31 & 11.46 & \cellcolor{lightpink}19.32 \\
FedPGP & 81.44 & 36.76 & \cellcolor{lightpink}50.64 \\
FedPHA & 81.72 & 21.75 & \cellcolor{lightpink}34.33 \\
MMA+FedAVG & 69.99 & 50.13 & \cellcolor{lightpink}58.41 \\
pFedMMA & 67.06 & 48.15 & \cellcolor{lightpink}56.05 \\
pFedMOAP & 89.55 & 37.89 & \cellcolor{lightpink}53.25 \\
PromptFL & 77.78 & 43.96 & \cellcolor{lightpink}56.17 \\
PromptFolio & 82.36 & 32.96 & \cellcolor{lightpink}47.07 \\
\hline
BLEND & 86.96 & 49.84 & \cellcolor{lightpink}\textbf{63.37} \\
\hline
\end{tabular}
\label{tab:large100}
\end{minipage}
\begin{minipage}{0.32\linewidth}
\centering
\scriptsize
\captionof{table}{$N=200$ clients}
\begin{tabular}{l|ccc}
\hline
\rowcolor{gray!20} Method & Per. & Gen. & H-Mean \\
\hline
FedOTP & 68.35 & 10.72 & \cellcolor{lightpink}18.54 \\
FedPGP & 82.03 & 34.68 & \cellcolor{lightpink}48.63 \\
FedPHA & 86.78 & 17.11 & \cellcolor{lightpink}28.58 \\
MMA+FedAVG & 64.06 & 43.50 & \cellcolor{lightpink}51.82 \\
pFedMMA & 56.48 & 42.38 & \cellcolor{lightpink}48.41 \\
pFedMOAP & 92.43 & 22.49 & \cellcolor{lightpink}36.18 \\
PromptFL & 76.75 & 34.15 & \cellcolor{lightpink}47.23 \\
PromptFolio & 82.03 & 25.12 & \cellcolor{lightpink}38.46 \\
\hline
BLEND & 85.62 & 43.17 & \cellcolor{lightpink}\textbf{57.39} \\
\hline
\end{tabular}
\label{tab:large200}
\end{minipage}
\end{table*}

\section{Different Few-shot Settings}\label{ap:few_shots}
To further explore the effect of the number of shots seen from each class during the model training across the clients, we performed an ablation study sweeping the number of shots on the values $\{2, 4, 8, 16, 32\}$. We have performed this experiment for all of the datasets mentioned in the paper and included the results in Figs.~\ref{fig:shots-UCF101}-\ref{fig:shots-ImageNet}. Also,  Fig.~\ref{fig:shots-avgAll} depicts the average performance across all the datasets.

Intuitively, as the number of shots increase, we would expect each participating client's local model to become more adapted to the local data distribution on that specific client, hence increasing personalization (this is reflected in  Fig.~\ref{fig:shots-avgAll}-left and in most of the datasets shown in Figs.~\ref{fig:shots-UCF101}-\ref{fig:shots-ImageNet}). Simultaneously, increasing the number of shots can result in the local models being more biased to their corresponding client's data distribution, hence damaging the generalization performance. This hypothesis is also confirmed in the generalization performance subplot of Fig.~\ref{fig:shots-avgAll}-middle and in most of the datasets shown in Figs.~\ref{fig:shots-UCF101}-\ref{fig:shots-ImageNet}. One key takeaway is that BLEND's average performance across the datasets is better than all other methods when H-Mean performance is considered (Fig.~\ref{fig:shots-avgAll}-right). Further, across all individual datasets,  BLEND's H-Mean performance is the best when the number of shots is less than or equal to 16, while being marginally outperformed in two datasets when the number of shots is $32$. This demonstrates the exceptional capability of BLEND in limited few-shot learning settings.


\begin{figure}[h]
    \centering
    \hspace{-14mm}
    \includegraphics[width=.9\linewidth]{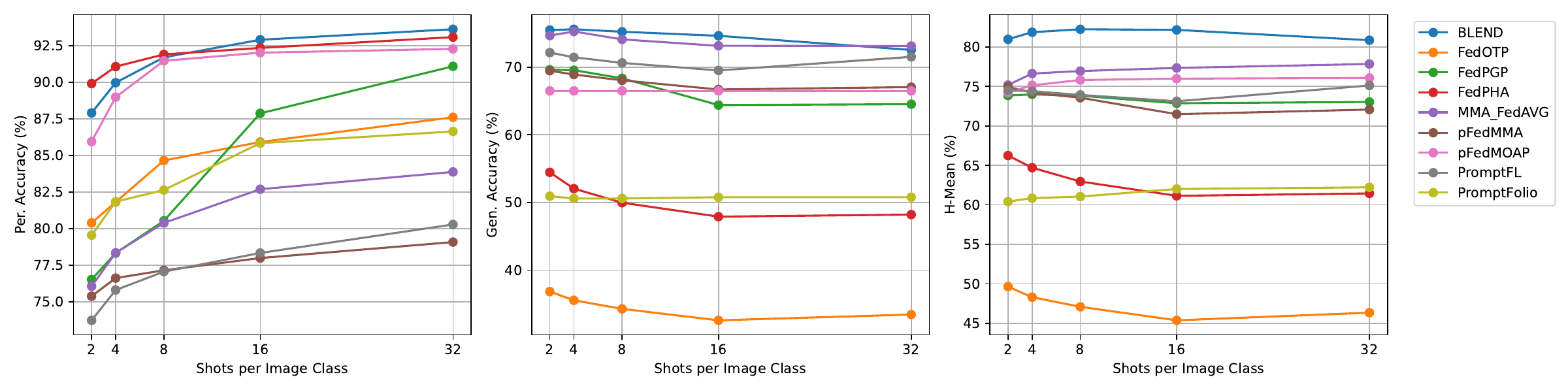}
    \caption{Personalization and generalization test accuracies alongside their harmonic mean (H-Mean) under changing number of shots per class \textbf{averaged across all datasets}.}
    \label{fig:shots-avgAll}
\end{figure}
\begin{figure}[h]
    \centering
    \includegraphics[width=.99\linewidth]{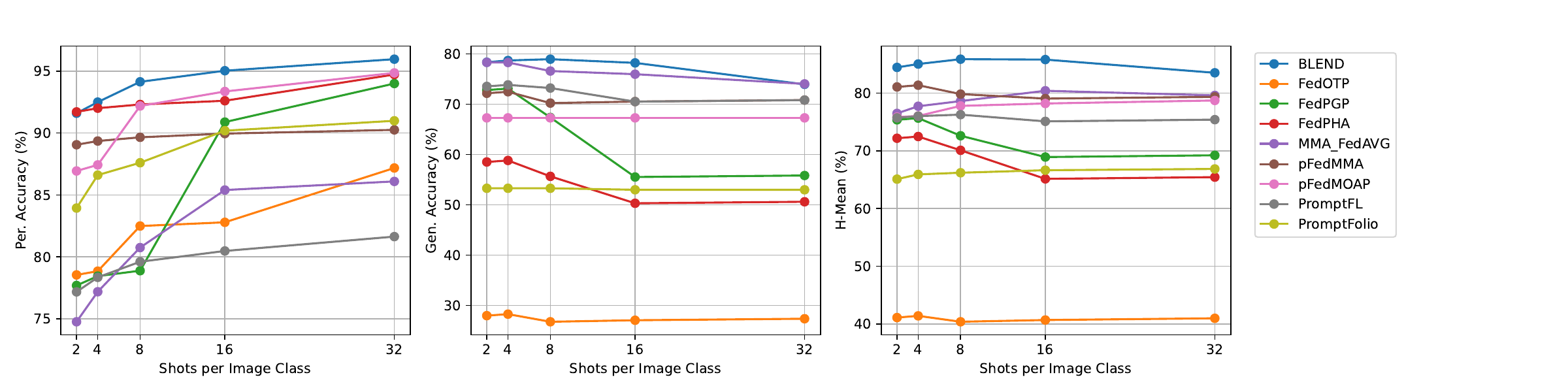}
    \caption{Personalization and generalization test accuracies alongside their harmonic mean (H-Mean) under changing number of shots per class for the \textbf{\textit{UCF101}} dataset.}
    \label{fig:shots-UCF101}
\end{figure}
\begin{figure}[h]
    \centering
    \includegraphics[width=.99\linewidth]{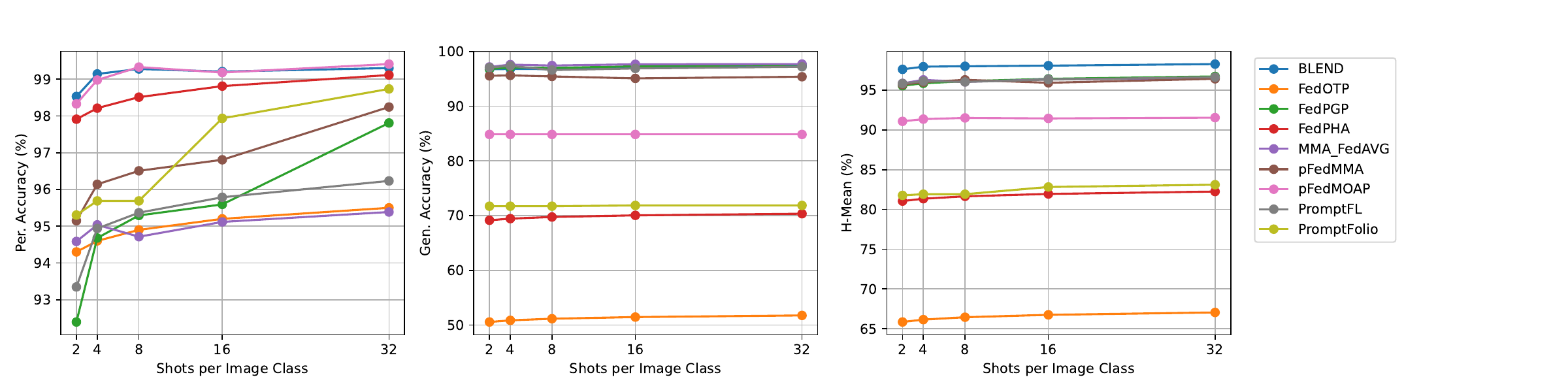}
    \caption{Personalization and generalization test accuracies alongside their harmonic mean (H-Mean) under changing number of shots per class for the \textbf{\textit{OxfordPets}} dataset.}
    \label{fig:shots-OxPets}
\end{figure}
\begin{figure}[h]
    \centering
    \includegraphics[width=.99\linewidth]{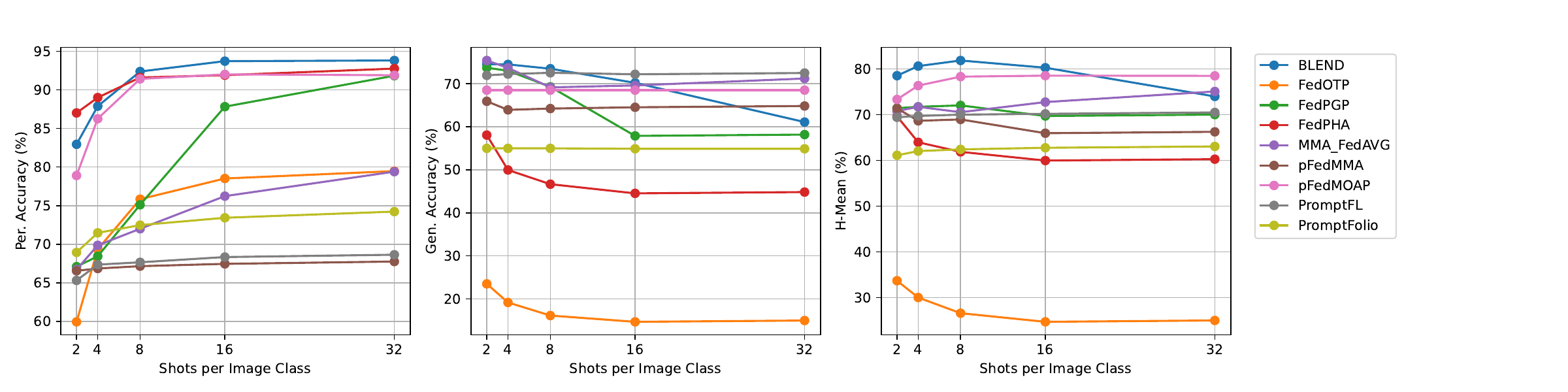}
    \caption{Personalization and generalization test accuracies alongside their harmonic mean (H-Mean) under changing number of shots per class for the \textbf{\textit{Stanford-Cars}} dataset.}
    \label{fig:shots-StanCards}
\end{figure}
\begin{figure}[h]
    \centering
    \includegraphics[width=.99\linewidth]{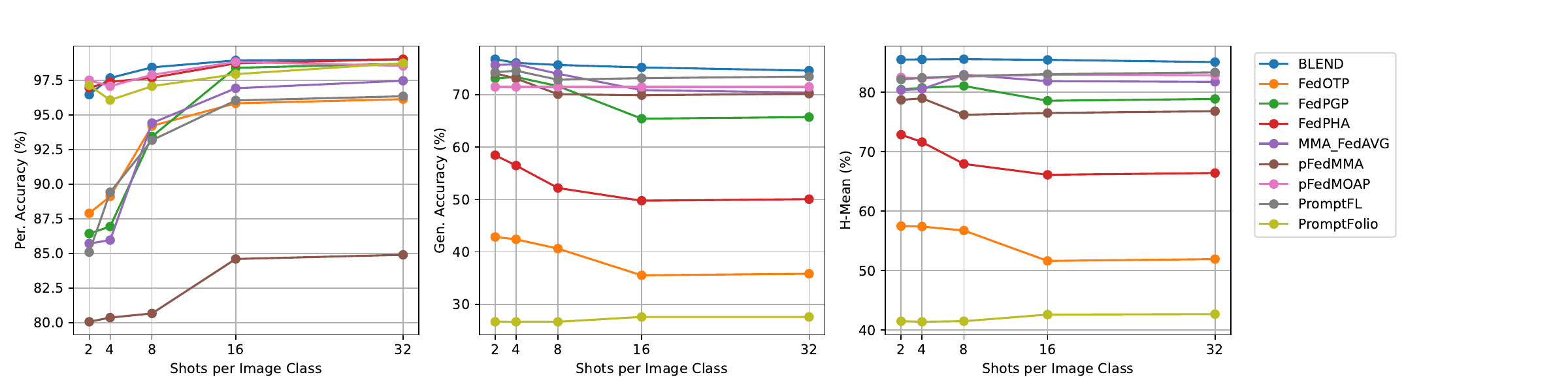}
    \caption{Personalization and generalization test accuracies alongside their harmonic mean (H-Mean) under changing number of shots per class for the \textbf{\textit{Flowers}} dataset.}
    \label{fig:shots-OxFlowers}
\end{figure}
\begin{figure}[h]
    \centering
    \includegraphics[width=.99\linewidth]{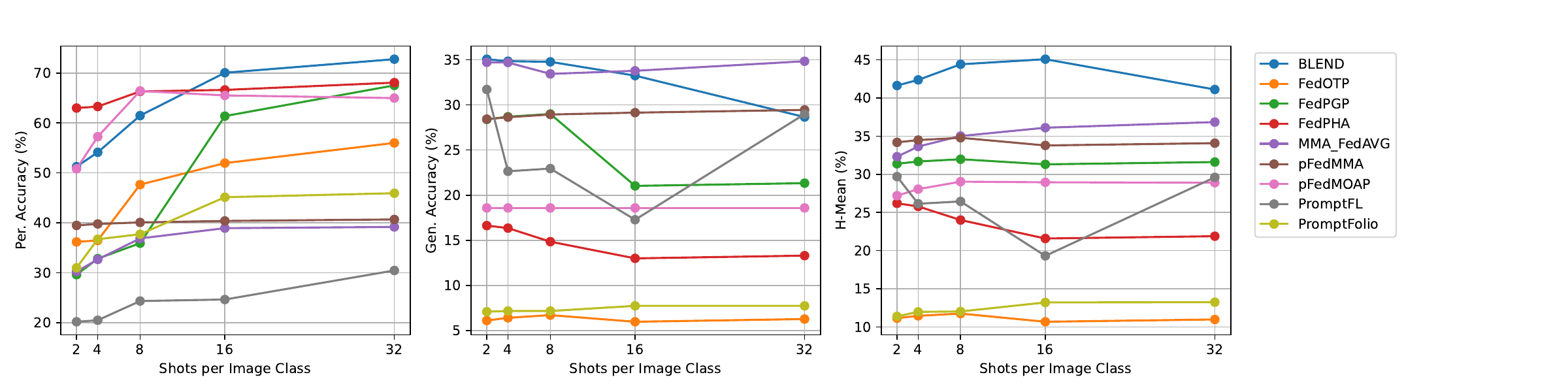}
    \caption{Personalization and generalization test accuracies alongside their harmonic mean (H-Mean) under changing number of shots per class for the \textbf{\textit{FGVC Aircraft}} dataset.}
    \label{fig:shots-fvgcair}
\end{figure}
\begin{figure}[h]
    \centering
    \includegraphics[width=.99\linewidth]{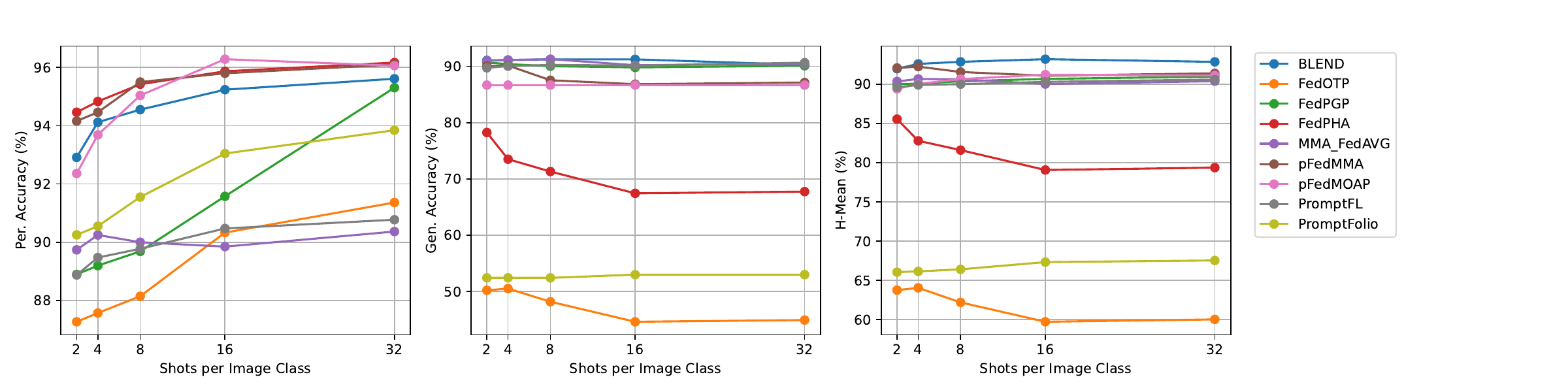}
    \caption{Personalization and generalization test accuracies alongside their harmonic mean (H-Mean) under changing number of shots per class for the \textbf{\textit{Food101}} dataset.}
    \label{fig:shots-Food101}
\end{figure}
\begin{figure}[h]
    \centering
    \includegraphics[width=.99\linewidth]{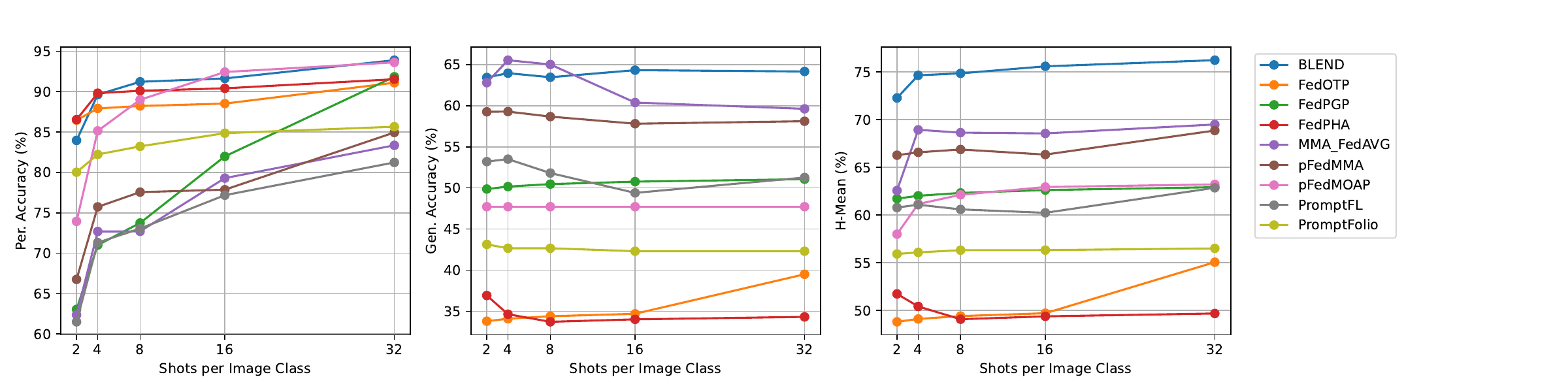}
    \caption{Personalization and generalization test accuracies alongside their harmonic mean (H-Mean) under changing number of shots per class for the \textbf{\textit{DTD}} dataset.}
    \label{fig:shots-DTD}
\end{figure}
\begin{figure}[h]
    \centering
    \includegraphics[width=.99\linewidth]{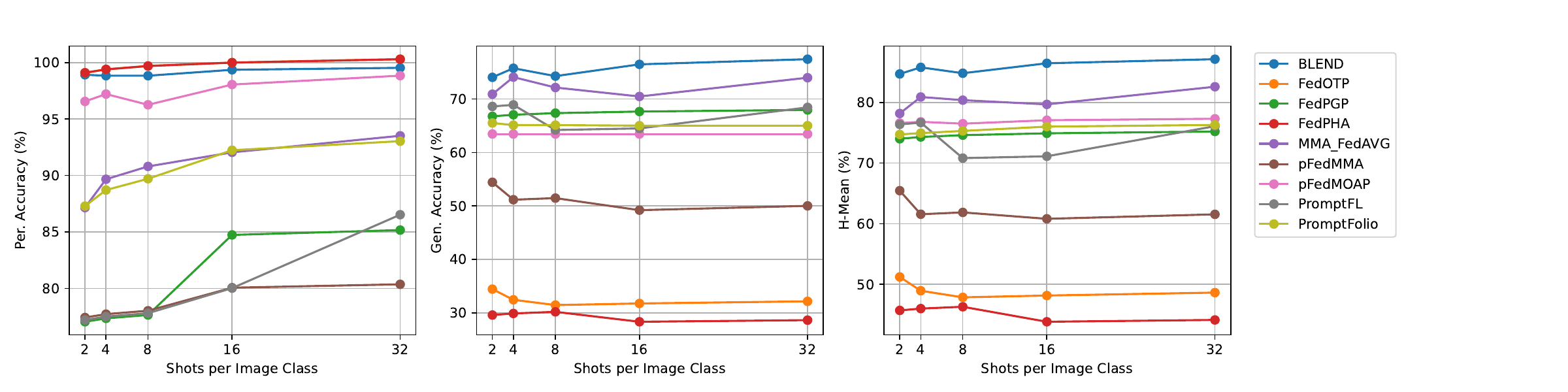}
    \caption{Personalization and generalization test accuracies alongside their harmonic mean (H-Mean) under changing number of shots per class for the \textbf{\textit{EuroSAT}} dataset.}
    \label{fig:shots-EuroSAT}
\end{figure}
\begin{figure}[h]
    \centering
    \includegraphics[width=.99\linewidth]{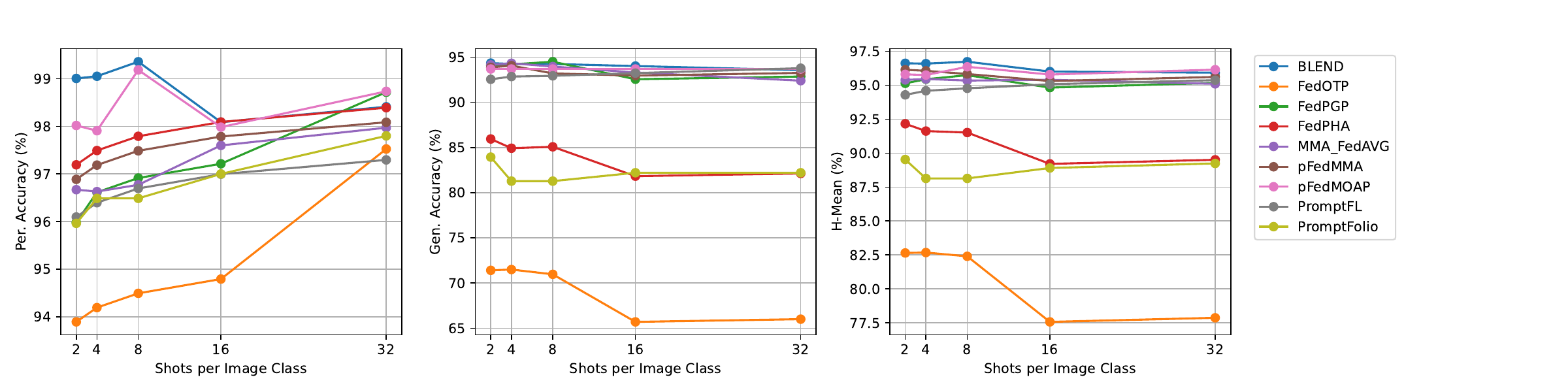}
    \caption{Personalization and generalization test accuracies alongside their harmonic mean (H-Mean) under changing number of shots per class for the \textbf{\textit{Caltech}} dataset.}
    \label{fig:shots-Caltech101}
\end{figure}
\begin{figure}[h]
    \centering
    \includegraphics[width=.99\linewidth]{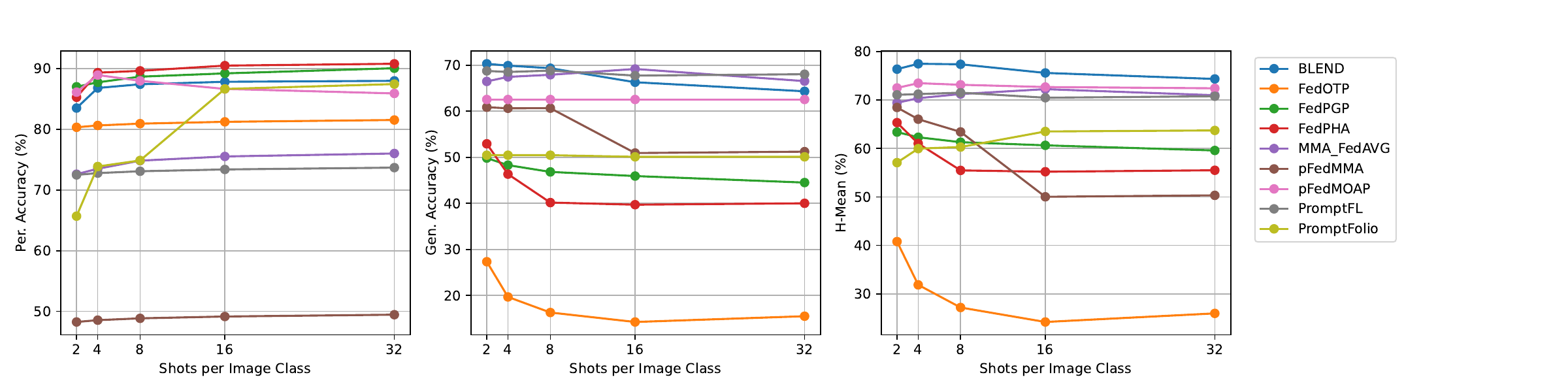}
    \caption{Personalization and generalization test accuracies alongside their harmonic mean (H-Mean) under changing number of shots per class for the \textbf{\textit{ImageNet}} dataset.}
    \label{fig:shots-ImageNet}
\end{figure}

\clearpage
\section{Different Bottleneck Sizes for Global Adapters}\label{ap:bottle_global}
We study the impact of global adapter bottleneck rank $r = r_x = r_w$ across all datasets by sweeping the rank $r\in\{\emptyset, 4, 8, 16, 32, 64, 128, 256, \text{Full-rank}\}$, where $\emptyset$ corresponds to the case where no adapter is used. We have summarized the result for each dataset in Tables \ref{tab:gabs-DTD}-\ref{tab:gabs-Caltech101}, while the average performance across all the datasets is shown in  Table~\ref{tab:gabs-avgAll}. 

 Based on the complexity of the dataset under test, the best personalization and generalization performance can be achieved by employing different global adapter bottleneck sizes. Taking Flowers (Table~\ref{tab:gabs-OxFlowers}) as an example, the best generalization performance occurs when no global adapter is employed while the best generalization performance is realized under full-rank bottleneck size for DTD dataset (Table~\ref{tab:gabs-DTD}). A similar phenomenon can be observed when focusing on the personalization performance. However, based on the average behavior observed in Table~\ref{tab:gabs-avgAll}, we can infer that the performance that maintains a balance between personalization and generalization is realized under $r=16$ which requires low computation and communication budget.
\begin{table*}[h]
\centering
\begin{minipage}{0.32\linewidth}
\centering
\scriptsize
\captionof{table}{\textbf{Avg across datasets.}}
\begin{tabular}{c|ccc}
\hline
\rowcolor{gray!20} Rank ($r$) & Per. & Gen. & H-Mean \\
\hline
No Adapter & 90.12 & 73.81 & \cellcolor{lightpink}{80.66} \\
4 & 92.25 & 74.69 & \cellcolor{lightpink}{82.05} \\
8 & 92.45 & \textbf{74.70} & \cellcolor{lightpink}{82.10} \\
16 & 92.91 & 74.62 & \cellcolor{lightpink}{\textbf{82.15}} \\
32 & \textbf{92.93} & 74.39 & \cellcolor{lightpink}{81.97} \\
64 & 92.91 & 74.09 & \cellcolor{lightpink}{81.74} \\
128 & 92.80 & 73.70 & \cellcolor{lightpink}{81.43} \\
256 & 92.68 & 73.47 & \cellcolor{lightpink}{81.20} \\
Full-rank & 92.48 & 73.63 & \cellcolor{lightpink}{81.26} \\
\hline
\end{tabular}
\label{tab:gabs-avgAll}
\end{minipage}
\hfill
\begin{minipage}{0.32\linewidth}
\centering
\scriptsize
\captionof{table}{DTD}
\begin{tabular}{c|ccc}
\hline
\rowcolor{gray!20} Rank ($r$) & Per. & Gen. & H-Mean \\
\hline
No Adapter & 85.93 & 59.66 & \cellcolor{lightpink}{70.42} \\
4 & 91.32 & 64.21 & \cellcolor{lightpink}{75.40} \\
8 & 91.46 & 64.26 & \cellcolor{lightpink}{75.49} \\
16 & 91.64 & 64.32 & \cellcolor{lightpink}{75.59} \\
32 & \textbf{91.92} & 65.12 & \cellcolor{lightpink}{76.23} \\
64 & 91.89 & 64.69 & \cellcolor{lightpink}{75.93} \\
128 & 91.51 & 64.98 & \cellcolor{lightpink}{76.00} \\
256 & 91.59 & 65.33 & \cellcolor{lightpink}{76.26} \\
Full-rank & 91.49 & \textbf{65.78} & \cellcolor{lightpink}{\textbf{76.53}} \\
\hline
\end{tabular}
\label{tab:gabs-DTD}
\end{minipage}
\hfill
\begin{minipage}{0.32\linewidth}
\centering
\scriptsize
\captionof{table}{EuroSAT}
\begin{tabular}{c|ccc}
\hline
\rowcolor{gray!20} Rank ($r$) & Per. & Gen. & H-Mean \\
\hline
No Adapter & 98.31 & 66.50 & \cellcolor{lightpink}{79.33} \\
4 & \textbf{99.54} & 72.16 & \cellcolor{lightpink}{83.66} \\
8 & 99.39 & 74.73 & \cellcolor{lightpink}{85.31} \\
16 & 99.36 & \textbf{76.50} & \cellcolor{lightpink}{\textbf{86.44}} \\
32 & 99.48 & 76.16 & \cellcolor{lightpink}{86.27} \\
64 & 99.34 & 75.42 & \cellcolor{lightpink}{85.74} \\
128 & 99.20 & 73.74 & \cellcolor{lightpink}{84.60} \\
256 & 99.02 & 72.14 & \cellcolor{lightpink}{83.47} \\
Full-rank & 98.59 & 71.04 & \cellcolor{lightpink}{82.57} \\
\hline
\end{tabular}
\label{tab:gabs-EuroSAT}
\end{minipage}
\end{table*}

\vspace{-2mm}
\begin{table*}[h]
\centering
\begin{minipage}{0.32\linewidth}
\centering
\scriptsize
\captionof{table}{FGVC Aircraft}
\begin{tabular}{c|ccc}
\hline
\rowcolor{gray!20} Rank ($r$) & Per. & Gen. & H-Mean \\
\hline
No Adapter & 58.17 & 32.92 & \cellcolor{lightpink}{42.05} \\
4 & 64.39 & \textbf{34.54} & \cellcolor{lightpink}{44.96} \\
8 & 66.51 & 33.71 & \cellcolor{lightpink}{44.74} \\
16 & 70.05 & 33.24 & \cellcolor{lightpink}{\textbf{45.09}} \\
32 & 71.04 & 32.65 & \cellcolor{lightpink}{44.74} \\
64 & 71.45 & 32.04 & \cellcolor{lightpink}{44.24} \\
128 & \textbf{72.47} & 31.83 & \cellcolor{lightpink}{44.23} \\
256 & 72.16 & 30.86 & \cellcolor{lightpink}{43.24} \\
Full-rank & 71.83 & 31.74 & \cellcolor{lightpink}{44.03} \\
\hline
\end{tabular}
\label{tab:gabs-fgvcair}
\end{minipage}
\hfill
\begin{minipage}{0.32\linewidth}
\centering
\scriptsize
\captionof{table}{Food101}
\begin{tabular}{c|ccc}
\hline
\rowcolor{gray!20} Rank ($r$) & Per. & Gen. & H-Mean \\
\hline
No Adapter & 95.24 & 90.51 & \cellcolor{lightpink}{92.82} \\
4 & \textbf{95.57} & \textbf{91.39} & \cellcolor{lightpink}{\textbf{93.43}} \\
8 & 95.44 & 91.34 & \cellcolor{lightpink}{93.34} \\
16 & 95.24 & 91.25 & \cellcolor{lightpink}{93.20} \\
32 & 95.10 & 91.02 & \cellcolor{lightpink}{93.02} \\
64 & 94.70 & 90.64 & \cellcolor{lightpink}{92.62} \\
128 & 94.29 & 90.33 & \cellcolor{lightpink}{92.27} \\
256 & 94.23 & 90.35 & \cellcolor{lightpink}{92.25} \\
Full-rank & 93.82 & 90.15 & \cellcolor{lightpink}{91.95} \\
\hline
\end{tabular}
\label{tab:gabs-Food101}
\end{minipage}
\hfill
\begin{minipage}{0.32\linewidth}
\centering
\scriptsize
\captionof{table}{Flowers}
\begin{tabular}{c|ccc}
\hline
\rowcolor{gray!20} Rank ($r$) & Per. & Gen. & H-Mean \\
\hline
No Adapter & 94.23 & \textbf{75.80} & \cellcolor{lightpink}{84.01} \\
4 & 98.99 & 74.92 & \cellcolor{lightpink}{85.29} \\
8 & 98.72 & 75.32 & \cellcolor{lightpink}{85.45} \\
16 & 98.94 & 75.21 & \cellcolor{lightpink}{85.46} \\
32 & 99.11 & 74.99 & \cellcolor{lightpink}{85.38} \\
64 & \textbf{99.33} & 75.09 & \cellcolor{lightpink}{\textbf{85.52}} \\
128 & 99.10 & 74.91 & \cellcolor{lightpink}{85.32} \\
256 & 99.11 & 74.07 & \cellcolor{lightpink}{84.78} \\
Full-rank & 99.08 & 74.33 & \cellcolor{lightpink}{84.94} \\
\hline
\end{tabular}
\label{tab:gabs-OxFlowers}
\end{minipage}
\end{table*}

\vspace{-2mm}
\begin{table*}[h]
\centering
\begin{minipage}{0.32\linewidth}
\centering
\scriptsize
\captionof{table}{OxfordPets}
\begin{tabular}{c|ccc}
\hline
\rowcolor{gray!20} Rank ($r$) & Per. & Gen. & H-Mean \\
\hline
No Adapter & 98.19 & 95.86 & \cellcolor{lightpink}{97.01} \\
4 & 99.14 & 96.79 & \cellcolor{lightpink}{97.95} \\
8 & 99.14 & 96.81 & \cellcolor{lightpink}{97.96} \\
16 & \textbf{99.21} & 96.99 & \cellcolor{lightpink}{98.09} \\
32 & 99.13 & 97.21 & \cellcolor{lightpink}{98.16} \\
64 & 98.98 & 97.21 & \cellcolor{lightpink}{98.09} \\
128 & 98.99 & 97.35 & \cellcolor{lightpink}{\textbf{98.16}} \\
256 & 98.84 & 97.35 &\cellcolor{lightpink}{ 98.09} \\
Full-rank & 98.85 & \textbf{97.47} & \cellcolor{lightpink}{98.15} \\
\hline
\end{tabular}
\label{tab:gabs-OxPets}
\end{minipage}
\hfill
\begin{minipage}{0.32\linewidth}
\centering
\scriptsize
\captionof{table}{Stanford Cars}
\begin{tabular}{c|ccc}
\hline
\rowcolor{gray!20} Rank ($r$) & Per. & Gen. & H-Mean \\
\hline
No Adapter & 90.88 & \textbf{73.39} & \cellcolor{lightpink}{81.21} \\
4 & 93.04 & 72.66 & \cellcolor{lightpink}{\textbf{81.60}} \\
8 & 93.48 & 71.50 & \cellcolor{lightpink}{81.03} \\
16 & \textbf{93.72} & 70.20 & \cellcolor{lightpink}{80.27} \\
32 & 93.58 & 68.14 & \cellcolor{lightpink}{78.86} \\
64 & 93.63 & 67.46 & \cellcolor{lightpink}{78.42} \\
128 & 93.20 & 66.56 & \cellcolor{lightpink}{77.66} \\
256 & 93.06 & 67.36 & \cellcolor{lightpink}{78.15} \\
Full-rank & 92.95 & 68.54 & \cellcolor{lightpink}{78.90} \\
\hline
\end{tabular}
\label{tab:gabs-StanCars}
\end{minipage}
\hfill
\begin{minipage}{0.32\linewidth}
\centering
\scriptsize
\captionof{table}{UCF101}
\begin{tabular}{c|ccc}
\hline
\rowcolor{gray!20} Rank ($r$) & Per. & Gen. & H-Mean \\
\hline
No Adapter & 92.19 & 75.64 & \cellcolor{lightpink}{83.10} \\
4 & 94.63 & 78.61 & \cellcolor{lightpink}{85.88} \\
8 & 94.59 & \textbf{78.68} & \cellcolor{lightpink}{\textbf{85.90}} \\
16 & 95.03 & 78.20 & \cellcolor{lightpink}{85.80} \\
32 & 94.69 & 77.70 &\cellcolor{lightpink}{ 85.36} \\
64 & \textbf{95.21} & 77.59 & \cellcolor{lightpink}{85.50} \\
128 & 94.83 & 75.92 & \cellcolor{lightpink}{84.33} \\
256 & 94.93 & 76.16 & \cellcolor{lightpink}{84.52} \\
Full-rank & 94.43 & 75.76 & \cellcolor{lightpink}{84.07} \\
\hline
\end{tabular}
\label{tab:gabs-UCF101}
\end{minipage}
\end{table*}

\vspace{-2mm}
\begin{table*}[!h]
\centering
\begin{minipage}{0.32\linewidth}
\centering
\scriptsize
\captionof{table}{ImageNet}
\begin{tabular}{c|ccc}
\hline
\rowcolor{gray!20} Rank ($r$) & Per. & Gen. & H-Mean \\
\hline
4 & 87.81 & 67.28 & \cellcolor{lightpink}{\textbf{76.19}} \\
8 & 87.73 & 66.51 & \cellcolor{lightpink}{75.66} \\
16 & \textbf{87.83} & 66.31 & \cellcolor{lightpink}{75.56} \\
32 & 87.11 & 66.96 & \cellcolor{lightpink}{75.72} \\
64 & 86.53 & 66.79 & \cellcolor{lightpink}{75.39} \\
128 & 86.37 & 67.37 & \cellcolor{lightpink}{75.70} \\
256 & 85.85 & 67.37 & \cellcolor{lightpink}{75.50} \\
Full-rank & 85.81 & \textbf{67.47} & \cellcolor{lightpink}{75.54} \\
\hline
\end{tabular}
\label{tab:gabs-ImageNet}
\end{minipage}
\begin{minipage}{0.32\linewidth}
\centering
\scriptsize
\captionof{table}{Caltech}
\begin{tabular}{c|ccc}
\hline
\rowcolor{gray!20} Rank ($r$) & Per. & Gen. & H-Mean \\
\hline
No Adapter & 97.97 & 94.04 & \cellcolor{lightpink}{95.97} \\
4 & \textbf{98.12} & \textbf{94.32} & \cellcolor{lightpink}{\textbf{96.18}} \\
8 & 98.08 & 94.18 & \cellcolor{lightpink}{96.09} \\
16 & 98.09 & 94.01 & \cellcolor{lightpink}{96.01} \\
32 & 98.11 & 93.92 & \cellcolor{lightpink}{95.97} \\
64 & 98.07 & 93.98 & \cellcolor{lightpink}{95.98} \\
128 & 98.05 & 94.04 & \cellcolor{lightpink}{96.00} \\
256 & 97.96 & 93.72 & \cellcolor{lightpink}{95.79} \\
Full-rank & 97.96 & 94.02 & \cellcolor{lightpink}{95.95} \\
\hline
\end{tabular}
\label{tab:gabs-Caltech101}
\end{minipage}
\end{table*}
 
\clearpage
\section{Different Bottleneck Sizes for Personalized Projection Adapter}\label{ap:bottle_per}
To study the effect of the personalized projection adapter's bottleneck size, we have conducted a separate set of experiments, sweeping the bottleneck rank $r_p$ over the range $r_p\in\{4, 8, 16, 32, 64, 128, 256, \text{Full-rank}\}$.
The results for individual datasets are presented in Tables~\ref{tab:pabs-DTD}-\ref{tab:pabs-Caltech101}, while the average performance across all datasets is reported in  Table~\ref{tab:pabs-avgAll}.

Given the utility of the personalized projection adapter in our method (discussed in Sec.~\ref{sec:method}), it is expected that a more sophisticated structure for this adapter will result in better personalization performance. This hypothesis is verified as in 7 out of 10 cases in our experiments (Tables~\ref{tab:pabs-DTD}, \ref{tab:pabs-EuroSAT}, \ref{tab:pabs-Food101}, \ref{tab:pabs-StanCars}, \ref{tab:pabs-UCF101}, \ref{tab:pabs-ImageNet}, \ref{tab:pabs-Caltech101}) the best personalization performance occurs under full-rank settings. On the other hand, the generalization performance often peaks under the two shallowest ranks $r_p\in\{4, 8\}$ (Tables~\ref{tab:pabs-DTD}, \ref{tab:pabs-EuroSAT}, \ref{tab:pabs-OxFlowers}, \ref{tab:pabs-OxPets}, \ref{tab:pabs-StanCars}), showing how dedication of relatively larger number of parameters to personalization in the trained global model can negatively affect the generalization performance. 

Focusing on Table~\ref{tab:pabs-avgAll} that contains the average personalization and generalization performances alongside their H-Mean for all of the tested $r_p$ values across all datasets, we observe that the best personalization performance is achieved  under the $\text{Full-rank}$ setting and the best balance between generalization and personalization (i.e., the H-Mean) is realized under one of the larger bottleneck ranks ($r_p = 256$).
We can thus conclude that sophisticated personalized projection adapters, even though  harmful to the generalization performance in some cases, have proven to be significantly beneficial in improving personalization and consequently the  balance between personalization and generalization.
\vspace{-4mm}
\begin{table*}[h]
\centering
\begin{minipage}{0.32\linewidth}
\centering
\scriptsize
\captionof{table}{\textbf{Avg across datasets.}}
\begin{tabular}{c|ccc}
\hline
\rowcolor{gray!20} Rank ($r_p$) & Per. & Gen. & H-Mean \\
\hline
4 & 89.99 & \textbf{75.00} & \cellcolor{lightpink}{81.44} \\
8 & 90.45 & 74.83 & \cellcolor{lightpink}{81.50} \\
16 & 91.00 & 74.86 & \cellcolor{lightpink}{81.73} \\
32 & 91.59 & 74.91 & \cellcolor{lightpink}{81.96} \\
64 & 92.12 & 74.64 & \cellcolor{lightpink}{81.92} \\
128 & 92.57 & 74.71 & \cellcolor{lightpink}{82.12} \\
256 & 92.91 & 74.62 & \cellcolor{lightpink}{\textbf{82.15}} \\
Full-rank & \textbf{92.95} & 74.46 & \cellcolor{lightpink}{82.08} \\
\hline
\end{tabular}
\label{tab:pabs-avgAll}
\end{minipage}
\hfill
\begin{minipage}{0.32\linewidth}
\centering
\scriptsize
\captionof{table}{DTD}
\begin{tabular}{c|ccc}
\hline
\rowcolor{gray!20} Rank ($r_p$) & Per. & Gen. & H-Mean \\
\hline
4 & 88.15 & \textbf{65.00} & \cellcolor{lightpink}{74.83} \\
8 & 89.21 & 64.56 & \cellcolor{lightpink}{74.91} \\
16 & 89.79 & 64.71 & \cellcolor{lightpink}{75.22} \\
32 & 90.13 & 64.77 & \cellcolor{lightpink}{75.38} \\
64 & 90.81 & 64.80 & \cellcolor{lightpink}{75.63} \\
128 & 91.40 & 64.67 & \cellcolor{lightpink}{75.75} \\
256 & 91.64 & 64.32 & \cellcolor{lightpink}{75.59} \\
Full-rank & \textbf{92.93} & 63.93 & \cellcolor{lightpink}{\textbf{75.75}} \\
\hline
\end{tabular}
\label{tab:pabs-DTD}
\end{minipage}
\hfill
\begin{minipage}{0.32\linewidth}
\centering
\scriptsize
\captionof{table}{EuroSAT}
\begin{tabular}{c|ccc}
\hline
\rowcolor{gray!20} Rank ($r_p$) & Per. & Gen. & H-Mean \\
\hline
4 & 97.32 & \textbf{78.69} & \cellcolor{lightpink}{\textbf{87.02}} \\
8 & 98.08 & 77.86 & \cellcolor{lightpink}{86.81} \\
16 & 98.40 & 77.39 & \cellcolor{lightpink}{86.64} \\
32 & 98.86 & 77.55 & \cellcolor{lightpink}{86.92} \\
64 & 98.96 & 75.57 & \cellcolor{lightpink}{85.70} \\
128 & 99.22 & 76.38 & \cellcolor{lightpink}{86.31} \\
256 & 99.36 & 76.50 & \cellcolor{lightpink}{86.44} \\
Full-rank & \textbf{99.53} & 76.19 &\cellcolor{lightpink}{86.31} \\
\hline
\end{tabular}
\label{tab:pabs-EuroSAT}
\end{minipage}
\end{table*}

\vspace{-4mm}
\begin{table*}[h]
\centering
\begin{minipage}{0.32\linewidth}
\centering
\scriptsize
\captionof{table}{FGVC Aircraft}
\begin{tabular}{c|ccc}
\hline
\rowcolor{gray!20} Rank ($r_p$) & Per. & Gen. & H-Mean \\
\hline
4 & 57.93 & 33.28 & \cellcolor{lightpink}{42.27} \\
8 & 58.47 & 32.67 & \cellcolor{lightpink}{41.92} \\
16 & 60.25 & 33.46 & \cellcolor{lightpink}{43.03} \\
32 & 62.78 & \textbf{33.78} & \cellcolor{lightpink}{43.92} \\
64 & 66.02 & 32.85 & \cellcolor{lightpink}{43.87} \\
128 & 67.60 & 33.08 & \cellcolor{lightpink}{44.42} \\
256 & \textbf{70.05} & 33.24 & \cellcolor{lightpink}{\textbf{45.09}} \\
Full-rank & 68.19 & 32.80 & \cellcolor{lightpink}{44.29} \\
\hline
\end{tabular}
\label{tab:pabs-FGVCAir}
\end{minipage}
\hfill
\begin{minipage}{0.32\linewidth}
\centering
\scriptsize
\captionof{table}{Food101}
\begin{tabular}{c|ccc}
\hline
\rowcolor{gray!20} Rank ($r_p$) & Per. & Gen. & H-Mean \\
\hline
4 & 94.75 & 91.26 & \cellcolor{lightpink}{\cellcolor{lightpink}{92.97}} \\
8 & 94.79 & 91.22 & \cellcolor{lightpink}{92.97} \\
16 & 94.87 & 91.28 & \cellcolor{lightpink}{93.04} \\
32 & 94.97 & 91.21 & \cellcolor{lightpink}{93.06} \\
64 & 95.15 & 91.13 & \cellcolor{lightpink}{93.10} \\
128 & 95.22 & \textbf{91.31} & \cellcolor{lightpink}{93.22} \\
256 & 95.24 & 91.25 & \cellcolor{lightpink}{93.20} \\
Full-rank & \textbf{95.43} & 91.16 & \cellcolor{lightpink}{\textbf{93.24}} \\
\hline
\end{tabular}
\label{tab:pabs-Food101}
\end{minipage}
\hfill
\begin{minipage}{0.32\linewidth}
\centering
\scriptsize
\captionof{table}{Flowers}
\begin{tabular}{c|ccc}
\hline
\rowcolor{gray!20} Rank ($r_p$) & Per. & Gen. & H-Mean \\
\hline
4 & 98.71 & 75.17 & \cellcolor{lightpink}{85.35} \\
8 & 98.83 & \textbf{75.69} & \cellcolor{lightpink}{\textbf{85.73}} \\
16 & 98.65 & 75.52 & \cellcolor{lightpink}{85.55} \\
32 & 99.01 & 75.54 & \cellcolor{lightpink}{85.70} \\
64 & 99.00 & 75.34 & \cellcolor{lightpink}{85.56} \\
128 & \textbf{99.17} & 75.31 & \cellcolor{lightpink}{85.61} \\
256 & 98.94 & 75.21 & \cellcolor{lightpink}{85.46} \\
Full-rank & 99.05 & 75.48 & \cellcolor{lightpink}{85.68} \\
\hline
\end{tabular}
\label{tab:pabs-OxFlowers}
\end{minipage}
\end{table*}

\vspace{-4mm}
\begin{table*}[h]
\centering
\begin{minipage}{0.32\linewidth}
\centering
\scriptsize
\captionof{table}{OxfordPets}
\begin{tabular}{c|ccc}
\hline
\rowcolor{gray!20} Rank ($r_p$) & Per. & Gen. & H-Mean \\
\hline
4 & 98.44 & \textbf{97.54} &  \cellcolor{lightpink}{97.99} \\
8 & 98.87 & 97.44 & \cellcolor{lightpink}{\textbf{98.15}} \\
16 & 99.02 & 97.23 & \cellcolor{lightpink}{98.12} \\
32 & 99.08 & 97.21 & \cellcolor{lightpink}{98.14} \\
64 & 99.12 & 96.93 & \cellcolor{lightpink}{98.02} \\
128 & 99.19 & 96.99 & \cellcolor{lightpink}{98.08} \\
256 & \textbf{99.21} & 96.99 & \cellcolor{lightpink}{98.09} \\
Full-rank & 99.20 & 96.99 & \cellcolor{lightpink}{98.08} \\
\hline
\end{tabular}
\label{tab:pabs-OxPets}
\end{minipage}
\hfill
\begin{minipage}{0.32\linewidth}
\centering
\scriptsize
\captionof{table}{Stanford Cars}
\begin{tabular}{c|ccc}
\hline
\rowcolor{gray!20} Rank ($r_p$) & Per. & Gen. & H-Mean \\
\hline
4 & 89.14 & 70.52 & \cellcolor{lightpink}{78.74} \\
8 & 90.29 & \textbf{70.87} & \cellcolor{lightpink}{79.41} \\
16 & 91.13 & 70.70 & \cellcolor{lightpink}{79.62} \\
32 & 92.21 & 70.75 & \cellcolor{lightpink}{80.07} \\
64 & 92.85 & 70.45 & \cellcolor{lightpink}{80.11} \\
128 & 93.47 & 70.55 & \cellcolor{lightpink}{\textbf{80.41}} \\
256 & 93.72 & 70.20 & \cellcolor{lightpink}{80.27} \\
Full-rank & \textbf{93.80} & 69.53 & \cellcolor{lightpink}{79.86} \\
\hline
\end{tabular}
\label{tab:pabs-StanCars}
\end{minipage}
\hfill
\begin{minipage}{0.32\linewidth}
\centering
\scriptsize
\captionof{table}{UCF101}
\begin{tabular}{c|ccc}
\hline
\rowcolor{gray!20} Rank ($r_p$) & Per. & Gen. & H-Mean \\
\hline
4 & 93.27 & 77.55 & \cellcolor{lightpink}{84.69} \\
8 & 92.94 & 77.77 & \cellcolor{lightpink}{84.68} \\
16 & 93.82 & 77.62 & \cellcolor{lightpink}{84.96} \\
32 & 94.22 & 78.17 & \cellcolor{lightpink}{85.45} \\
64 & 94.17 & \textbf{78.29} & \cellcolor{lightpink}{85.50} \\
128 & 94.87 & 78.18 & \cellcolor{lightpink}{85.72} \\
256 & 95.03 & 78.20 & \cellcolor{lightpink}{\textbf{85.80}} \\
Full-rank & \textbf{95.22} & 77.87 & \cellcolor{lightpink}{85.68} \\
\hline
\end{tabular}
\label{tab:pabs-UCF101}
\end{minipage}
\hfill
\end{table*}

\vspace{-4mm}
\begin{table*}[!h]
\centering
\begin{minipage}{0.32\linewidth}
\centering
\scriptsize
\captionof{table}{ImageNet}
\begin{tabular}{c|ccc}
\hline
\rowcolor{gray!20} Rank ($r_p$) & Per. & Gen. & H-Mean \\
\hline
4 & 84.23 & 66.76 & \cellcolor{lightpink}{74.48} \\
8 & 85.04 & 66.15 & \cellcolor{lightpink}{74.41} \\
16 & 86.11 & 66.51 & \cellcolor{lightpink}{75.05} \\
32 & 86.61 & 66.28 & \cellcolor{lightpink}{75.09} \\
64 & 87.02 & \textbf{66.78} & \cellcolor{lightpink}{75.57} \\
128 & 87.42 & 66.38 & \cellcolor{lightpink}{75.46} \\
256 & 87.83 & 66.31 & \cellcolor{lightpink}{75.56} \\
Full-rank & \textbf{88.00} & 66.49 & \cellcolor{lightpink}{\textbf{75.75}} \\
\hline
\end{tabular}
\label{tab:pabs-ImageNet}
\end{minipage}
\begin{minipage}{0.32\linewidth}
\centering
\scriptsize
\captionof{table}{Caltech}
\begin{tabular}{c|ccc}
\hline
\rowcolor{gray!20} Rank ($r_p$) & Per. & Gen. & H-Mean \\
\hline
4 & 97.97 & 94.20 & \cellcolor{lightpink}{96.05} \\
8 & 98.00 & 94.10 & \cellcolor{lightpink}{96.01} \\
16 & 97.92 & 94.19 & \cellcolor{lightpink}{96.02} \\
32 & 97.98 & 93.85 & \cellcolor{lightpink}{95.87} \\
64 & 98.06 & 94.23 & \cellcolor{lightpink}{96.11} \\
128 & 98.15 & \textbf{94.28} & \cellcolor{lightpink}{\textbf{96.17}} \\
256 & 98.09 & 94.01 & \cellcolor{lightpink}{96.01} \\
Full-rank & \textbf{98.16} & 94.19 & \cellcolor{lightpink}{96.13} \\
\hline
\end{tabular}
\label{tab:pabs-Caltech101}
\end{minipage}
\end{table*}

\clearpage
\section{Adapter Injection at Different Layers}\label{ap:adapter_injection}
We further explore the effect of injecting adapters throughout the backbone. To this end, we  conduct a series of experiments through which we sweep different depths and placements of the adapters between the transformer blocks of the backbone. Our experiments include the following $\mathsf{start-end}$ (i.e., $\ell^{\mathsf{S}}-\ell^{\mathsf{E}}$) combinations for transformer blocks in which we have used adapters: $\{1-1, 1-2, 1-3, \cdots, 1-11, 1-12, 2-12, 3-12, \cdots, 11-12, 12-12\}$. Our results for these experiments regarding the personalization, generalization, and H-Mean of these two are included in Figs.~\ref{fig:injection-UCF101}-\ref{fig:injection-Caltech101}, and the average performance across all the datasets is reported in Fig.~\ref{fig:injection-avgAll}.

The H-Mean curves (i.e., the right sub-plots) in Figs.~\ref{fig:injection-UCF101}–\ref{fig:injection-Caltech101} indicate that the optimal placement of the adapters depends on dataset-specific characteristics. In some datasets, the most important features are captured by the early transformer layers, in which case inserting adapters in those layers yields the best performance (e.g., Fig.~\ref{fig:injection-StanCars}-right). In contrast, some datasets may benefit more from adapters in the later transformer layers, suggesting that their key features emerge deeper in the network (e.g., Fig.~\ref{fig:injection-OxFlowers}-right). Nevertheless, when performance is averaged across all datasets (Fig.~\ref{fig:injection-avgAll}-right), the peak H-Mean is achieved when adapters are distributed across a broader portion of the backbone (layers 1 to 12), highlighting the advantage of multi-layer adaptation. Further, from Fig.~\ref{fig:injection-avgAll}-right, it can be seen that a comparable H-Mean performance can be achieved when inserting the adapters from layer 5 to 12 (i.e., under a lower computation/communication overhead).

\begin{figure}[h]
    \centering
    \includegraphics[width=.98\linewidth]{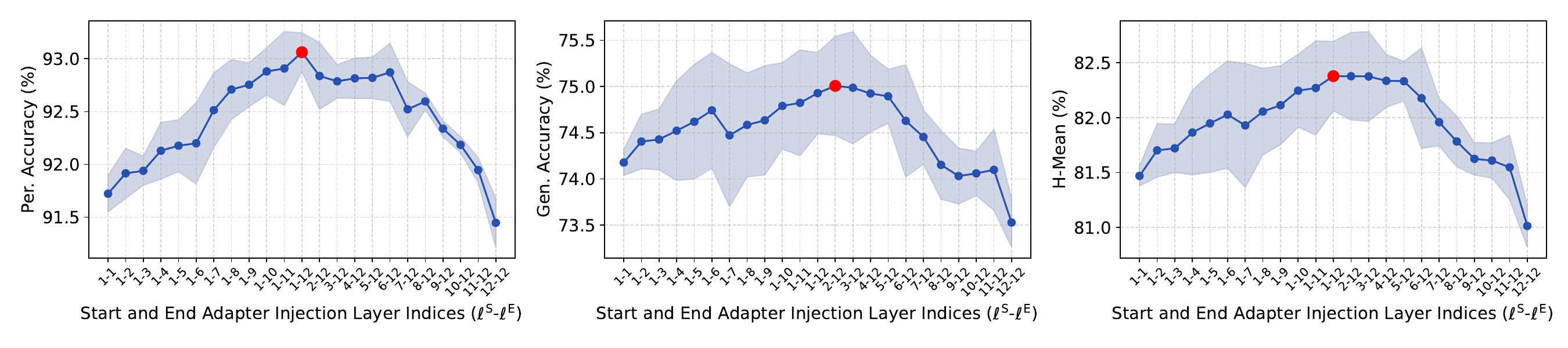}
    \caption{Personalization and generalization test accuracies alongside their harmonic mean (H-Mean) under changing adapter injection settings \textbf{averaged across all datasets}.}
    \label{fig:injection-avgAll}
\end{figure}

\begin{figure}[h]
    \centering
    \includegraphics[width=.98\linewidth]{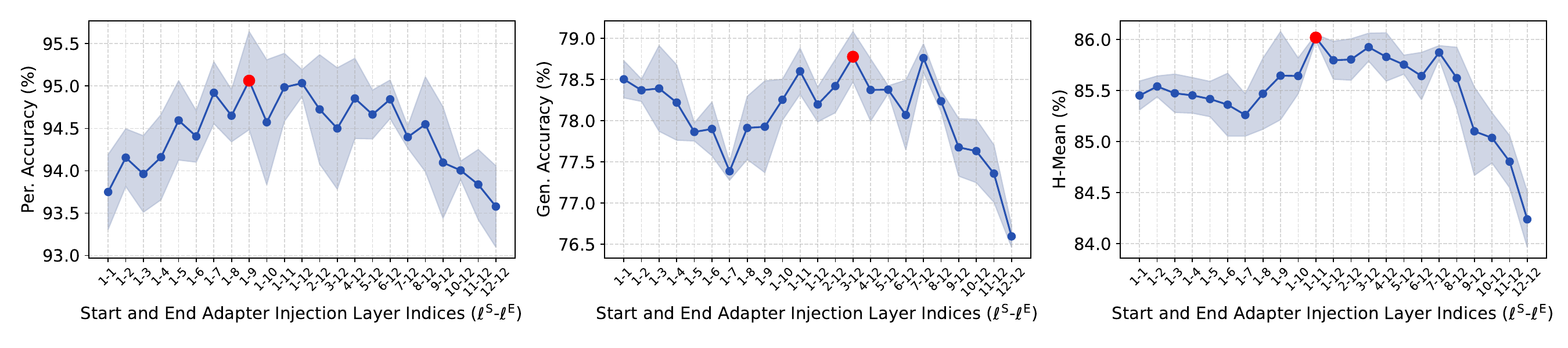}
    \caption{Personalization and generalization test accuracies alongside their harmonic mean (H-Mean) under changing adapter injection settings for the \textbf{\textit{UCF101}} dataset.}
    \label{fig:injection-UCF101}
\end{figure}
\begin{figure}[h]
    \centering
    \includegraphics[width=.98\linewidth]{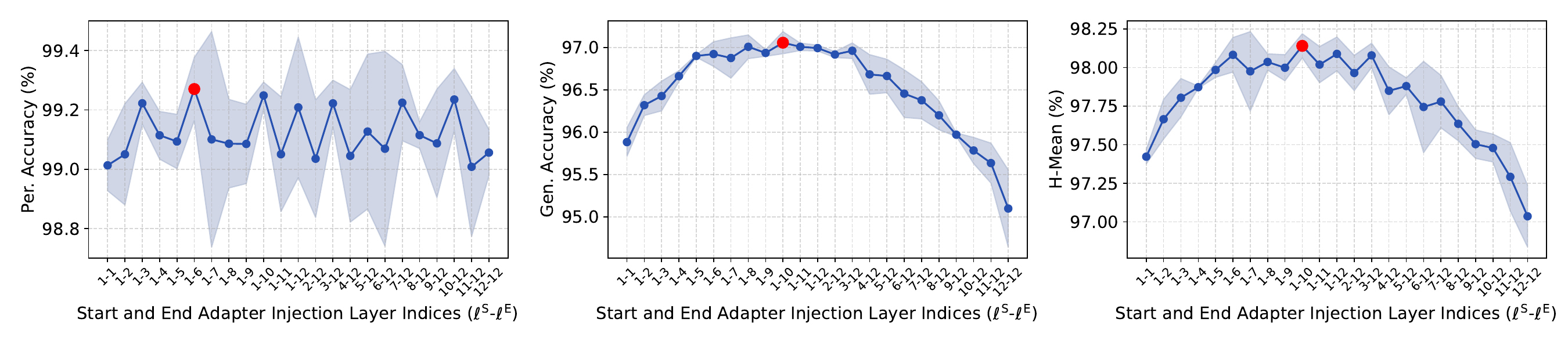}
    \caption{Personalization and generalization test accuracies alongside their harmonic mean (H-Mean) under changing adapter injection settings for the \textbf{\textit{OxfordPets}} dataset.}
    \label{fig:injection-OxPets}
\end{figure}
\begin{figure}[h]
    \centering
    \includegraphics[width=.98\linewidth]{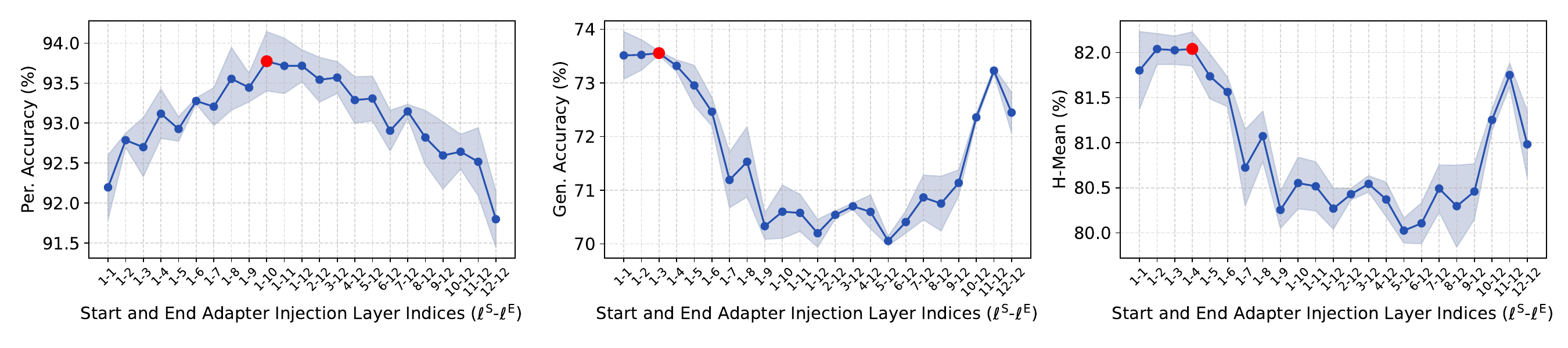}
    \caption{Personalization and generalization test accuracies alongside their harmonic mean (H-Mean) under changing adapter injection settings for the \textbf{\textit{Stanford-Cars}} dataset.}
    \label{fig:injection-StanCars}
\end{figure}
\begin{figure}[h]
    \centering
    \includegraphics[width=.98\linewidth]{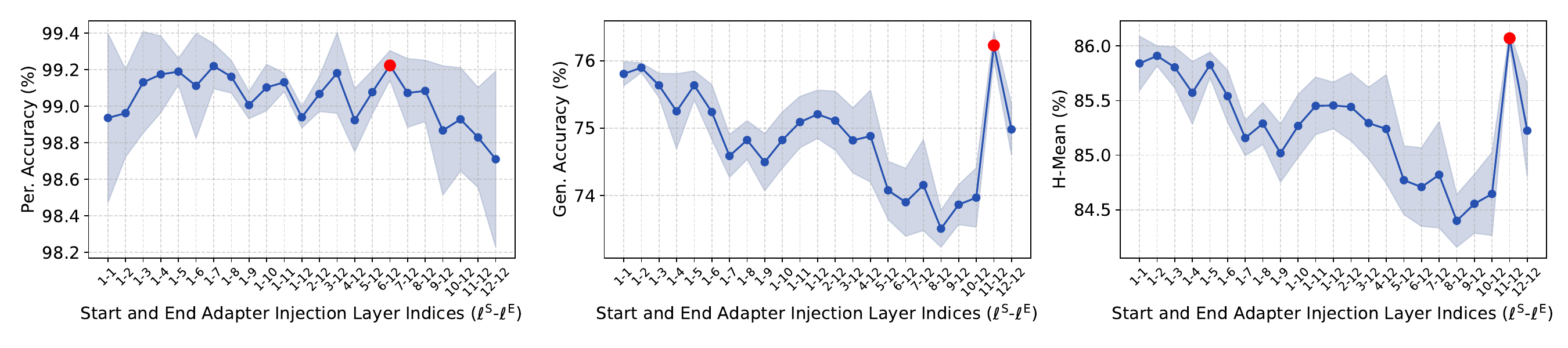}
    \caption{Personalization and generalization test accuracies alongside their harmonic mean (H-Mean) under changing adapter injection settings for the \textbf{\textit{Flowers}} dataset.}
    \label{fig:injection-OxFlowers}
\end{figure}
\begin{figure}[h]
    \centering
    \includegraphics[width=.98\linewidth]{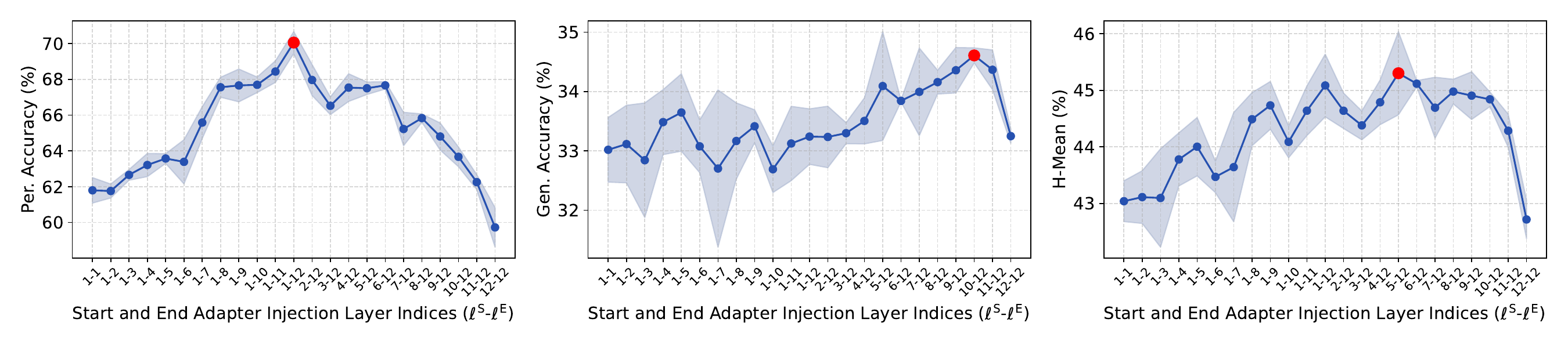}
    \caption{Personalization and generalization test accuracies alongside their harmonic mean (H-Mean) under changing adapter injection settings for the \textbf{\textit{FGVC Aircraft}} dataset.}
    \label{fig:injection-FGVCAir}
\end{figure}
\begin{figure}[h]
    \centering
    \includegraphics[width=.98\linewidth]{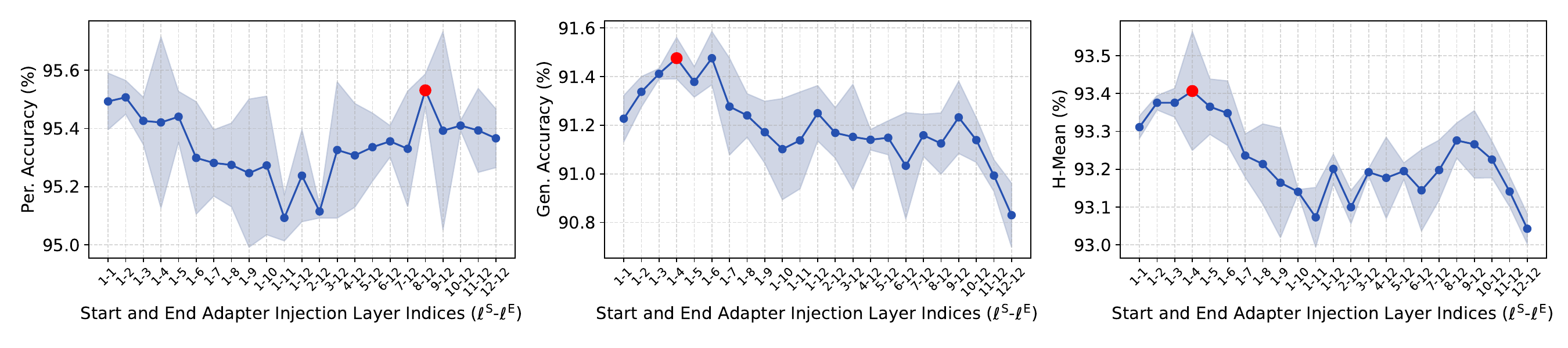}
    \caption{Personalization and generalization test accuracies alongside their harmonic mean (H-Mean) under changing adapter injection settings for the \textbf{\textit{Food101}} dataset.}
    \label{fig:injection-Food101}
\end{figure}
\begin{figure}[h]
    \centering
    \includegraphics[width=.98\linewidth]{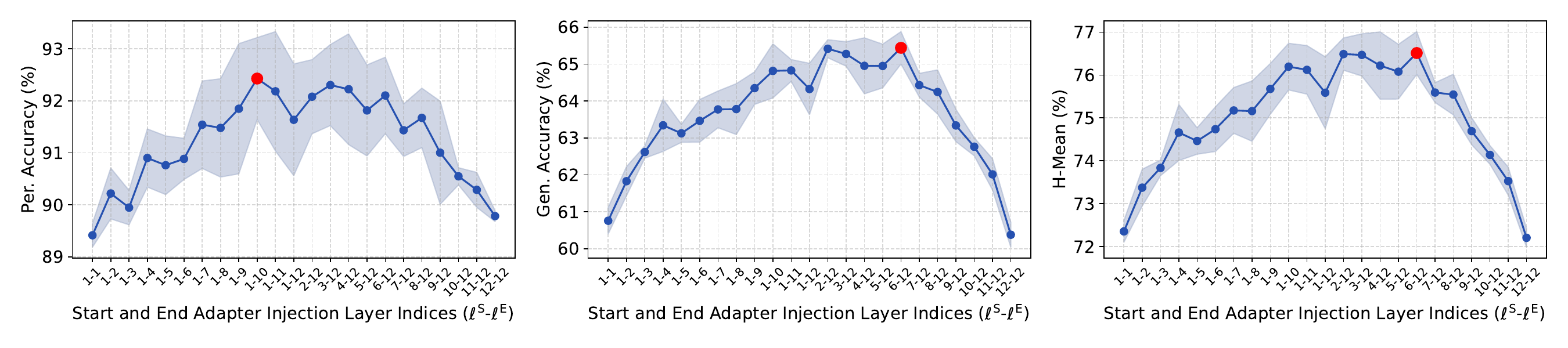}
    \caption{Personalization and generalization test accuracies alongside their harmonic mean (H-Mean) under changing adapter injection settings for the \textbf{\textit{DTD}} dataset.}
    \label{fig:injection-DTD}
\end{figure}
\begin{figure}[h]
    \centering
    \includegraphics[width=.98\linewidth]{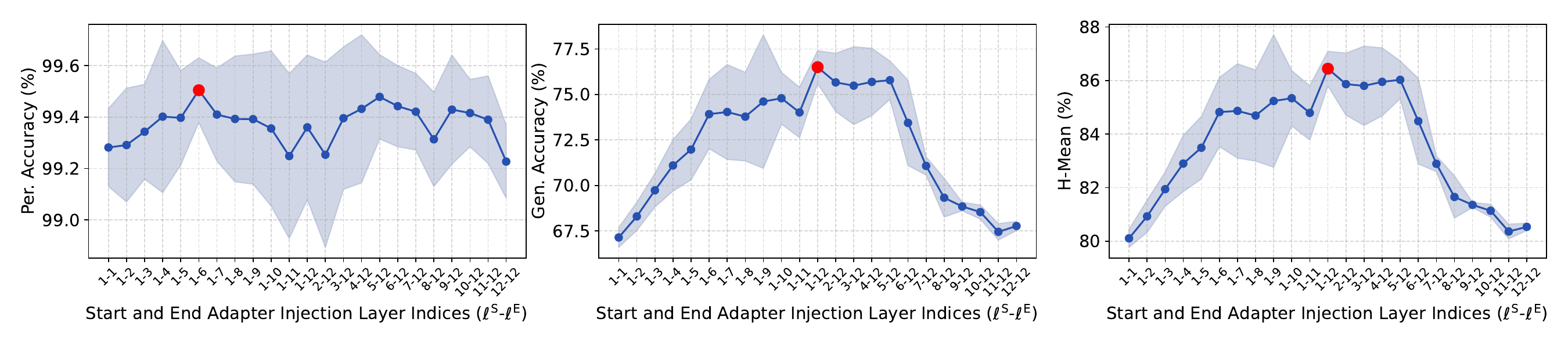}
    \caption{Personalization and generalization test accuracies alongside their harmonic mean (H-Mean) under changing adapter injection settings for the \textbf{\textit{EuroSAT}} dataset.}
    \label{fig:injection-EuroSAT}
\end{figure}
\begin{figure}[h]
    \centering
    \includegraphics[width=.98\linewidth]{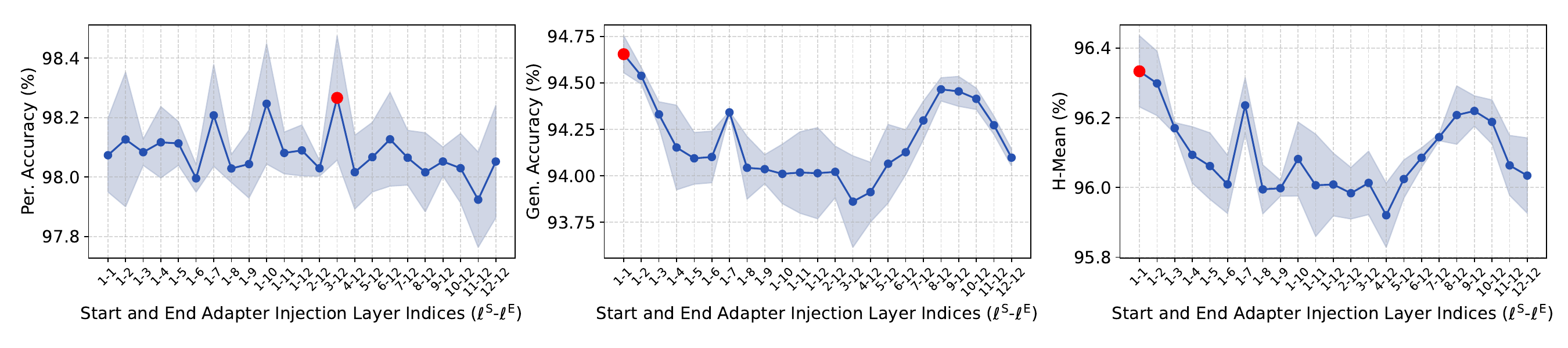}
    \caption{Personalization and generalization test accuracies alongside their harmonic mean (H-Mean) under changing adapter injection settings for the \textbf{\textit{Caltech}} dataset.}
    \label{fig:injection-Caltech101}
\end{figure}
\begin{figure}[h]
    \centering
    \includegraphics[width=.98\linewidth]{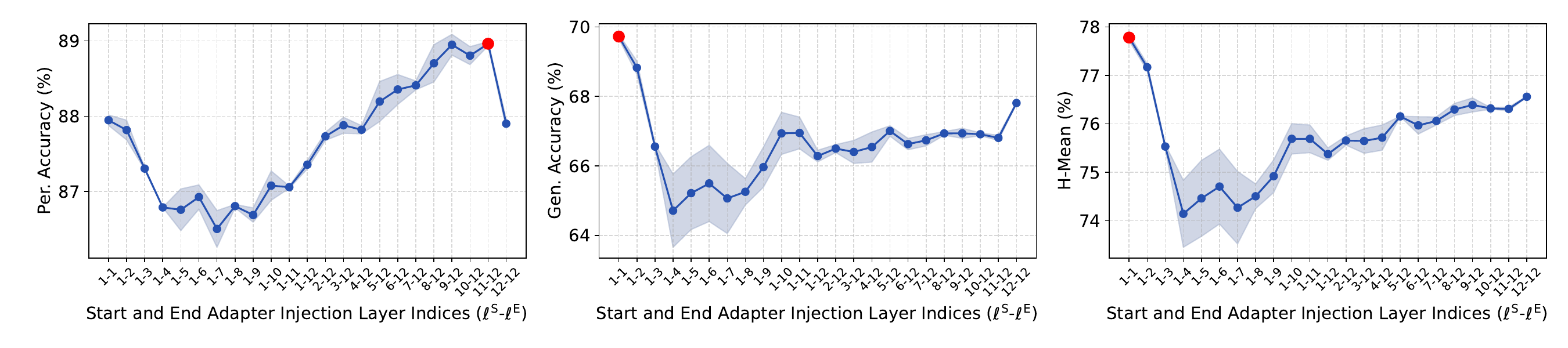}
    \caption{Personalization and generalization test accuracies alongside their harmonic mean (H-Mean) under changing adapter injection settings for the \textbf{\textit{ImageNet}} dataset.}
    \label{fig:injection-ImageNet}
\end{figure}

\clearpage
\section{Effect of $\gamma$ on Performance}\label{ap:gamma}
To evaluate the effect of $\gamma$ on the model performance, we reported the personalization, generalization, and the H-Mean of the two across the datasets
in Figs.~\ref{fig:gamma-UCF101}-\ref{fig:gamma-Imagenet}, while Fig.~\ref{fig:gamma-AvgAll} reports the average performance across the datasets.

As explained in Eq.~\eqref{eq:fusion_rep}, the parameter $\gamma$ corresponds to the emphasis of personalization vs. generalization in training. This is captured in our results as the personalization performance (left sub-plot in all result sets) improves for $\gamma>0$ for all datasets compared to the case of $\gamma=0$. Also, we observe a decline in the generalization performance for higher values of $\gamma$ (middle sub-plot in all result sets). Focusing on the left sub-plot in the result set for the \textit{UCF101} dataset (Fig.~\ref{fig:gamma-UCF101}), we can see that the gap between the best ($\gamma=0.4$) and worst ($\gamma=0$) personalization performances is almost $10\%$, where the sharp increase in the personalization performance with $\gamma$ increasing from $0$ to a slightly higher value ($\gamma=0.09$) shows the immediate effectiveness of adding the personalization term. Also, the right sub-plot of Fig.~\ref{fig:gamma-UCF101} shows the H-Mean performance where the optimal balance of personalization vs. generalization is reached at $\gamma=0.09$.

Focusing on Fig.~\ref{fig:gamma-AvgAll} that reports the average performance across the datasets, it can be observed that the best test accuracy considering both the generalization and personalization is achieved with small values of $\gamma$, indicating the equilibrium point of $\gamma=0.09$ when the H-Mean performance is measured. The same trends can also be observed with minor fluctuations in the experiments on other datasets (Figs.~\ref{fig:gamma-UCF101}-\ref{fig:gamma-Imagenet}).

\begin{figure}[h]
    \centering
    \includegraphics[width=.9\linewidth]{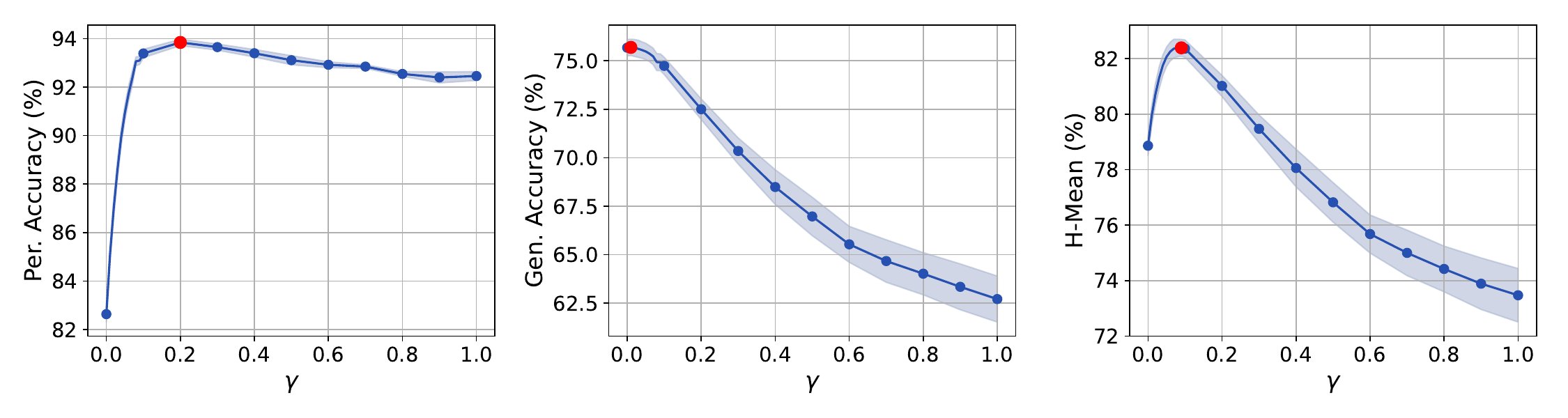}
    \vspace{-2mm}
    \caption{Personalization and generalization test accuracies alongside their harmonic mean (H-Mean) under changing values of parameter $\gamma$ \textbf{averaged across all datasets}.}
    \label{fig:gamma-AvgAll}
\end{figure}
\begin{figure}[h]
    \centering
    \includegraphics[width=.9\linewidth]{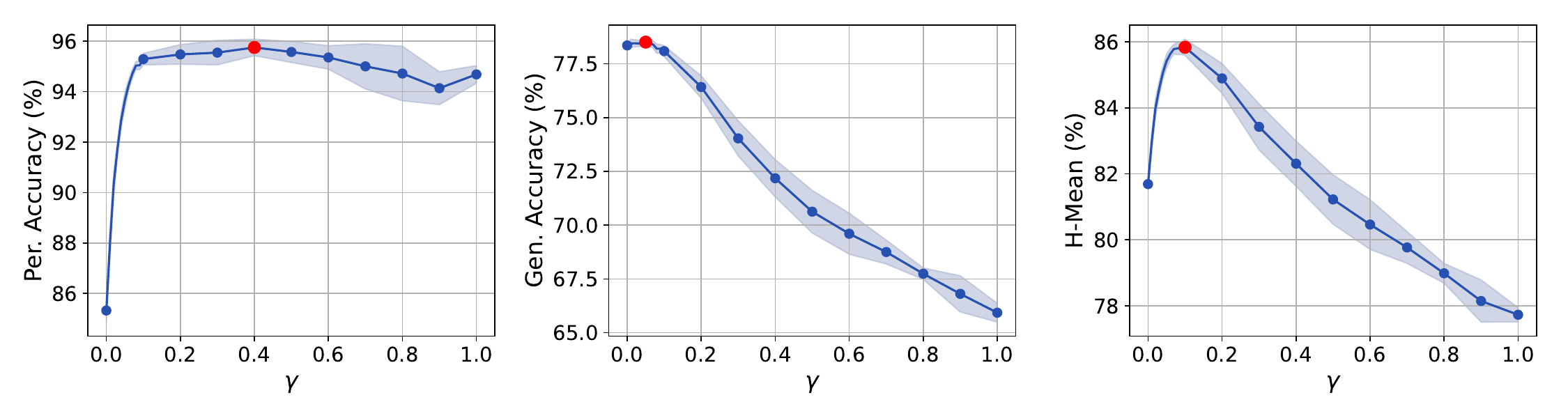}
    \vspace{-2mm}
    \caption{Personalization and generalization test accuracies alongside their harmonic mean (H-Mean) under changing values of parameter $\gamma$ for the \textbf{\textit{UCF101}} dataset.}
    \label{fig:gamma-UCF101}
\end{figure}
\begin{figure}[h]
    \centering
    \includegraphics[width=.9\linewidth]{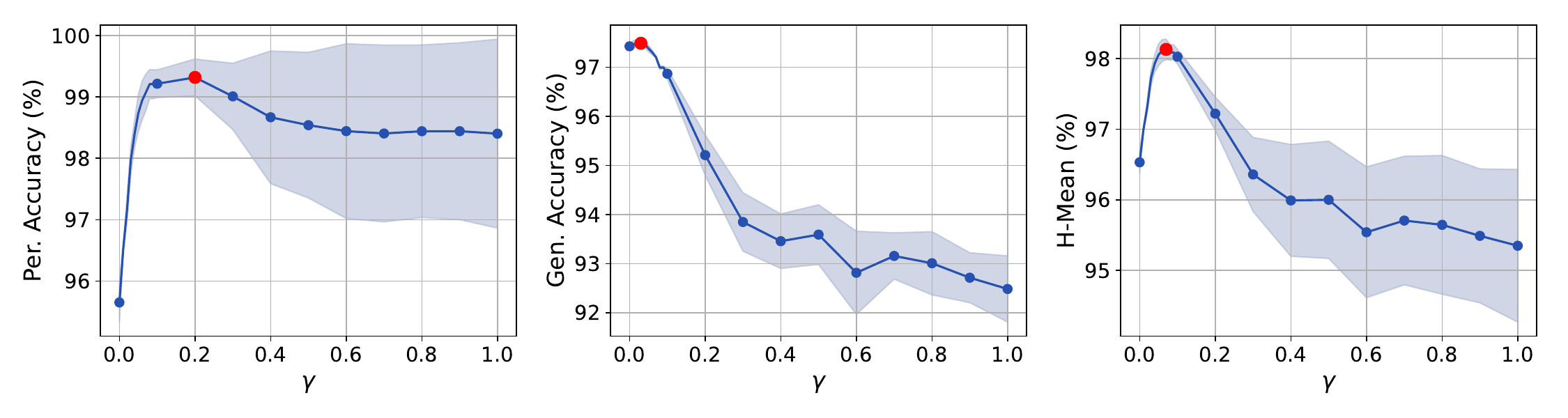}
    \caption{Personalization and generalization test accuracies alongside their harmonic mean (H-Mean) under changing values of parameter $\gamma$ for the \textbf{\textit{Oxford Pets}} dataset.}
    \label{fig:gamma-OxPets}
\end{figure}
\begin{figure}[h]
    \centering
    \includegraphics[width=.9\linewidth]{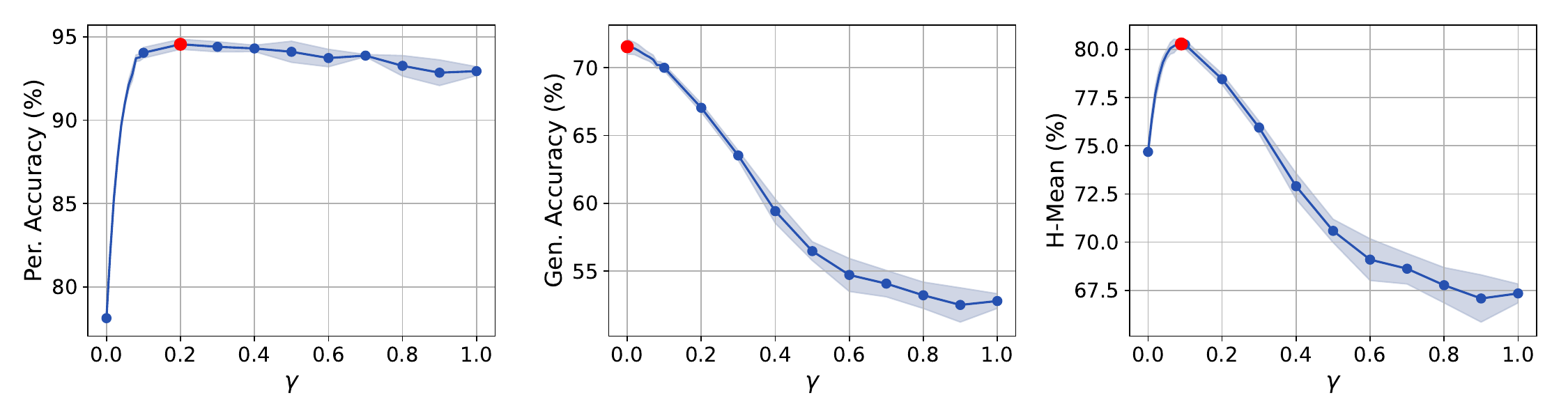}
    \caption{Personalization and generalization test accuracies alongside their harmonic mean (H-Mean) under changing values of parameter $\gamma$ for the \textbf{\textit{Stanford Cars}} dataset.}
    \label{fig:gamma-StanCars}
\end{figure}
\begin{figure}[h]
    \centering
    \includegraphics[width=.9\linewidth]{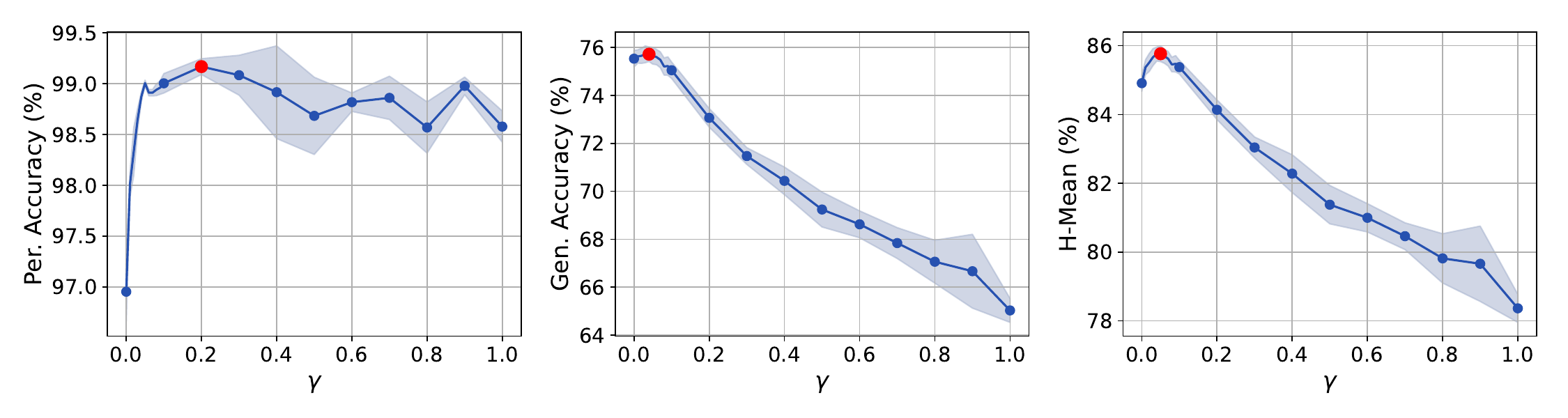}
    \caption{Personalization and generalization test accuracies alongside their harmonic mean (H-Mean) under changing values of parameter $\gamma$ for the \textbf{\textit{Flowers}} dataset.}
    \label{fig:gamma-OxFlowers}
\end{figure}
\begin{figure}[h]
    \centering
    \includegraphics[width=.9\linewidth]{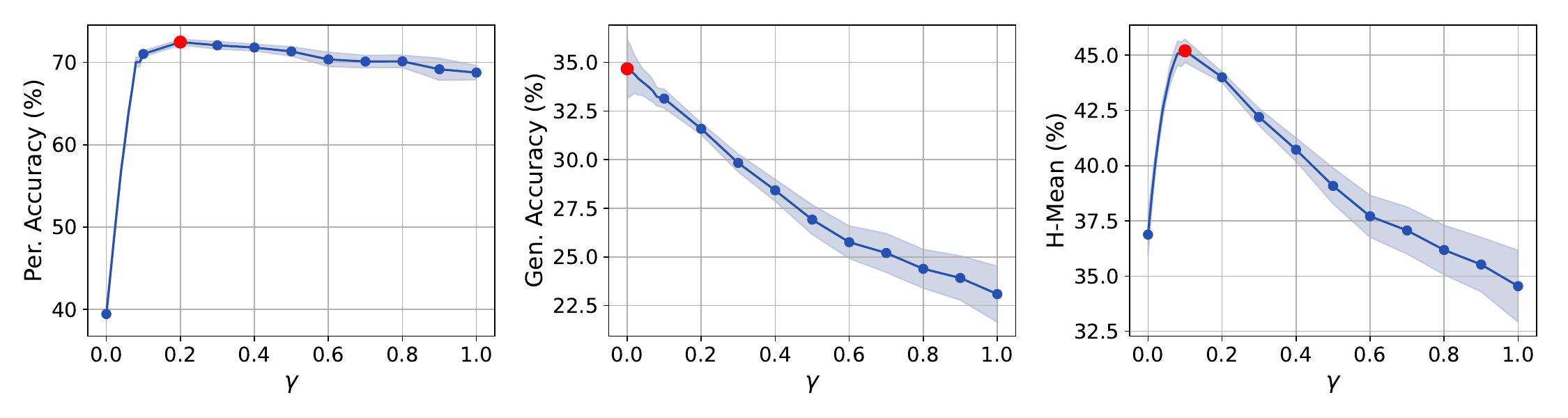}
    \caption{Personalization and generalization test accuracies alongside their harmonic mean (H-Mean) under changing values of parameter $\gamma$ for the \textbf{\textit{FGVC Aircraft}} dataset.}
    \label{fig:gamma-FGVCAir}
\end{figure}
\begin{figure}[h]
    \centering
    \includegraphics[width=.9\linewidth]{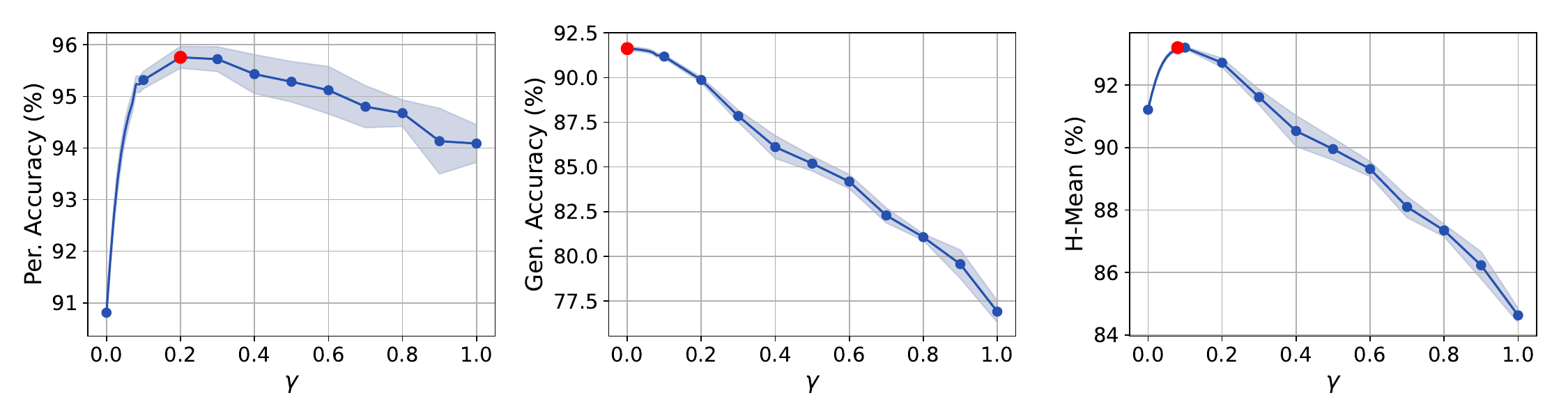}
    \caption{Personalization and generalization test accuracies alongside their harmonic mean (H-Mean) under changing values of parameter $\gamma$ for the \textbf{\textit{Food101}} dataset.}
    \label{fig:gamma-Food101}
\end{figure}
\begin{figure}[h]
    \centering
    \includegraphics[width=.9\linewidth]{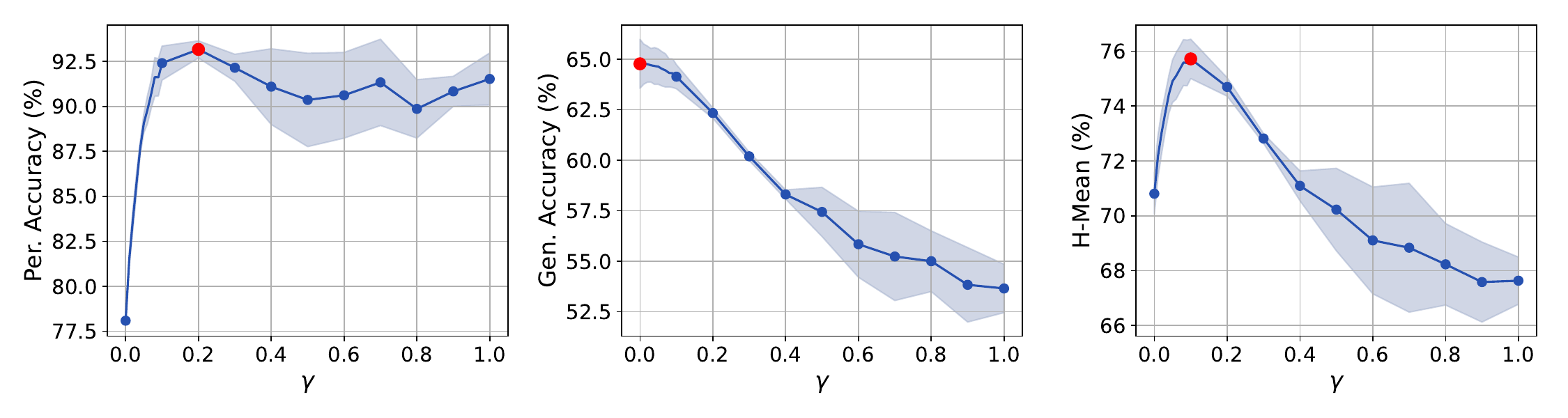}
    \caption{Personalization and generalization test accuracies alongside their harmonic mean (H-Mean) under changing values of parameter $\gamma$ for the \textbf{\textit{DTD}} dataset.}
    \label{fig:gamma-DTD}
\end{figure}
\begin{figure}[h]
    \centering
    \includegraphics[width=.9\linewidth]{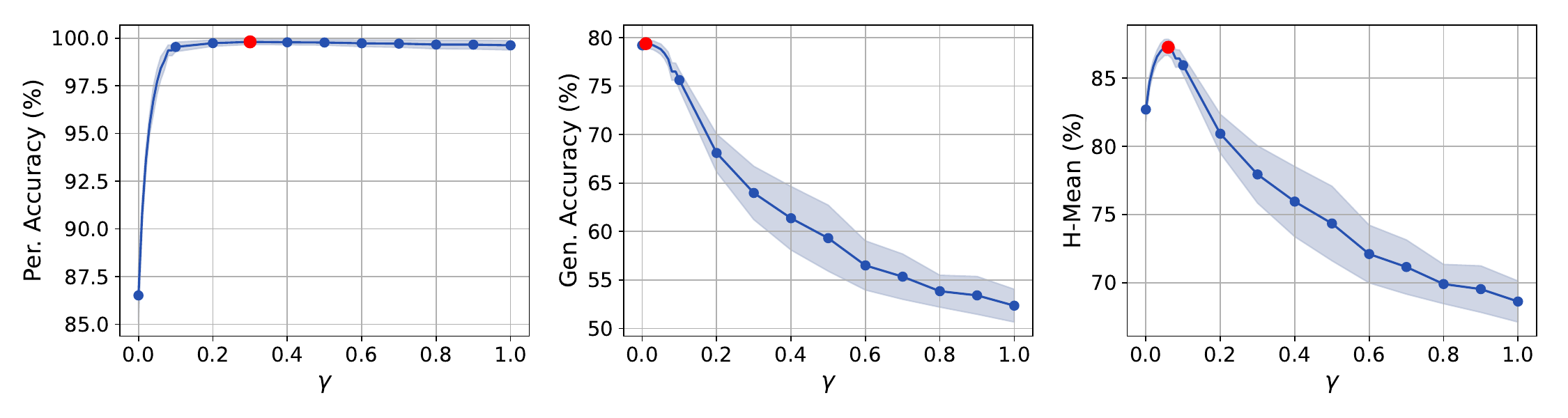}
    \caption{Personalization and generalization test accuracies alongside their harmonic mean (H-Mean) under changing values of parameter $\gamma$ for the \textbf{\textit{EuroSAT}} dataset.}
    \label{fig:gamma-EuroSat}
\end{figure}
\begin{figure}[h]
    \centering
    \includegraphics[width=.9\linewidth]{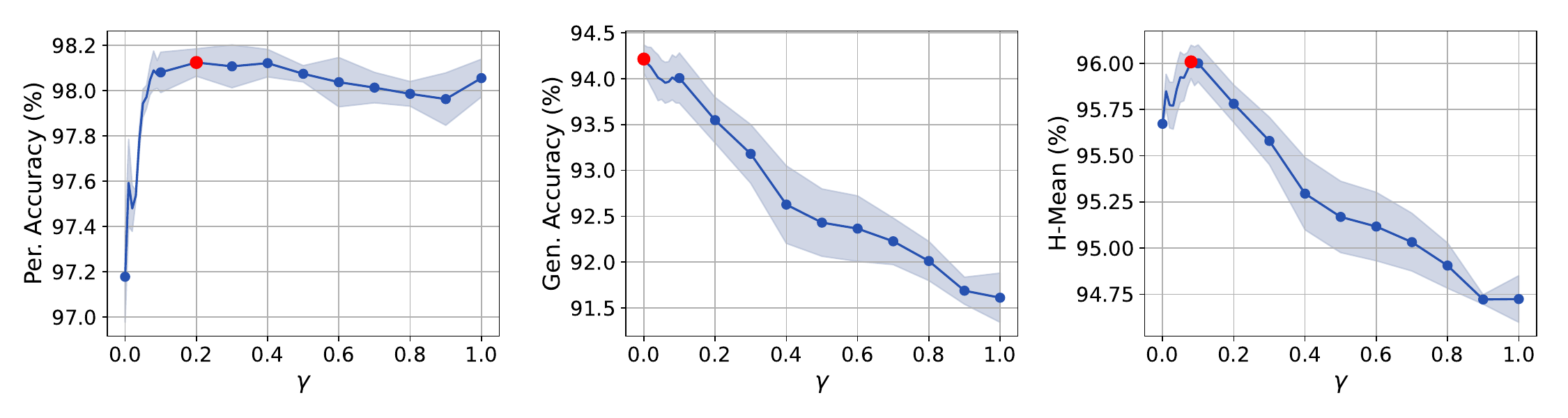}
    \caption{Personalization and generalization test accuracies alongside their harmonic mean (H-Mean) under changing values of parameter $\gamma$ for the \textbf{\textit{Caltech}} dataset.}
    \label{fig:gamma-Caltech101}
\end{figure}
\begin{figure}[!t]
    \centering
    \includegraphics[width=.9\linewidth]{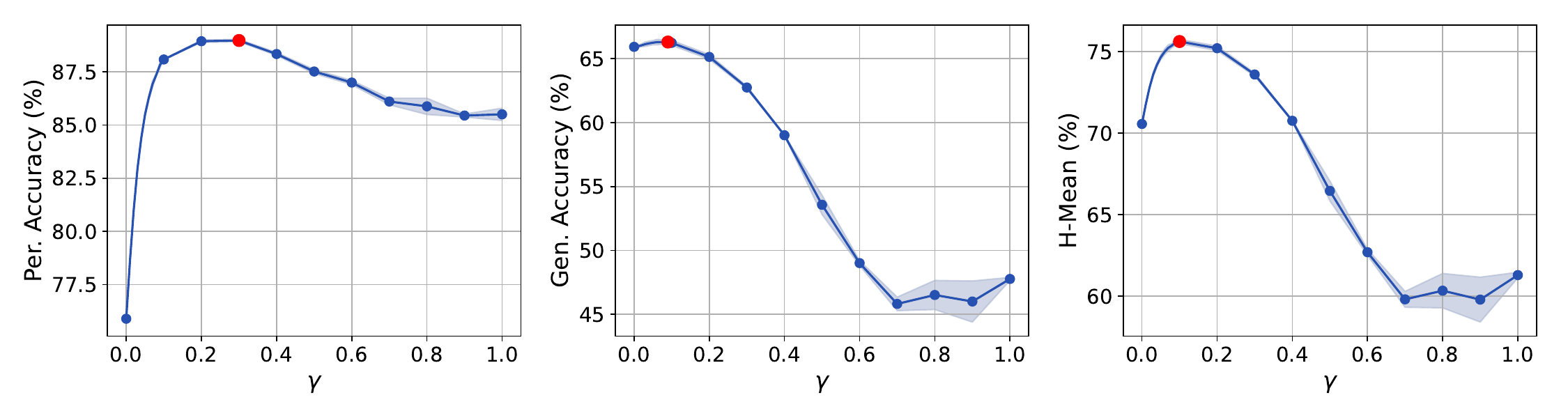}
    \caption{Personalization and generalization test accuracies alongside their harmonic mean (H-Mean) under changing values of parameter $\gamma$ for the \textbf{\textit{ImageNet}} dataset.}
    \label{fig:gamma-Imagenet}
\end{figure}

\clearpage
\section{Effect of $\lambda$ on Performance}\label{ap:lambda}
We study the effect of the parameter $\lambda$ which involves the tradeoff between the running anchor loss $\big(\mathcal{L}^{\mathsf{Anch}}_n\big)$ and the fusion loss $\big(\mathcal{L}^{\mathsf{Fus}}_n\big)$ in the loss function (refer to Eq.~\eqref{eq:final_loss}). 
In Figs.~\ref{fig:lambda-UCF101}-\ref{fig:lambda-Caltech101}, we have shown the personalization and generalization performances of the trained global models on various datasets alongside their H-Mean with the best performance for each one shown in a red marker/dot for each corresponding sub-plot.

In Eq.~\eqref{eq:final_loss}, we can see that increasing $\lambda$ will result in more emphasis on the fusion term ($\mathcal{L}^{\mathsf{Fus}}_n$). Doing so, as pointed out in Sec.~\ref{sec:local-model-ft}, will improve the personalization performance while it may sacrifice the generalization performance of the model which can be generally observed in Figs.~\ref{fig:lambda-UCF101}-~\ref{fig:lambda-Caltech101}. 
In rare cases, such as the Flowers dataset (Fig.~\ref{fig:lambda-OxFlowers}), we can observe that the generalization performance may improve as $\lambda$ increases. Such behaviors can be attributed to dataset-specific characteristics. Nevertheless, in most  scenarios (Fig.~\ref{fig:lambda-UCF101}-\ref{fig:lambda-StanCars} and \ref{fig:lambda-FGVCAir}-\ref{fig:lambda-Caltech101}), a decline in generalization and improvement in personalization can be observed upon increasing the value of $\lambda$. Collectively, the optimal choice of $\lambda$, resulting in the best H-Mean of personalization and generalization, is obtained at $\lambda=0.2$ (Fig.~\ref{fig:lambda_avg_datasets}-right).
\vspace{-2mm}
\begin{figure}[h]
    \centering
    \includegraphics[width=.9\linewidth]{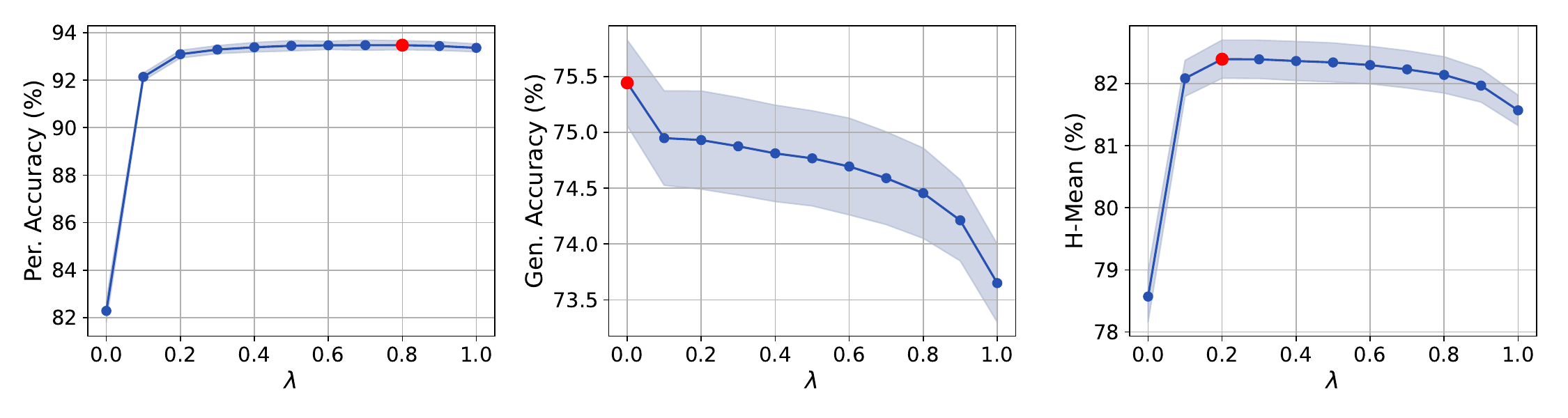}
    \vspace{-5mm}
    \caption{Personalization and generalization test accuracies alongside their harmonic mean (H-Mean) under changing values of parameter $\lambda$ \textbf{averaged across all datasets}.}
    \label{fig:lambda_avg_datasets}
\end{figure}
\vspace{-2mm}
\begin{figure}[h]
    \centering
    \includegraphics[width=.9\linewidth]{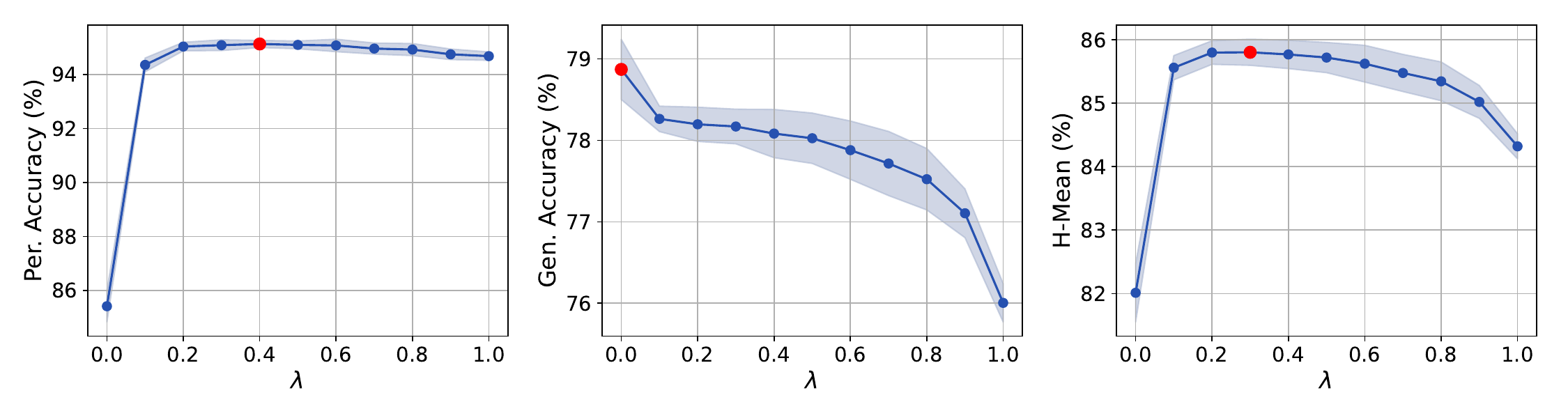}
    \vspace{-5mm}
    \caption{Personalization and generalization test accuracies alongside their harmonic mean (H-Mean) under changing values of parameter $\lambda$ for the \textbf{\textit{UCF101}} dataset.}
    \label{fig:lambda-UCF101}
\end{figure}
\vspace{-2mm}
\begin{figure}[h]
    \centering
    \includegraphics[width=.9\linewidth]{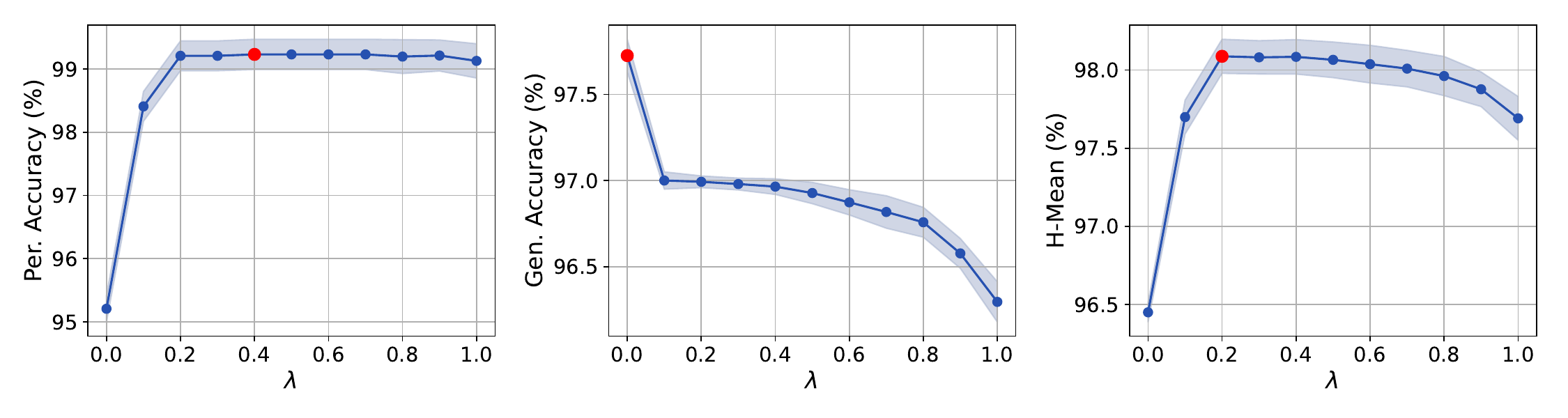}
    \vspace{-5mm}
    \caption{Personalization and generalization test accuracies alongside their harmonic mean (H-Mean) under changing values of parameter $\lambda$ for the \textbf{\textit{OxfordPets}} dataset.}
    \label{fig:lambda-OxPets}
\end{figure}
\begin{figure}[h]
    \centering
    \includegraphics[width=.9\linewidth]{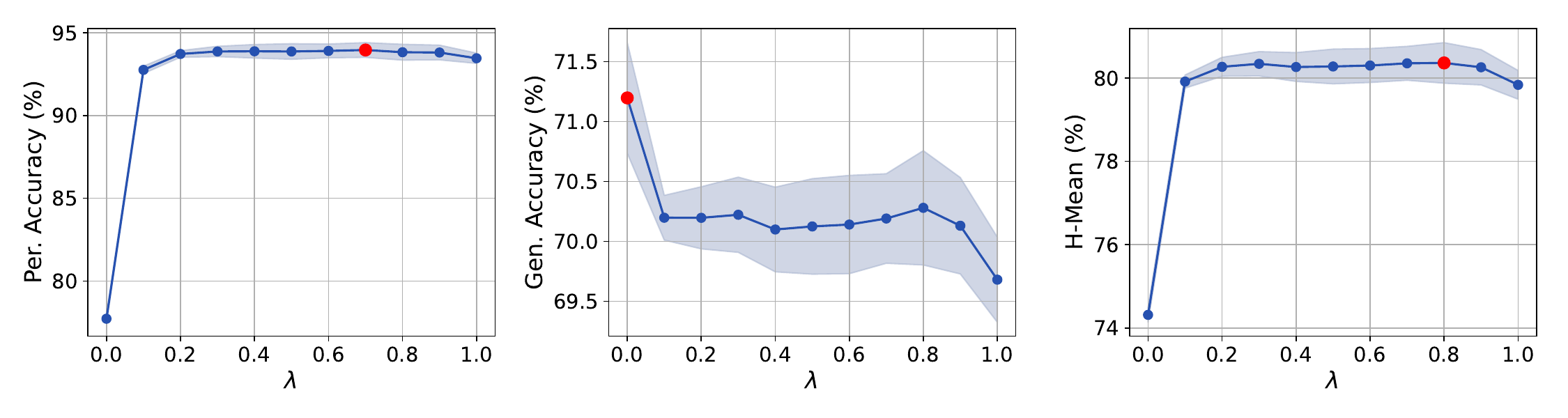}
    \caption{Personalization and generalization test accuracies alongside their harmonic mean (H-Mean) under changing values of parameter $\lambda$ for the \textbf{\textit{Stanford-Cars}} dataset.}
    \label{fig:lambda-StanCars}
\end{figure}
\begin{figure}[h]
    \centering
    \includegraphics[width=.9\linewidth]{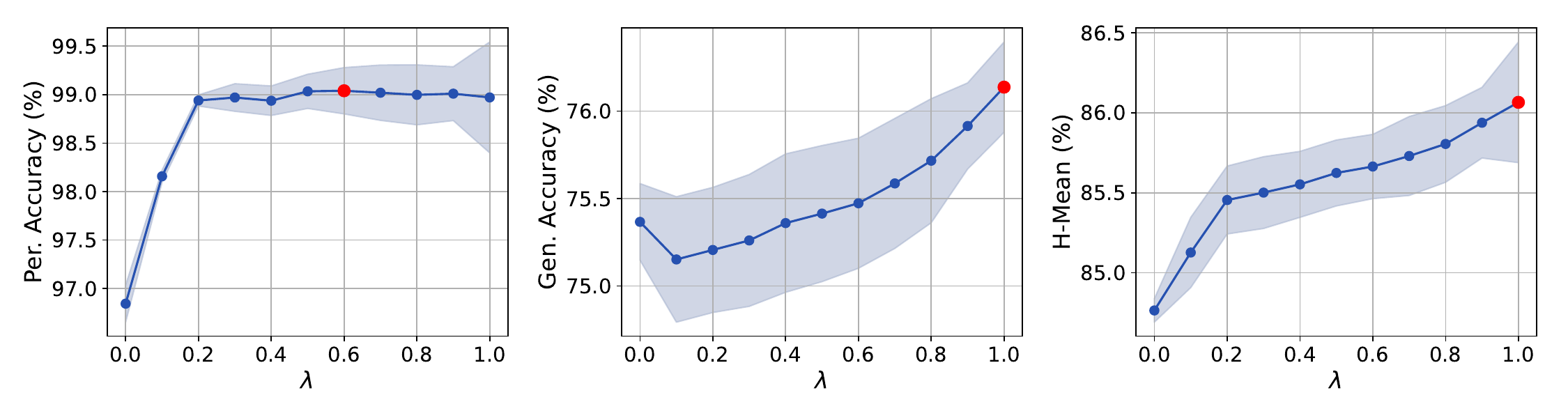}
    \caption{Personalization and generalization test accuracies alongside their harmonic mean (H-Mean) under changing values of parameter $\lambda$ for the \textbf{\textit{Flowers}} dataset.}
    \label{fig:lambda-OxFlowers}
\end{figure}
\begin{figure}[h]
    \centering
    \includegraphics[width=.9\linewidth]{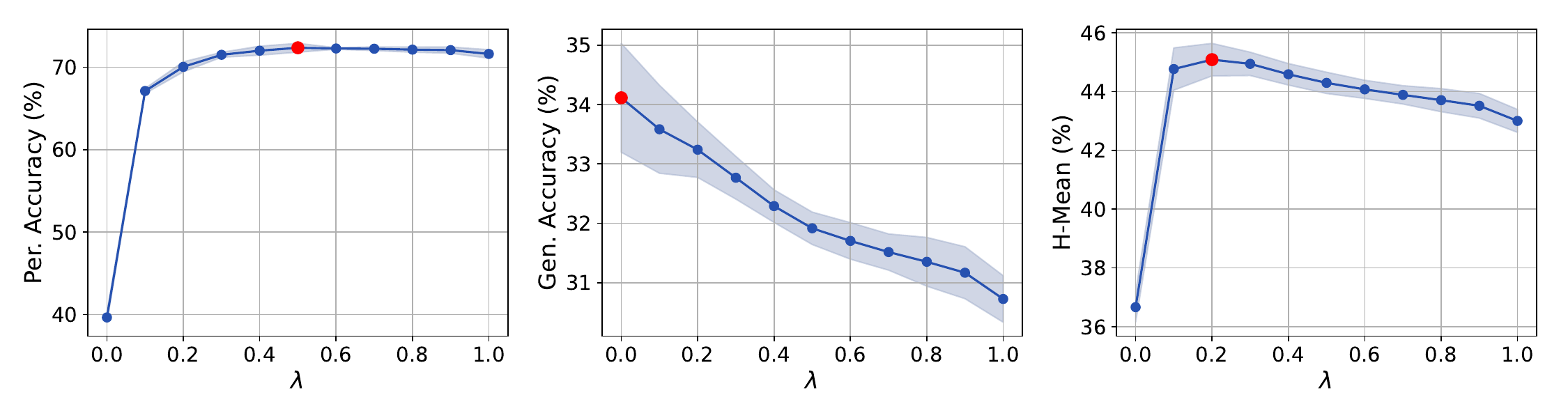}
    \caption{Personalization and generalization test accuracies alongside their harmonic mean (H-Mean) under changing values of parameter $\lambda$ for the \textbf{\textit{FGVC Aircraft}} dataset.}
    \label{fig:lambda-FGVCAir}
\end{figure}
\begin{figure}[h]
    \centering
    \includegraphics[width=.9\linewidth]{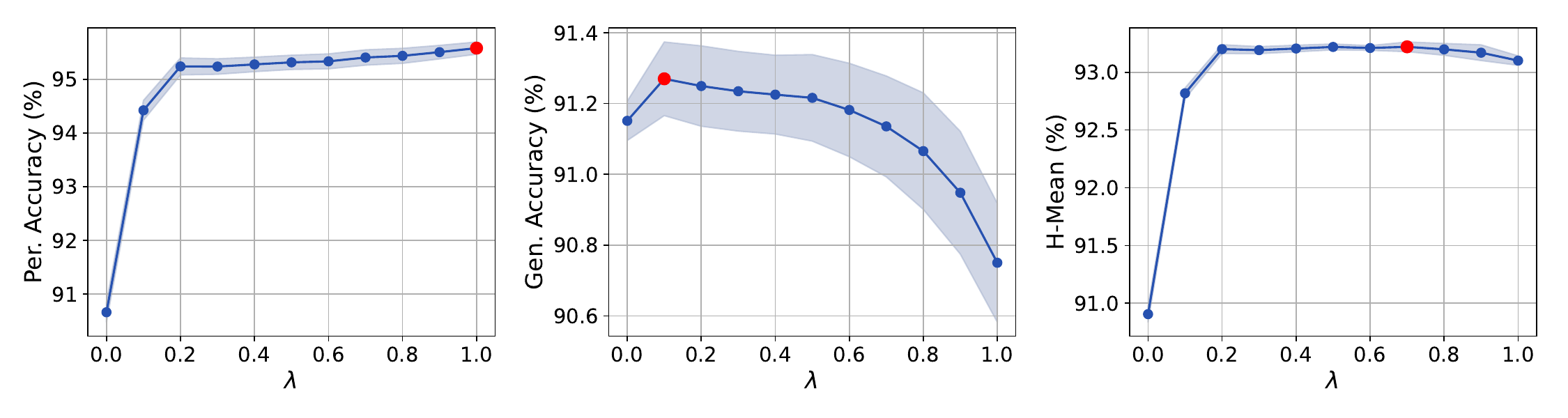}
    \caption{Personalization and generalization test accuracies alongside their harmonic mean (H-Mean) under changing values of parameter $\lambda$ for the \textbf{\textit{Food101}} dataset.}
    \label{fig:lambda-Food101}
\end{figure}
\begin{figure}[h]
    \centering
    \includegraphics[width=.9\linewidth]{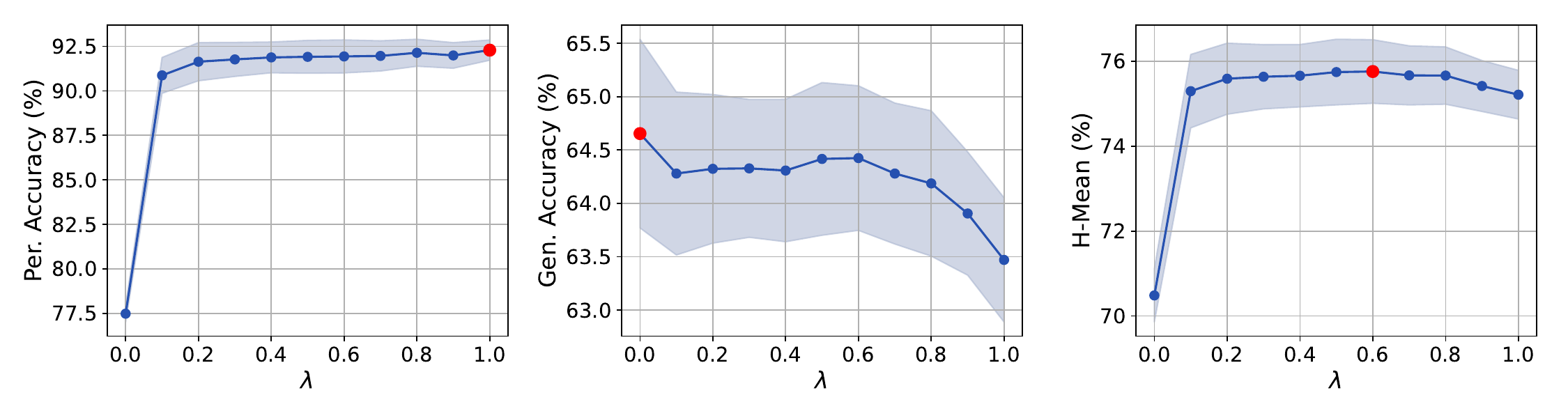}
    \caption{Personalization and generalization test accuracies alongside their harmonic mean (H-Mean) under changing values of parameter $\lambda$ for the \textbf{\textit{DTD}} dataset.}
    \label{fig:lambda-DTD}
\end{figure}
\begin{figure}[h]
    \centering
    \includegraphics[width=.9\linewidth]{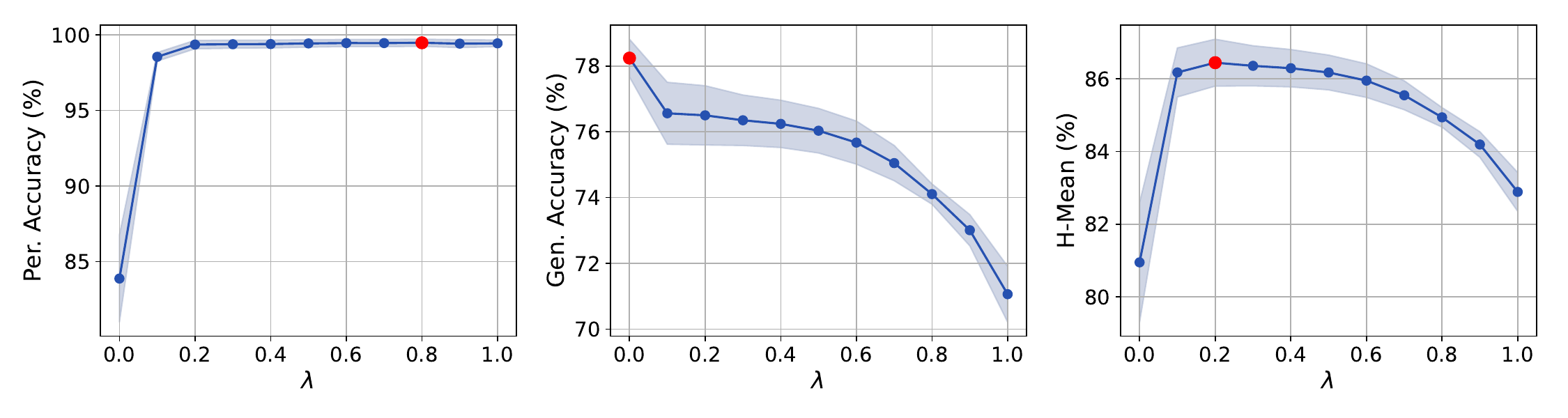}
    \caption{Personalization and generalization test accuracies alongside their harmonic mean (H-Mean) under changing values of parameter $\lambda$ for the \textit{EuroSAT} dataset.}
    \label{fig:lambda-EuroSat}
\end{figure}
\begin{figure}[h]
    \centering
    \includegraphics[width=.9\linewidth]{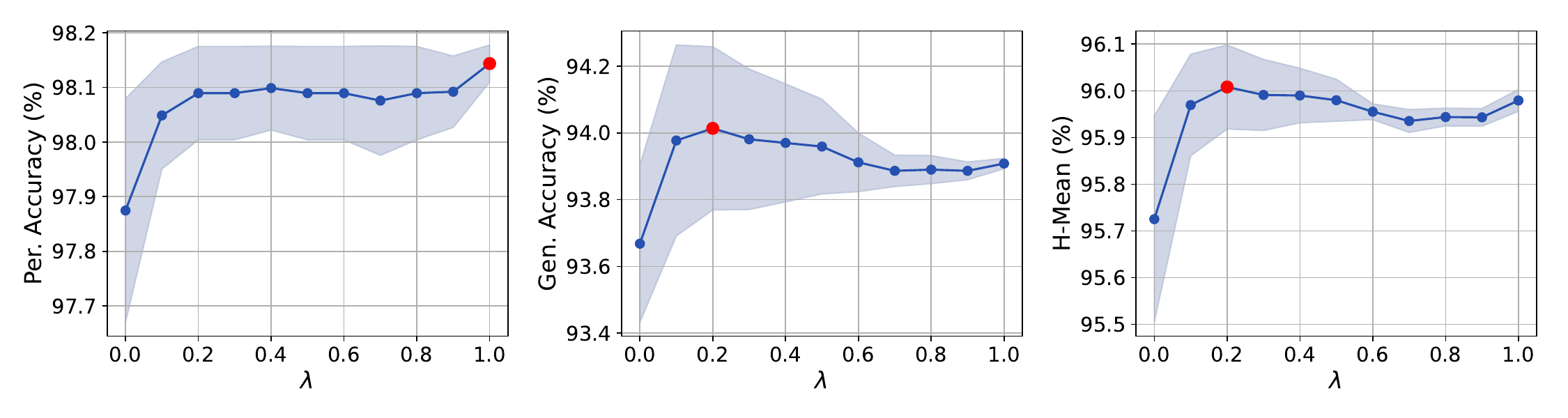}
    \caption{Personalization and generalization test accuracies alongside their harmonic mean (H-Mean) under changing values of parameter $\lambda$ for the \textbf{\textit{Caltech}} dataset.}
    \label{fig:lambda-Caltech101}
\end{figure}
\begin{figure}[h]
    \centering
    \includegraphics[width=.9\linewidth]{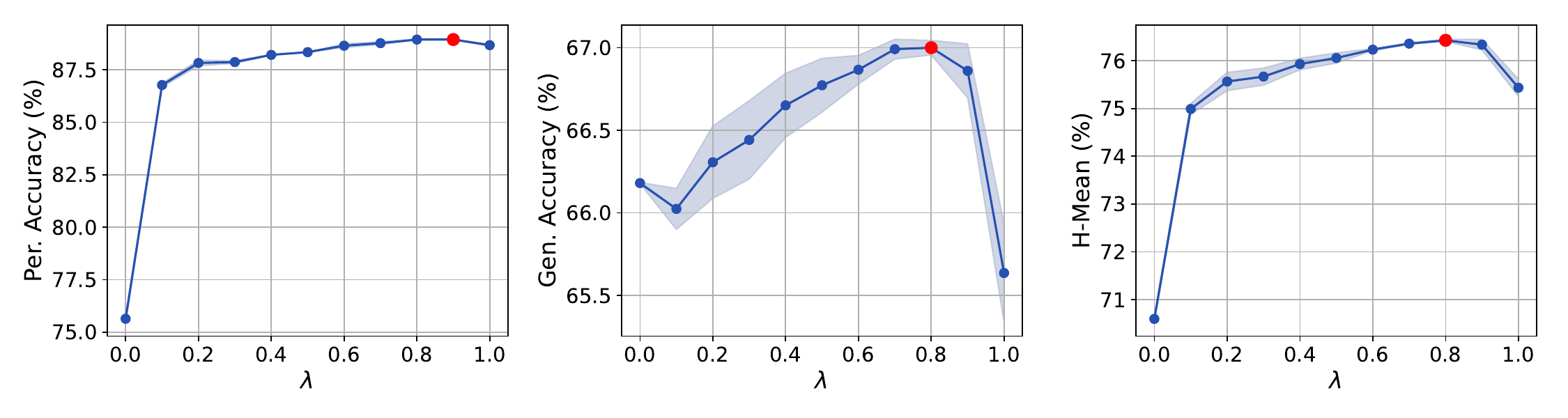}
    \caption{Personalization and generalization test accuracies alongside their harmonic mean (H-Mean) under changing values of parameter $\lambda$ for the \textbf{\textit{ImageNet}} dataset.}
    \label{fig:lambda-ImageNet}
\end{figure}

\clearpage
\section{Effect of Aggregation on Various Components}\label{ap:aggregation}
We devised a set of experiments to localize the effect of aggregating each trainable component in the model. We consider the following scenarios:
\begin{itemize}
    \item Our proposed method ({BLEND}) where the vision and text adapters are aggregated while the personalized projection adapter is preserved locally.
    \item Aggregation of all the trainable components (corresponding to ``Full" in the tables)
    \item No aggregation for any component where all the trained components are kept local (corresponding to ``NoAgg" in the tables)
    \item Only vision adapters are aggregated (corresponding to ``Only Vision" in the tables)
    \item Only text adapters are aggregated (corresponding to ``Only Text" in the tables)
\end{itemize}
Our results indicate that for 8 out of 10 scenarios (Tables~\ref{tab:agg-DTD}-\ref{tab:agg-FGVCAir}, \ref{tab:agg-OxFlowers}, \ref{tab:agg-StanCars}-\ref{tab:agg-Caltech101}), aggregating only the text adapters results in the best personalization performance. This comes naturally, as the information about the classes is primarily encoded into the text modality. Hence, aggregating only the text encoders, given the use of pretrained CLIP models for the vision backbones, can enable the models to swiftly adapt to the local data distributions, making personalization less effortful.
Intuitively, since the (Full) scenario aggregates all components, we expect it to result in better generalization performance. Our results verify this hypothesis as 9 out of 10 experiments (Tables~\ref{tab:agg-DTD}-\ref{tab:agg-Food101}, \ref{tab:agg-OxPets}-\ref{tab:agg-Caltech101}) have their best generalization test accuracy under the full aggregation scenario.
However, in all datasets (Tables~\ref{tab:agg-DTD}-\ref{tab:agg-Caltech101}) along with the case that the average performance across the datasets is considered (Table~\ref{tab:agg-avgAll}), 
considering the balance between the generalization and personalization (i.e., the H-Mean performance), {BLEND} results in the best performance.

\begin{table*}[h]
\centering
\begin{minipage}{0.32\linewidth}
\centering
\scriptsize
\captionof{table}{\textbf{Avg across datasets.}}
\begin{tabular}{c|ccc}
\hline
\rowcolor{gray!20} \textbf{Agg. Strategy} & Per. & Gen. & H-Mean \\
\hline
{BLEND} & 92.91 & 74.62 & \cellcolor{lightpink}{\textbf{82.15}} \\
Full & 83.35 & \textbf{75.41} & \cellcolor{lightpink}{79.01} \\
NoAgg & 89.12 & 69.71 & \cellcolor{lightpink}{77.71} \\
Only Vision & 74.64 & 70.02 & \cellcolor{lightpink}{71.82} \\
Only Text & \textbf{93.84} & 67.31 & \cellcolor{lightpink}{77.04} \\
\hline
\end{tabular}
\label{tab:agg-avgAll}
\end{minipage}
\hfill
\begin{minipage}{0.32\linewidth}
\centering
\scriptsize
\captionof{table}{DTD}
\begin{tabular}{c|ccc}
\hline
\rowcolor{gray!20} \textbf{Agg. Strategy} & Per. & Gen. & H-Mean \\
\hline
{BLEND} & 91.64 & 64.32 & \cellcolor{lightpink}{\textbf{75.59}} \\
Full & 81.83 & \textbf{65.62} & \cellcolor{lightpink}{72.84} \\
NoAgg & 92.76 & 62.96 & \cellcolor{lightpink}{75.01} \\
Only Vision & 86.15 & 61.92 & \cellcolor{lightpink}{72.05} \\
Only Text & \textbf{94.56} & 62.77 & \cellcolor{lightpink}{75.46} \\
\hline
\end{tabular}
\label{tab:agg-DTD}
\end{minipage}
\hfill
\begin{minipage}{0.32\linewidth}
\centering
\scriptsize
\captionof{table}{EuroSAT}
\begin{tabular}{c|ccc}
\hline
\rowcolor{gray!20} \textbf{Agg. Strategy} & Per. & Gen. & H-Mean \\
\hline
BLEND & 99.36 & 76.50 & \cellcolor{lightpink}{\textbf{86.44}} \\
Full & 89.50 & \textbf{78.88} & \cellcolor{lightpink}{83.85} \\
NoAgg & 99.56 & 72.15 & \cellcolor{lightpink}{83.66} \\
Only Vision & 99.43 & 72.85 & \cellcolor{lightpink}{84.09} \\
Only Text & \textbf{99.71} & 73.89 & \cellcolor{lightpink}{84.88} \\
\hline
\end{tabular}
\label{tab:agg-EuroSAT}
\end{minipage}
\end{table*}

\vspace{-3mm}
\begin{table*}[h]
\centering
\begin{minipage}{0.32\linewidth}
\centering
\scriptsize
\captionof{table}{FGVC Aircraft}
\begin{tabular}{c|ccc}
\hline
\rowcolor{gray!20} \textbf{Agg. Strategy} & Per. & Gen. & H-Mean \\
\hline
BLEND& 70.05 & 33.24 & \cellcolor{lightpink}{\textbf{45.09}} \\
Full & 40.22 & \textbf{35.17} & \cellcolor{lightpink}{37.53} \\
NoAgg & 52.46 & 26.49 & \cellcolor{lightpink}{35.20} \\
Only Vision & 25.28 & 30.35 & \cellcolor{lightpink}{27.58} \\
Only Text & \textbf{72.04} & 23.09 & \cellcolor{lightpink}{34.97} \\
\hline
\end{tabular}
\label{tab:agg-FGVCAir}
\end{minipage}
\hfill
\begin{minipage}{0.32\linewidth}
\centering
\scriptsize
\captionof{table}{Food101}
\begin{tabular}{c|ccc}
\hline
\rowcolor{gray!20} \textbf{Agg. Strategy} & Per. & Gen. & H-Mean \\
\hline
BLEND & 95.24 & 91.25 & \cellcolor{lightpink}{\textbf{93.20}} \\
Full & 90.61 & \textbf{91.27} & \cellcolor{lightpink}{90.94} \\
NoAgg & \textbf{96.00} & 87.10 & \cellcolor{lightpink}{91.33} \\
Only Vision & 86.58 & 88.00 & \cellcolor{lightpink}{87.28} \\
Only Text & 95.59 & 84.33 & \cellcolor{lightpink}{89.61} \\
\hline
\end{tabular}
\label{tab:agg-Food101}
\end{minipage}
\hfill
\begin{minipage}{0.32\linewidth}
\centering
\scriptsize
\captionof{table}{Flowers}
\begin{tabular}{c|ccc}
\hline
\rowcolor{gray!20} \textbf{Agg. Strategy} & Per. & Gen. & H-Mean \\
\hline
BLEND & 98.94 & \textbf{75.21} & \cellcolor{lightpink}{\textbf{85.46}} \\
Full & 97.84 & 74.11 & \cellcolor{lightpink}{84.34} \\
NoAgg & 96.64 & 69.52 & \cellcolor{lightpink}{80.87} \\
Only Vision & 80.51 & 74.09 & \cellcolor{lightpink}{77.17} \\
Only Text & \textbf{99.07} & 65.18 & \cellcolor{lightpink}{78.63} \\
\hline
\end{tabular}
\label{tab:agg-OxFlowers}
\end{minipage}
\end{table*}

\vspace{-3mm}
\begin{table*}[h]
\centering
\begin{minipage}{0.32\linewidth}
\centering
\scriptsize
\captionof{table}{OxfordPets}
\begin{tabular}{c|ccc}
\hline
\rowcolor{gray!20} \textbf{Agg. Strategy} & Per. & Gen. & H-Mean \\
\hline
BLEND & 99.21 & 96.99 & \cellcolor{lightpink}{\textbf{98.09}} \\
Full & 94.87 & \textbf{97.45} & \cellcolor{lightpink}{96.14} \\
NoAgg & \textbf{99.27} & 96.28 & \cellcolor{lightpink}{97.75} \\
Only Vision & 99.02 & 96.62 & \cellcolor{lightpink}{97.81} \\
Only Text & 99.23 & 96.06 & \cellcolor{lightpink}{97.62} \\
\hline
\end{tabular}
\label{tab:agg-OxPets}
\end{minipage}
\hfill
\begin{minipage}{0.32\linewidth}
\centering
\scriptsize
\captionof{table}{Stanford Cars}
\begin{tabular}{c|ccc}
\hline
\rowcolor{gray!20} \textbf{Agg. Strategy} & Per. & Gen. & H-Mean \\
\hline
BLEND & 93.72 & 70.20 & \cellcolor{lightpink}{\textbf{80.27}} \\
Full & 78.95 & \textbf{71.11} & \cellcolor{lightpink}{74.82} \\
NoAgg & 81.11 & 56.14 & \cellcolor{lightpink}{66.35} \\
Only Vision & 47.83 & 57.44 & \cellcolor{lightpink}{52.20} \\
Only Text & \textbf{94.47} & 48.10 & \cellcolor{lightpink}{63.75} \\
\hline
\end{tabular}
\label{tab:agg-StanCars}
\end{minipage}
\hfill
\begin{minipage}{0.32\linewidth}
\centering
\scriptsize
\captionof{table}{UCF101}
\begin{tabular}{c|ccc}
\hline
\rowcolor{gray!20} \textbf{Agg. Strategy} & Per. & Gen. & H-Mean \\
\hline
BLEND & 95.03 & 78.20 & \cellcolor{lightpink}{\textbf{85.80}} \\
Full & 86.51 & \textbf{79.65} & \cellcolor{lightpink}{82.94} \\
NoAgg & 94.04 & 71.37 & \cellcolor{lightpink}{81.15} \\
Only Vision & 73.16 & 68.97 & \cellcolor{lightpink}{71.00} \\
Only Text & \textbf{95.32} & 71.45 & \cellcolor{lightpink}{81.67} \\
\hline
\end{tabular}
\label{tab:agg-UCF101}
\end{minipage}
\end{table*}

\vspace{-3mm}
\begin{table*}[!h]
\centering
\begin{minipage}{0.32\linewidth}
\centering
\scriptsize
\captionof{table}{ImageNet}
\begin{tabular}{c|ccc}
\hline
\rowcolor{gray!20} \textbf{Agg. Strategy} & Per. & Gen. & H-Mean \\
\hline
BLEND & 87.83 & 66.31 & \cellcolor{lightpink}{\textbf{75.56}} \\
Full & 76.08 & \textbf{67.15} & \cellcolor{lightpink}{71.34} \\
NoAgg & 81.57 & 61.38 & \cellcolor{lightpink}{70.05} \\
Only Vision & 51.66 & 56.73 & \cellcolor{lightpink}{54.08} \\
Only Text & \textbf{90.30} & 54.57 & \cellcolor{lightpink}{68.03} \\
\hline
\end{tabular}
\label{tab:agg-ImageNet}
\end{minipage}
\begin{minipage}{0.32\linewidth}
\centering
\scriptsize
\captionof{table}{Caltech}
\begin{tabular}{c|ccc}
\hline
\rowcolor{gray!20} \textbf{Agg. Strategy} & Per. & Gen. & H-Mean \\
\hline
BLEND & 98.09 & \textbf{94.01} & \cellcolor{lightpink}{\textbf{96.01}} \\
Full & 97.11 & 93.67 & \cellcolor{lightpink}{95.36} \\
NoAgg & 97.80 & 93.68 & \cellcolor{lightpink}{95.69} \\
Only Vision & 96.75 & 93.20 & \cellcolor{lightpink}{94.94} \\
Only Text & \textbf{98.12} & 93.65 & \cellcolor{lightpink}{95.83} \\
\hline
\end{tabular}
\label{tab:agg-Caltech101}
\end{minipage}
\end{table*}

\clearpage
\section{Effect of Trainable Text Projection Head}\label{ap:trainable_txt_encoder}
To examine the effect of frozen vs. trainable text projection head, we experiment with different text projection head adapter ranks for the values $\{\emptyset, 4, 8, 16, 32, 64, 128, 256, \text{Full-rank}\}$, where $\emptyset$ corresponds to the case where no adapter is used (i.e., the case where the text projection head is frozen). We  include the results for all datasets in Tables~\ref{tab:text-dtd}-\ref{tab:text-Caltech101} with the average performance across the datasets included in Table~\ref{tab:text-avgAll}. The best personalization result in 6 out of 10 scenarios (Tables~\ref{tab:text-EuroSAT}, \ref{tab:text-FGVCAir}, \ref{tab:text-OxPets}, \ref{tab:text-UCF101}-\ref{tab:text-Caltech101}) is achieved without adapting the text projection head to the new data (i.e., by keeping it frozen). The same situation holds for the generalization performance in 5 out of 10 scenarios (Tables~\ref{tab:text-Food101}-\ref{tab:text-UCF101}). Given that the best average H-Mean across all datasets is also observed in the absence of the text projection head training (Table~\ref{tab:text-avgAll}, right column), we conclude that fine-tuning the text projection head over new sets of data can be redundant on average as it introduces training overhead with no performance gains in terms of H-Mean.
\begin{table*}[h]
\centering
\begin{minipage}{0.32\linewidth}
\centering
\scriptsize
\captionof{table}{\textbf{Avg across datasets.}}
\begin{tabular}{c|ccc}
\hline
\rowcolor{gray!20} Text Proj. Rank & Per. & Gen. & H-Mean \\
\hline
No Adapter & 92.91 & \textbf{74.62} & \cellcolor{lightpink}{\textbf{82.15}} \\
4 & 92.92 & 73.96 & \cellcolor{lightpink}{81.74} \\
8 & \textbf{92.98} & 73.88 & \cellcolor{lightpink}{81.72} \\
16 & 92.56 & 73.76 & \cellcolor{lightpink}{81.48} \\
32 & 92.60 & 73.76 & \cellcolor{lightpink}{81.47} \\
64 & 92.49 & 73.72 & \cellcolor{lightpink}{81.40} \\
128 & 92.11 & 73.56 & \cellcolor{lightpink}{81.19} \\
256 & 91.70 & 73.18 & \cellcolor{lightpink}{80.77} \\
Full-rank & 91.40 & 72.84 & \cellcolor{lightpink}{80.46} \\
\hline
\end{tabular}
\label{tab:text-avgAll}
\end{minipage}
\hfill
\begin{minipage}{0.32\linewidth}
\centering
\scriptsize
\captionof{table}{DTD}
\begin{tabular}{c|ccc}
\hline
\rowcolor{gray!20} Text Proj. Rank & Per. & Gen. & H-Mean \\
\hline
No Adapter & 91.64 & 64.32 & \cellcolor{lightpink}{75.59} \\
4 & 92.93 & \textbf{64.84} & \cellcolor{lightpink}{\textbf{76.38}} \\
8 & \textbf{92.99} & 64.76 & \cellcolor{lightpink}{76.35} \\
16 & 92.50 & 64.46 & \cellcolor{lightpink}{75.98} \\
32 & 92.89 & 63.80 & \cellcolor{lightpink}{75.65} \\
64 & 92.20 & 62.92 & \cellcolor{lightpink}{74.80} \\
128 & 92.25 & 62.25 & \cellcolor{lightpink}{74.33} \\
256 & 91.39 & 61.17 & \cellcolor{lightpink}{73.28} \\
Full-rank & 91.12 & 60.83 & \cellcolor{lightpink}{72.95} \\
\hline
\end{tabular}
\label{tab:text-dtd}
\end{minipage}
\hfill
\begin{minipage}{0.32\linewidth}
\centering
\scriptsize
\captionof{table}{EuroSAT}
\begin{tabular}{c|ccc}
\hline
\rowcolor{gray!20} Text Proj. Rank & Per. & Gen. & H-Mean \\
\hline
No Adapter & \textbf{99.36} & 76.50 & \cellcolor{lightpink}{\textbf{86.44}} \\
4 & 99.01 & 73.43 & \cellcolor{lightpink}{84.33} \\
8 & 99.07 & 73.35 & \cellcolor{lightpink}{84.29} \\
16 & 98.78 & 75.01 & \cellcolor{lightpink}{85.27} \\
32 & 98.62 & 75.44 & \cellcolor{lightpink}{85.49} \\
64 & 98.60 & \textbf{76.53} & \cellcolor{lightpink}{86.17} \\
128 & 98.27 & 75.95 & \cellcolor{lightpink}{85.68} \\
256 & 98.58 & 75.13 & \cellcolor{lightpink}{85.27} \\
Full-rank & 98.07 & 71.33 & \cellcolor{lightpink}{82.59} \\
\hline
\end{tabular}
\label{tab:text-EuroSAT}
\end{minipage}
\end{table*}

\vspace{-3mm}
\begin{table*}[h]
\centering
\begin{minipage}{0.32\linewidth}
\centering
\scriptsize
\captionof{table}{FGVC Aircraft}
\begin{tabular}{c|ccc}
\hline
\rowcolor{gray!20} Text Proj. Rank & Per. & Gen. & H-Mean \\
\hline
No Adapter & \textbf{70.05} & 33.24 & \cellcolor{lightpink}{45.09} \\
4 & 69.65 & \textbf{33.37} & \cellcolor{lightpink}{\textbf{45.12}} \\
8 & 69.71 & 33.29 & \cellcolor{lightpink}{45.06} \\
16 & 68.48 & 32.92 & \cellcolor{lightpink}{44.46} \\
32 & 69.25 & 32.48 & \cellcolor{lightpink}{44.22} \\
64 & 68.97 & 32.04 & \cellcolor{lightpink}{43.75} \\
128 & 68.05 & 32.62 & \cellcolor{lightpink}{44.10} \\
256 & 66.91 & 31.86 & \cellcolor{lightpink}{43.16} \\
Full-rank & 67.02 & 33.31 & \cellcolor{lightpink}{44.50} \\
\hline
\end{tabular}
\label{tab:text-FGVCAir}
\end{minipage}
\hfill
\begin{minipage}{0.32\linewidth}
\centering
\scriptsize
\captionof{table}{Food101}
\begin{tabular}{c|ccc}
\hline
\rowcolor{gray!20} Text Proj. Rank & Per. & Gen. & H-Mean \\
\hline
No Adapter & 95.24 & \textbf{91.25} & \cellcolor{lightpink}{\textbf{93.20}} \\
4 & 95.36 & 91.13 & \cellcolor{lightpink}{\textbf{93.20}} \\
8 & \textbf{95.42} & 91.05 & \cellcolor{lightpink}{93.18} \\
16 & 94.93 & 90.72 & \cellcolor{lightpink}{92.78} \\
32 & 94.98 & 90.37 & \cellcolor{lightpink}{92.62} \\
64 & 94.86 & 90.10 & \cellcolor{lightpink}{92.42} \\
128 & 94.66 & 90.18 & \cellcolor{lightpink}{92.36} \\
256 & 93.91 & 90.06 & \cellcolor{lightpink}{91.95} \\
Full-rank & 93.13 & 89.88 & \cellcolor{lightpink}{91.48} \\
\hline
\end{tabular}
\label{tab:text-Food101}
\end{minipage}
\hfill
\begin{minipage}{0.32\linewidth}
\centering
\scriptsize
\captionof{table}{Flowers}
\begin{tabular}{c|ccc}
\hline
\rowcolor{gray!20} Text Proj. Rank & Per. & Gen. & H-Mean \\
\hline
No Adapter & 98.94 & \textbf{75.21} & \cellcolor{lightpink}{\textbf{85.46}} \\
4 & 99.35 & 74.35 & \cellcolor{lightpink}{85.05} \\
8 & \textbf{99.41} & 74.27 & \cellcolor{lightpink}{85.02} \\
16 & 99.37 & 73.39 & \cellcolor{lightpink}{84.43} \\
32 & 99.32 & 73.47 & \cellcolor{lightpink}{84.46} \\
64 & 99.30 & 73.34 & \cellcolor{lightpink}{84.37} \\
128 & 99.30 & 73.17 & \cellcolor{lightpink}{84.25} \\
256 & 99.33 & 72.76 & \cellcolor{lightpink}{84.00} \\
Full-rank & 99.27 & 72.48 & \cellcolor{lightpink}{83.79} \\
\hline
\end{tabular}
\label{tab:text-Flowers}
\end{minipage}
\end{table*}

\vspace{-3mm}
\begin{table*}[h]
\centering
\begin{minipage}{0.32\linewidth}
\centering
\scriptsize
\captionof{table}{OxfordPets}
\begin{tabular}{c|ccc}
\hline
\rowcolor{gray!20} Text Proj. Rank & Per. & Gen. & H-Mean \\
\hline
No Adapter & \textbf{99.21} & \textbf{96.99} & \cellcolor{lightpink}{\textbf{98.09}} \\
4 & 99.11 & 96.05 & \cellcolor{lightpink}{97.56} \\
8 & 99.17 & 95.97 & \cellcolor{lightpink}{97.55} \\
16 & 99.13 & 96.29 & \cellcolor{lightpink}{97.69} \\
32 & 98.24 & 96.07 & \cellcolor{lightpink}{97.14} \\
64 & 98.41 & 95.76 & \cellcolor{lightpink}{97.07} \\
128 & 98.38 & 95.20 & \cellcolor{lightpink}{96.76} \\
256 & 98.08 & 95.77 & \cellcolor{lightpink}{96.91} \\
Full-rank & 98.49 & 96.16 & \cellcolor{lightpink}{97.31} \\
\hline
\end{tabular}
\label{tab:text-OxPets}
\end{minipage}
\hfill
\begin{minipage}{0.32\linewidth}
\centering
\scriptsize
\captionof{table}{Stanford Cars}
\begin{tabular}{c|ccc}
\hline
\rowcolor{gray!20} Text Proj. Rank & Per. & Gen. & H-Mean \\
\hline
No Adapter & 93.72 & \textbf{70.20} & \cellcolor{lightpink}{\textbf{80.27}} \\
4 & 93.75 & 68.51 & \cellcolor{lightpink}{79.17} \\
8 & \textbf{93.81} & 68.43 & \cellcolor{lightpink}{79.14} \\
16 & 92.90 & 66.61 & \cellcolor{lightpink}{77.59} \\
32 & 92.89 & 67.92 & \cellcolor{lightpink}{78.47} \\
64 & 93.76 & 68.79 & \cellcolor{lightpink}{79.35} \\
128 & 92.74 & 69.49 & \cellcolor{lightpink}{79.45} \\
256 & 92.73 & 68.75 & \cellcolor{lightpink}{78.96} \\
Full-rank & 92.72 & 69.30 & \cellcolor{lightpink}{79.32} \\
\hline
\end{tabular}
\label{tab:text-StanCars}
\end{minipage}
\hfill
\begin{minipage}{0.32\linewidth}
\centering
\scriptsize
\captionof{table}{UCF101}
\begin{tabular}{c|ccc}
\hline
\rowcolor{gray!20} Text Proj. Rank & Per. & Gen. & H-Mean \\
\hline
No Adapter & \textbf{95.03} & \textbf{78.20} & \cellcolor{lightpink}{\textbf{85.80}} \\
4 & 94.76 & 76.67 & \cellcolor{lightpink}{84.76} \\
8 & 94.82 & 76.59 & \cellcolor{lightpink}{84.73} \\
16 & 94.37 & 76.24 & \cellcolor{lightpink}{84.35} \\
32 & 94.85 & 76.58 & \cellcolor{lightpink}{84.74} \\
64 & 94.60 & 76.28 & \cellcolor{lightpink}{84.45} \\
128 & 93.99 & 76.41 & \cellcolor{lightpink}{84.29} \\
256 & 94.02 & 76.44 & \cellcolor{lightpink}{84.33} \\
Full-rank & 92.69 & 75.13 & \cellcolor{lightpink}{82.99} \\
\hline
\end{tabular}
\label{tab:text-UCF101}
\end{minipage}
\end{table*}

\vspace{-3mm}
\begin{table*}[!h]
\centering
\begin{minipage}{0.32\linewidth}
\centering
\scriptsize
\captionof{table}{ImageNet}
\begin{tabular}{c|ccc}
\hline
\rowcolor{gray!20} Text Proj. Rank & Per. & Gen. & H-Mean \\
\hline
No Adapter & \textbf{87.83} & 66.31 & \cellcolor{lightpink}{75.56} \\
4 & 87.25 & 66.93 & \cellcolor{lightpink}{75.75} \\
8 & 87.31 & 66.85 & \cellcolor{lightpink}{75.72} \\
16 & 87.14 & \textbf{67.49} & \cellcolor{lightpink}{\textbf{76.07}} \\
32 & 86.86 & 67.46 & \cellcolor{lightpink}{75.94} \\
64 & 86.15 & 67.36 & \cellcolor{lightpink}{75.61} \\
128 & 85.46 & 66.15 & \cellcolor{lightpink}{74.58} \\
256 & 84.07 & 65.70 & \cellcolor{lightpink}{73.76} \\
Full-rank & 83.67 & 65.84 & \cellcolor{lightpink}{73.69} \\
\hline
\end{tabular}
\label{tab:text-ImageNet}
\end{minipage}
\begin{minipage}{0.32\linewidth}
\centering
\scriptsize
\captionof{table}{Caltech}
\begin{tabular}{c|ccc}
\hline
\rowcolor{gray!20} Text Proj. Rank & Per. & Gen. & H-Mean \\
\hline
No Adapter & \textbf{98.09} & 94.01 & \cellcolor{lightpink}{96.01} \\
4 & 97.99 & 94.32 & \cellcolor{lightpink}{96.12} \\
8 & 98.05 & 94.24 & \cellcolor{lightpink}{96.11} \\
16 & 98.01 & \textbf{94.47} & \cellcolor{lightpink}{\textbf{96.21}} \\
32 & 98.05 & 94.02 & \cellcolor{lightpink}{95.99} \\
64 & 98.02 & 94.09 & \cellcolor{lightpink}{96.02} \\
128 & 97.99 & 94.17 & \cellcolor{lightpink}{96.04} \\
256 & 97.96 & 94.18 & \cellcolor{lightpink}{96.03} \\
Full-rank & 97.82 & 94.14 & \cellcolor{lightpink}{95.95} \\
\hline
\end{tabular}
\label{tab:text-Caltech101}
\end{minipage}
\end{table*}

\clearpage
\section{Performance Variation Across Different Settings}\label{app:performance_variation}

Table~\ref{tab:blend_performance_variation} summarizes BLEND's performance under varying levels of data heterogeneity, number of clients, and local data scarcity. Across all settings, BLEND achieves larger H-mean gains over the best baseline as task difficulty increases. Specifically, the gains increase with stronger data heterogeneity, more clients, and fewer local training samples. Similar improvements are observed on the real-world iNaturalist-User-120k dataset as the number of clients increases. These results highlight BLEND's effectiveness in challenging federated/distributed fine-tuning scenarios.

\begin{table*}[!h]
\centering
\caption{Performance variation of BLEND across different settings. H-mean gains are reported relative to the best baseline.}
\label{tab:blend_performance_variation}
\resizebox{\textwidth}{!}{%
\begin{tabular}{l|ll}
\hline
\rowcolor{gray!20} \textbf{Setting} &
\textbf{Increasing Difficulty} &
\textbf{H-Mean Gain vs. Best Baseline} \\
\hline

Data heterogeneity (across 10 datasets in Sec.~\ref{sec:experiments}) &
$\alpha^{\mathrm{Dir}}: 10 \rightarrow 0.1 \rightarrow 0.01$ &
$+0.9 \rightarrow +2.9 \rightarrow +4.8$ \\

\hline

Number of clients (across 10 datasets in Sec.~\ref{sec:experiments}) &
$N: 10 \rightarrow 20 \rightarrow 30$ &
$+4.8 \rightarrow +5.6 \rightarrow +6.1$ \\

\hline

Real-World iNaturalist-User-120k Dataset (Appendix~\ref{ap:real_large}) &
$N: 100 \rightarrow 200$ &
$+4.9 \rightarrow +5.6$ \\

\hline

Local data scarcity (across 10 datasets in Sec.~\ref{sec:experiments}) &
\#Shots: $32 \rightarrow 16 \rightarrow 8 \rightarrow 4 \rightarrow 2$ &
$+3.0 \rightarrow +4.8 \rightarrow +5.2
\rightarrow +5.3 \rightarrow +5.7$ \\

\hline
\end{tabular}%
}
\end{table*}

\clearpage

\section{Effect of the Running Anchor Loss}
\label{app:running_anchor}

To further investigate the effectiveness of BLEND's running anchor loss, we compare it against two alternative zero-shot CLIP regularization strategies: (i) \textbf{Zero-Shot G-Head}, which uses the frozen zero-shot CLIP vision representation for the G-head, and (ii) \textbf{Zero-Shot Distilled G-Head}, which employs KL divergence to distill zero-shot knowledge from frozen CLIP's vision projection into the G-head.

All experiments follow the experimental setting of Table~I in the main paper. As shown in Table~\ref{tab:running_anchor}, our proposed running anchor loss achieves the highest H-mean of 82.15\%, outperforming the zero-shot G-head and zero-shot distilled G-head by 2.20 and 1.61 percentage points, respectively. These results demonstrate the effectiveness of our running anchor loss in balancing personalization and generalization.

\begin{table}[!h]
\centering
\scriptsize
\caption{Comparison of BLEND's running anchor with alternative zero-shot CLIP regularization strategies.}
\label{tab:running_anchor}
\begin{tabular}{l|ccc}
\hline
\rowcolor{gray!20}\textbf{Method} & \textbf{Per.} & \textbf{Gen.} & \textbf{H-Mean} \\
\hline
BLEND with Zero-Shot G-Head & 89.50 & 72.31 & 79.95 \\
BLEND with Zero-Shot Distilled G-Head & 92.83 & 72.35 & 80.54 \\
BLEND with Running Anchor Loss & \textbf{92.91} & \textbf{74.62} & \textbf{82.15} \\
\hline
\end{tabular}
\end{table}

\clearpage

\section{Effect of the Adaptive Fusion Weight}
\label{app:adaptive_fusion}
To further investigate the sensitivity of BLEND to fusion weight $\gamma$, we conduct additional ablation studies over different values of $\gamma$ while varying the degree of data heterogeneity and the number of clients. Table~\ref{tab:adaptive_fusion} summarizes the optimal fusion weights and their corresponding H-mean performance. As the Dirichlet concentration parameter decreases from $10$ to $0.01$, indicating increasingly heterogeneous client data, the optimal $\gamma$ increases from $0.03$ to $0.09$ (this is because under more heterogeneous data, the need for model personalization increases, which is achieved under a larger value of $\gamma$). Similarly, increasing the number of clients from $10$ to $30$ results in a small increase in the optimal $\gamma$ from $0.09$ to $0.10$. These observations suggest that setting-dependent fusion weights can further improve BLEND's performance.

\begin{table}[!h]
\centering
\scriptsize
\caption{Effect of fusion weight $\gamma$ under varying levels of data heterogeneity and different numbers of clients. }
\label{tab:adaptive_fusion}
\begin{minipage}[t]{0.46\columnwidth}
\raggedleft
\begin{tabular}{ccc}
\hline
\rowcolor{gray!20} $\alpha^{\mathrm{Dir}}$ & Optimal $\gamma$ & H-Mean \\
\hline
10 & 0.03 & 79.66 \\
0.1 & 0.06 & 80.86 \\
0.01 & 0.09 & 82.15 \\
\hline
\end{tabular}
\end{minipage}%
\hspace{2.5em}%
\begin{minipage}[t]{0.46\columnwidth}
\raggedright
\begin{tabular}{ccc}
\hline
\rowcolor{gray!20} $N$ & Optimal $\gamma$ & H-Mean \\
\hline
10 & 0.09 & 82.15 \\
20 & 0.10 & 82.83 \\
30 & 0.10 & 83.10 \\
\hline
\end{tabular}
\end{minipage}
\end{table}

\clearpage

\section{Effect of Explicit Personalization}
\label{app:explicit_personalization}

To investigate the effect of explicitly supervising the personalized branch (i.e., P-head), we introduce an additional cross-entropy (CE) loss that directly optimizes the personalized representation, rather than relying solely on its contribution through the fusion CE loss. We refer to this variant as BLEND (Sep-Per). As shown in Table~\ref{tab:explicit_personalization}, introducing explicit personalization increases personalization accuracy from 92.91\% to 93.56\%, but decreases generalization accuracy from 74.62\% to 70.96\%. Consequently, the H-mean decreases from 82.15\% to 79.78\%. These results suggest that directly optimizing the personalized branch can improve client-specific performance at the expense of generalization, supporting BLEND's original formulation, which optimizes the personalized representation through the fusion objective.

\begin{table}[!h]
\centering
\scriptsize
\caption{Effect of explicitly supervising the personalized branch with a separate cross-entropy loss.}
\label{tab:explicit_personalization}
\begin{tabular}{l|ccc}
\hline
\rowcolor{gray!20} \textbf{Method} & \textbf{Per.} & \textbf{Gen.} & \textbf{H-Mean} \\
\hline
BLEND (Sep-Per) & \textbf{93.56} & 70.96 & 79.78 \\
BLEND (Original) & 92.91 & \textbf{74.62} & \textbf{82.15} \\
\hline
\end{tabular}
\end{table}

\clearpage

\section{Hyperparameter Selection and Held-Out Validation}
\label{app:hyperparameter_validation}

In this ablation, we evaluate the robustness of BLEND's hyperparameter selection using a held-out validation split that is independent of the training and test sets. As shown in Tables~\ref{tab:validation_gamma_lambda} and~\ref{tab:validation_adapter_ranks}, the highest validation H-mean is obtained with $\gamma=0.09$, $\lambda=0.2$, a global adapter dimension of $16$, and a projection adapter dimension of $256$. These values coincide with the configuration identified in our original test-set sensitivity analyses. This agreement provides independent validation of the selected configuration and indicates that the reported hyperparameter choices are not specific to the original test split.

\begin{table}[!h]
\centering
\scriptsize
\caption{Hyperparameter sensitivity of the fusion weight $\gamma$ and regularization coefficient $\lambda$ on the held-out validation split. The highest H-mean for each hyperparameter is highlighted in bold.}
\label{tab:validation_gamma_lambda}
\begin{tabular}{l|ccccccc}
\hline
\rowcolor{gray!20} $\gamma$ & 0.0 & 0.06 & 0.08 & 0.09 & 0.1 & 0.3 & 1 \\
\hline
H-Mean & 78.86 & 82.25 & 82.38 & \textbf{82.39} & 82.36 & 79.47 & 73.47 \\
\hline
\rowcolor{gray!20} $\lambda$ & 0.0 & 0.1 & 0.19 & 0.2 & 0.21 & 0.5 & 1 \\
\hline
H-Mean & 77.32 & 82.08 & 82.24 & \textbf{82.39} & 82.36 & 82.21 & 81.16 \\
\hline
\end{tabular}
\end{table}

\begin{table}[!h]
\centering
\scriptsize
\caption{Sensitivity to the global and projection adapter dimensions on the held-out validation split. The highest H-mean for each adapter is highlighted in bold.}
\label{tab:validation_adapter_ranks}
\begin{tabular}{l|ccccccccc}
\hline
\rowcolor{gray!20} Global adapter dim. & 0 & 4 & 8 & 16 & 32 & 64 & 128 & 256 & 512 \\
\hline
H-Mean & 80.84 & 82.11 & 82.18 & \textbf{82.39} & 82.06 & 81.97 & 81.75 & 81.31 & 81.32 \\
\hline
\end{tabular}

\vspace{1em}
\begin{tabular}{l|cccccccc}
\hline
\rowcolor{gray!20} Projection adapter dim. & 4 & 8 & 16 & 32 & 64 & 128 & 256 & 768 \\
\hline
H-Mean & 81.78 & 81.92 & 82.08 & 82.14 & 82.20 & 82.23 & \textbf{82.39} & 82.18 \\
\hline
\end{tabular}
\end{table}

\section{Limitations and Generalization Beyond CLIP}
\label{app:generalization}

\paragraph{Extension to Other Multimodal Architectures.} BLEND currently focuses on CLIP-based image--text models, consistent with prior work on personalized federated learning with vision-language models. Extending BLEND to other multimodal architectures primarily requires reconsidering where global and personalized modules are inserted into the model pipeline. For instance, in Qwen2.5-VL, one possible extension is to maintain a global vision merger while introducing a separate personalized vision merger. Their resulting visual representations could then be fused and passed to the language model as visual tokens for generation.

\paragraph{Extension to Other Downstream Tasks.} While BLEND currently focuses on classification, extending it to image--text retrieval would require adapting its fusion mechanism and training objectives. Unlike classification, retrieval involves ranking image--text pairs rather than selecting from a fixed set of classes. Therefore, the fusion of global and personalized visual representations must preserve cross-modal alignment while balancing generalization and client-specific personalization. For generative tasks such as visual question answering (VQA), BLEND's global--local decomposition could be retained, while its fusion mechanism and anchor objectives would need to be adapted to autoregressive generation.

\paragraph{Extension to Multi-Task and Multi-Modal Settings.} Extending BLEND to more complex multi-task and multi-modal settings introduces two important challenges. First, different tasks may rely on different modalities to varying degrees, making the information imbalance across modalities task-dependent. This introduces an additional form of heterogeneity, which we refer to as \textit{heterogeneity in information imbalance across tasks}. Addressing this challenge may require task-dependent placement of personalized modules and adaptive fusion weights. Second, globally shared and aggregatable components must account for task similarity, modality relevance, and potential negative cross-task transfer under heterogeneous client data distributions. Investigating these challenges represents an important direction for extending BLEND to broader multimodal federated learning scenarios.

\newpage

\end{document}